\documentclass[twoside]{article}

\usepackage[accepted]{aistats2026}

\usepackage[utf8]{inputenc} 
\usepackage[T1]{fontenc}
\usepackage{xr-hyper}
\usepackage{hyperref}       
\usepackage{url}            
\usepackage{booktabs}       
\usepackage{amsfonts}       
\usepackage{nicefrac}       
\usepackage{microtype}      
\usepackage{xcolor}         
\usepackage{amsthm}
\usepackage{amsmath}
\usepackage{amssymb}
\usepackage{dsfont}
\usepackage{mathtools}
\usepackage{bbm}
\usepackage{graphicx}
\usepackage{subcaption} 
\usepackage[capitalise]{cleveref}
\usepackage{algorithm}
\usepackage{algorithmic}
\usepackage{natbib}
\usepackage{amssymb}
\usepackage[dvipsnames]{xcolor}

\usepackage{tikz}

\newtheorem{theorem}{Theorem}[section]
\newtheorem{corollary}{Corollary}[theorem]
\newtheorem{proposition}{Proposition}[section]
\newtheorem{lemma}[theorem]{Lemma}
\newtheorem*{remark}{Remark}
\newtheorem*{theorem*}{Theorem}

\theoremstyle{definition}

\newcommand{\pro}{\mathbb{P}}
\newcommand{\ind}{\mathbb{I}}
\newcommand{\y}{y}
\newcommand{\meth}{\varphi}

\newcommand{\clim}{\nu}
\newcommand{\p}{T}
\newcommand{\hatp}{\widetilde{T}}

\newcommand{\err}{\mathcal{R}}
\newcommand{\exe}{\mathbb{E}}
\newcommand{\ntasks}{N}
\newcommand{\nworkers}{H}
\newcommand{\nlabels}{C}
\newcommand{\oclass}{c}
\newcommand{\eclass}{\hat{c}}
\newcommand{\oMAP}{\textrm{oMAP}}
\newcommand{\eMAP}{\textrm{eMAP}}
\newcommand{\MAP}{\textrm{MAP}}
\newcommand{\MV}{\textrm{MV}}
\DeclareMathSymbol{\mhyphen}{\mathord}{AMSa}{"39}
\newcommand{\zone}{{0\mhyphen1}}
\newcommand{\loss}{\mathcal{L}}

\newcommand{\bigO}{\mathcal{O}}

\DeclareRobustCommand{\xmark}{
\textcolor{BrickRed}{\tikz[scale=0.23] {
    \draw[line width=0.7,line cap=round] (0,0) to [bend left=6] (1,1);
    \draw[line width=0.7,line cap=round] (0.2,0.95) to [bend right=3] (0.8,0.05);
}}}
\DeclareRobustCommand{\cmark}{
\textcolor{OliveGreen}{\tikz[scale=0.23] {
    \draw[line width=0.7,line cap=round] (0.25,0) to [bend left=10] (1,1);
    \draw[line width=0.8,line cap=round] (0,0.35) to [bend right=1] (0.23,0);
}}}

\usepackage{enumitem}

\makeatletter
\newcommand\blfootnote[1]{%
  \begingroup
  \renewcommand\thefootnote{}%
  \let\Hy@raisedlink\@gobble
  \footnotetext{#1}%
  \addtocounter{footnote}{0}%
  \endgroup
}
\makeatother

\makeatletter
\def\endfirstcolumn{\if@firstcolumn\newpage\fi}
\makeatother

\begin{document}

\DeclarePairedDelimiter\ceil{\lceil}{\rceil}
\DeclarePairedDelimiter\floor{\lfloor}{\rfloor}
\everypar\expandafter{\the\everypar\looseness=-1}
\linepenalty=1000

\runningauthor{Purificato, Bucarelli, Nelakanti, Bacciu, Silvestri, Mantrach}

\twocolumn[

\aistatstitle{The Majority Vote Paradigm Shift: When Popular Meets Optimal}

\aistatsauthor{Antonio Purificato$^{1,2,\diamondsuit}$ \And  Maria Sofia Bucarelli$^{2,3,4,5,\diamondsuit,\spadesuit}$ \And  Anil Nelakanti$^{6,\clubsuit}$ \AND Andrea Bacciu$^{1}$ \And Fabrizio Silvestri$^{2}$ \And Amin Mantrach$^{1}$}

\aistatsaddress{\\ $^{1}$Amazon,  $^{2}$Sapienza University of Rome, $^3$CNRS, $^4$Inria, $^5$i3s, $^{6}$IIIT Hyderabad} 
]

\begin{abstract}
Reliably labelling data typically requires annotations from multiple human workers. However, humans are far from being perfect. 
Hence, it is a common practice to aggregate labels gathered from multiple annotators to make a more confident estimate of the true label.
Among many aggregation methods, the simple and well-known Majority Vote (MV)  selects the class label receiving the highest number of votes. 
However, despite its importance, the optimality of MV’s label aggregation has not been extensively studied. We address this gap in our work by characterising the conditions under which MV achieves the theoretically optimal lower bound on label estimation error. 
Our results capture the tolerable limits on annotation noise under which MV can optimally recover labels for a given class distribution.
\looseness -1 This certificate of optimality provides a more principled approach to model selection for label aggregation as an alternative to otherwise inefficient practices that sometimes include \textit{higher experts}, \textit{gold} labels, etc.,  that are all marred by the same human uncertainty despite huge time and monetary costs.
Experiments on both synthetic and real-world data corroborate our theoretical findings.
\end{abstract}

\begin{figure*}[t]
    \centering
    \includegraphics[width=.65\linewidth]{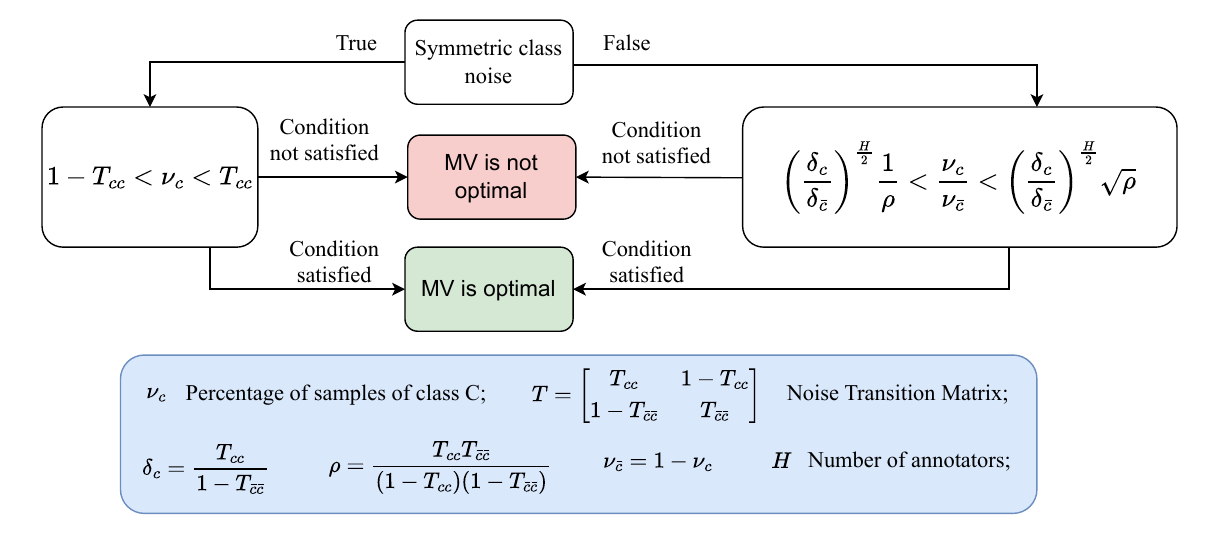}
    \caption{Flow-chart summarizing conditions for optimality of \MV~with its label estimates matching that of \oMAP.}
    \label{fig:diagram_summary}
\end{figure*}

\section{INTRODUCTION}
Data labeling or annotation is crucial for numerous real-world machine learning tasks like, for example, search and retrieval \citep{gligorov2013evaluation,JMLR:v18:16-206}, medical imaging \citep{wang2019medical}, and translation and summarization \citep{sarhan2020can}. 
If humans were perfect at labelling data, a single round of annotation would suffice, however, worker reliability varies widely across its population.
Hence it is common to gather a noisy label set~\citep{li2014wisdom} by requesting annotations from distinct workers for each task before aggregating them~\citep{huang-etal-2023-incorporating}.
The most commonly used aggregation method is \textit{Majority Vote} (MV).
MV estimation selects the class label that receives the highest number of votes from distinct annotators for a given task. Despite being easy to implement, it fails to recover true labels when incorrect annotations begin to dominate.
The theoretical lower bound on the minimum achievable error given the true worker noise was described by~\cite{nitzan1981characterization}.
Assuming there is an oracle revealing the true worker noise and class distribution, it reduces to the \textit{maximum a posteriori} estimate with each vote weighted by the odds ratio of the worker's reliability and the class distribution.
\blfootnote{\hspace{-5ex} $^\diamondsuit$Equal contribution. \\
$^\spadesuit$Work done while at Sapienza. \\
$^\clubsuit$Work done while at Amazon. \\
Correspondence to: \href{purifian@amazon.lu}{\texttt{purifian@amazon.lu}}
}
Correspondingly, we refer to this estimate as the oracle MAP (oMAP) that is practically unknown without access to such an oracle. 
This makes the problem ill-posed~\citep{spurling2021estimating,shi2021learning} for which numerous methods have been proposed in literature.
The earliest from~\cite{dawid} (DS) uses expectation-maximization (EM), followed by several iterative~\citep{glad, mace} and non-iterative methods~\citep{li2019exploiting,yang2024lightweight}.
Some of these methods offer guarantees on the expected error rate $\exe[\err]$ of their estimate $\hat{y}$ of the true label $\y$ where $\err\!=\!\frac{1}{|\{\y,\hat{y}\}|}\sum_{\{\y,\hat{y}\}}\ind[\hat{y}\!\neq\!\y]$. 
They help understand
\textit{``How many more annotations would I need for $\exe[\err]$ to be arbitrarily small?''} or say \textit{``What is minimum $\exe[\err]$ reachable given my annotation budget?''}
However, they do not help decide
\textit{``Which method is to be preferred for label aggregation on my data?''} or 
\textit{``Does MV achieve the lowest achievable error for my annotation job?''}.

Precisely, this is the question we answer by deriving the parameter configurations where MV is sample-wise optimal, \textit{i.e.,} when its label aggregate matches oMAP for all instances.
We present a mechanism to verify if these conditions hold true for any given real-data even while the true parameters, \textit{i.e.,} annotators' confusion matrix and class distribution are unknown.
Specifically, our contributions address the following research questions:
\begin{itemize}
    \item \textbf{RQ1}: \textit{Can MV achieve the theoretically optimal label estimation error and under what conditions?}\\
    We identify a range of reliability and class imbalance parameters where label recovery using MV is optimal for each instance with its estimate matching that of oMAP. We derive these conditions for informative, non-adversarial workers who share the same class confusion matrix when annotating binary tasks. We also extend our results to scenarios with uniformly perturbed noise transition matrices and to specific cases involving two distinct categories of annotators.
See \cref{fig:diagram_summary} for summary.
    \item \textbf{RQ2}: \textit{Are MV's optimality conditions verifiable on real-data and how tight are they?}\\
    The listed optimality conditions require an oracle to reveal the true parameters to verify if it holds. However, approximating true parameters with estimates works quite well in most cases though it comes with no certificate of correctness due to estimation errors.  To address this, we derive stricter conditions on the estimated parameters  that can be verified offering a high-probability guarantee of the true values lying within this regime.
    Our empirical results on synthetic and real data justify utility of checking the condition with estimated quantities to verify MV optimality.
    \endfirstcolumn
    \blfootnote{\hspace{-5ex} Code for reproducing the results can be found at \url{https://github.com/amazon-science/majority_vote_paradigm_shift}.}
    We also explore false-discovery and sensitivity trade-off in the two verification approaches we propose. Practitioners must choose from them depending on the requirements of their specific use-case. 
\end{itemize}

\section{RELATED WORK}\label{sec:related}
In practice, $\MV$ is the most popular method due to ease of implementation.
A popular alternative is a group of iterative solvers that update estimates of task labels and model priors (of class distribution and worker reliability) in successive loops.
\cite{dawid} and more recent variants  of iterative methods ~\citep{karger2014budget, li2014error, li2019truth, 9296844, chen2023sample} are examples of these methods. The parameters are estimated via an EM algorithm. 
Spectral algorithms that use the eigenspace of the worker-label matrix of annotations have also been proposed~\citep{ghosh, zhang2014spectral, pmlr-v151-tenzer22a} that have some similarities to iterative solvers (see~\cite{karger2014budget}).
 
Methods that leverage other principles like Bayesian formulation have also been studied \citep{soto2016urn}. \cite{li2019exploiting} use the mean-field variational method to
mine latent source relationships through tensor decomposition and object clustering. 
\cite{yang2024lightweight} view the label aggregation task as a dynamic system, with task identifiers serving as time-slices. Using a Dynamic Bayesian Network to model this system, they develop two label aggregation algorithms.
\cite{li2019truth} proposed a Bayesian model based on conjugate prior and iterative EM reasoning for highly redundant labeled data.
Similarities can be drawn among various methods and where feasible, they theoretically bound $\exe[\err]$ as $\exp{(-\bigO(vr))}$ for some measure of crowd reliability~$r$ and annotation volume~$v$.
They fundamentally differ in the implicit assumptions on the priors in model specification and the solver used that are necessary for the bounds to hold true.
For example,~\cite{karger2014budget} assumes sparse assignment of samples to workers, while~\cite{dalvi} requires large eigengap for worker-label annotation matrix.
Designing annotation jobs where the assumptions hold is non-trivial since they are not easily verifiable making it a trial-and-error to evaluate them suitability in practice.

Certain classes of methods explicitly estimate only the noise transition matrix of workers, sometimes with guarantees on optimality, from which label estimates could be inferred subsequently.
For example, \citet{bonald} employ correlations of worker triplets to estimate workers' noise transition with $\ell_\infty$-norm of the error bounded as $\bigO(\frac{1}{v\sqrt{\ntasks}})$, the minimax lower bound for their setting with $\ntasks$ being the number of samples or tasks labelled.
Similarly,~\citet{bucarelli2023leveraging} solve an optimisation problem that uses pairwise annotator similarities with a guarantee on $p$-norm of estimation error as $\bigO(\frac{1}{\sqrt{\ntasks}})$.
The methods of DS and others like~\cite{li2019exploiting} estimate this as a by-product of their algorithm.
We describe how these noise estimates can be leveraged to derive optimality certificates on real-data.

In contrast to above related literature, our work does not propose a new aggregation method but instead studies the conditions under which the simplest aggregation method optimally recovers labels.
A relevant work to our discussion is that of~\citet{nitzan1981characterization} that establishes the optimality of a suitably weighted majority vote.
A special case of their method is that of uniform class distribution with equally reliable workers\footnote{See Corollary 2b in~\citep{nitzan1981characterization}.} where it reduces to MV aligning with our finding for this scenario.
However, the equivalence of MV's label recovery to that of oMAP applies to a larger more general range of parameters that we establish in our work.

Notwithstanding some differences like the distinct but symmetric worker noise of~\cite{nitzan1981characterization} as against non-symmetric shared noise that we study, the assumption of access to an oracle renders it impractical for real-world use. We address this shortcoming in our work, describing a method to verify with high probability but without access to an oracle if a given parameter configuration is in the regime where MV is optimal.
This result is central to practical application of our theory that is validated in our experiments.
For all such configurations, label estimate for each sample from MV is guaranteed to match that of oMAP with no incentive for exploring more complex aggregation alternatives.

\section{GAP IN MV AND MAP METHODS}\label{sec:problem}

\paragraph*{Notation.}
The number of tasks, annotators, and classes are represented by $\ntasks, \nworkers$, and $\nlabels$, respectively.
$\nu$ denotes class distribution with subscript $i$ when referring to class with label $i$, \emph{i.e.}, $ \nu_{i}\!\coloneq\!\mathbb{P}(y\!=\!i)$ \citep{barber2012bayesian}.
The noise transition matrix for annotator~$a$,
$T_{ij}^{a} \coloneq \pro(x_k^a\!=\!j \vert y_k\!=\!i)$ 
is the probability of incorrectly assigning label $j$ for the task $x_k^a$ that has true label $y_k=i$.
Let $x$ denote the worker annotations for a task with true label $y$ of which $y_\meth$ is an estimate from  method~$\meth$.
Correspondingly, $(y_{\MV})_k, (y_{\MAP})_k$ refer to the label obtained through majority vote and maximum \textit{a posteriori} aggregation respectively for task $k$, \emph{i.e.},
\begin{align*}
    &(y_{\MV})_k  = \underset{c \in \{1,\dots,C\}}{\operatorname{argmax} } \sum_{h=1}^{H} \mathds{1}\left[x_k^h=c\right], &\\
    &(y_{\MAP})_k =  \underset{c \in \{1,\dots,C\}}{\operatorname{argmax} } \mathbb{P}( y_k=c \mid x_k^{h_1}, x_k^{h_2}, \dots, x_k^{h_H}).&
\end{align*}
Oracle MAP refers to the case where both $\nu$ and the true $\p$ are known while estimated MAP uses the estimate $\hatp$ and $\tilde{\nu}$. 

\paragraph*{Problem statement.}
We assume binary classification tasks with labels~$y\in\{0,1\}$ \citep{patrini2017making,karger2014budget,chen2021structured}. Given a set of task annotations $\{x_1,x_2 \dots x_N\}$, with $x_k \in \{ 0,1\}^H$, we characterize the parameter space $(\nu,\p)$ where the gap in probabilities of MV and oracle MAP of recovering the true labels $\{y_1, y_2 \dots y_N\}$ \textit{i.e.}, the quantity $\left(\pro(y_{\MV}\!=\!y)\!-\!\pro(y_{\oMAP}\!=\!y)\right)$ vanishes.
We choose \oMAP~because it establishes the lower bound for best achievable error rate given true~$(\clim,\p)$ and is the minimizer of the expected error rate under the $\zone$ loss~\citep{li2014error}, see \cref{prop:omap_bound} below.
\begin{proposition}\label{prop:omap_bound}
The oracle MAP estimator minimizes the expected $\zone$ loss.
 \end{proposition}
 \begin{proof} Expected $\zone$ loss is
$ \exe[\loss_{\zone}(y=\oclass, y_\oMAP=\eclass| x, \theta) ] $
$= \exe [\ind( c\neq \eclass| x, \theta)]$
$ = \sum_{\oclass \in \nlabels}\ind(c \neq \eclass) \pro (\oclass|x,\theta)$ \\
$= 1 - \sum_{\oclass \in \nlabels}(\oclass = \eclass)\pro(\oclass|x,\theta) = 1 - \pro(\eclass \mid x,\theta).
$

Hence, the minimizer of the 0-1 loss is the same as the maximizer of the oracle posterior $\pro(\eclass\vert x)$ derived from true parameters~$\theta = (\p,\clim)$ and observed labels~$x$.
\end{proof}
Correspondingly, given that oracle $\MAP$ has the lowest achievable expected $\zone$ loss, we use it as the benchmark for MV in quantifying its effectiveness in recovering the true labels.
We note that our study, however, will focus not on expected loss but on the probability of MV recovering the true label for each sample or instance relative to that of \oMAP~\textit{i.e.},
\begin{equation}
\label{eq:main_eq}
    \pro(y_{\MV} = y) = \pro(y_{\oMAP} = y).
\end{equation}
The noise introduced in the annotations by workers is described through a transition matrix, $\p = \bigl( \begin{smallmatrix} \p_{00} & \p_{01}\\\p_{10} & \p_{11} \end{smallmatrix} \bigr)$. 
We begin with the assumption that $T$ is identical for all annotators and relax it later in \cref{sec:relaxing_reliability_assumption}.
A one-parameter noise model has symmetric error rate for the two classes with $\p_{00}=\p_{11}$ while a two-parameter model admits distinct error rates, \textit{i.e.}, $\p_{00} \neq \p_{11}$.
\cref{eq:main_eq} can be rewritten by observing that worker's votes can be modelled using a suitably defined Binomial distribution with parameters $(\nworkers, \p)$ (see \cref{lemma:noisetransition_mv,lemma:noise_transition_matrix_MAP}).
Similarly, $\p^\meth$ captures the class-confusion matrix over the aggregated labels recovered from applying the method~$\meth$ that are defined as follows for $\p^\MV$ and $\p^\oMAP$.

\begin{lemma}[Noise transition matrix $T^\MV$, also Lemma 2.1~\citep{wei2023aggregate}]\label{lemma:noisetransition_mv}
   Assume $\nworkers$ is odd and binary label $\oclass \in \{0,1\}$, the noise transition matrix of $\MV$ aggregated labels is:
\begin{equation}
\label{T_00_majority}
    T_{\oclass\oclass}^{MV} = \sum_{i=\ceil{\frac{\nworkers}{2}}}^{\nworkers} \binom{\nworkers}{i} T_{\oclass\oclass}^{i}\,(1-T_{\oclass\oclass})^{\nworkers-i}.
\end{equation}
\end{lemma}

\begin{lemma}[Noise transition matrix $T^\oMAP$\label{lemma:noise_transition_matrix_MAP}]   Assume $\nworkers$ is odd, binary label $\oclass \in \{0,1\}$, then noise transition matrix of MAP aggregated labels is:
\begin{align}
\label{eq:T_map}
   T^{\oMAP}_{cc}  = \sum_{k=\floor{A_c +1} }^H \binom{H}{k}T_{cc}^k (1-T_{cc})^{H-k}.
\end{align}
with $ A_c = \frac{\log{\frac{\nu_{\bar\oclass}}{\nu_\oclass}} + \nworkers \log{\frac{T_{\bar\oclass\bar\oclass}}{1-T_{\oclass\oclass}}}}{\log{\frac{T_{\oclass\oclass}T_{\bar\oclass\bar\oclass}}{(1-T_{\oclass\oclass})(1-T_{\bar\oclass\bar\oclass})}}}$ 
where $\bar\oclass = 1-c$ and $\nu_c,\nu_{\bar\oclass}$ are priors for classes $\oclass,\bar\oclass$.
\end{lemma}

\looseness -1 Details of their derivation are deferred to \cref{appendix:computeT_cc_omap}.
Note that the lower bound of the summation in $\p^\MV$ is from $\ceil{\frac{\nworkers}{2}}$.
This is because MV only requires that majority of workers vote for the correct class to recover the true label.
We may now rewrite \cref{eq:main_eq} as follows:
\begin{equation}
\label{request_ineq}
    T_{00}^{\MV} + T_{11}^{\MV}\frac{1 -\nu_0}{\nu_0} = T_{00}^{\oMAP} + T_{11}^{\oMAP}\frac{1 -\nu_0}{\nu_0}.
\end{equation}
Plugging in expressions of $\p^\MV$ and $\p^\oMAP$ from \cref{lemma:noisetransition_mv,lemma:noise_transition_matrix_MAP} into \cref{request_ineq} and analysing the various cases that arise, we can deduce how \MV~and \oMAP~vary across the parameter space.

\subsection{Characterisation for symmetric class noise}
\label{sec:sharedcoins}

We restrict our analysis in this section to one-parameter models with $\varrho\!=\!\p_{01}\!=\!\p_{10}$ for ease of analysis and defer the more general two-parameter case to \cref{sec:twocoins}.
Substituting $T^\MV$ and $T^\oMAP$ in \cref{eq:main_eq} with their expressions, we get the equality conditions as a function of parameters $\nu$ and $\nworkers$.
For all other cases that fail this condition \oMAP~is better than \MV~with a non-zero gap.

\begin{theorem}[MV optimality criterion for one-parameter $\p$ for binary tasks\label{thm:binary_classes_shared_coins}]
If the probability of a label flip, $\varrho\!=\!T_{01}\!=\!T_{10}$, is less than 0.5 and denoting by  $\nu\!=\!\pro(y=0)$ the class zero distribution we have that:
\begin{align}
    \pro(y_\oMAP=y)=\pro(y_\MV=y)
    ~~\textrm{iff}~~\varrho <\nu < 1-\varrho.
   \label{eq:onecoinresult}
\end{align}
\end{theorem}

$\varrho<0.5$ translates to better than random labellers that is the most general case of non-adversarial workers with no other constraints \citep{kargerMulticlass,li2019exploiting,bucarelli2023leveraging}. Notice that if $\varrho < \nu < 1-\varrho$ it also holds that $\varrho < 1- \nu < 1-\varrho$. The condition in this theorem can be rewritten as 
$ \frac{\varrho}{1- \varrho } <\frac{\nu}{1-\nu} < \frac{1-\varrho  }{\varrho}$. This indicates that the class imbalance ratio must fall within the range defined by the odds ratio associated with the noise rate of the dataset and its inverse. 
For a perfectly balanced dataset, this condition is always satisfied. However, as a dataset becomes more imbalanced, the maximum noise rate at which MV remains optimal decreases.
The proof of \cref{thm:binary_classes_shared_coins} can be found in \cref{list:exact_matches_proofs}.
The condition for equality of MV and MAP is a rather simple and elegant boundary in \cref{thm:binary_classes_shared_coins} beyond which their probability gap begins to widen.
It is interesting to note that if this bound is known to hold true for a given parameter configuration $(\clim, \p)$, then \MV's label estimate exactly matches that of oMAP for each individual instance, unlike \cref{prop:omap_bound} that applies on average, as is common in literature.
Correspondingly, this result can be used for optimal label estimates without access to an oracle if only we knew that $(\clim,\p)$ satisfies \cref{thm:binary_classes_shared_coins}.
\looseness -1 However, this verification of the bounds itself assumes access to the true parameters and we show in \cref{sec:estimated_quantities} how this can be achieved without the oracle but only estimators.

\subsection{Extension to asymmetric class noise}
\label{sec:twocoins}
We extend \cref{thm:binary_classes_shared_coins} to the more general case of $(\p_{01}\neq\p_{10})$ that corresponds to admitting noise models with unequal sensitivity and specificity.
It translates to workers confusing class $1$ to class $0$ and vice-versa with different error rates.
In this case, the necessary and sufficient condition for probability of \MV~recovering the same label as that of \oMAP~aggregation is provided by the following theorem:

\begin{theorem}[MV optimality criterion for two-parameter $\p$ for binary tasks]\label{thm:twocoins}
    Assuming better than random workers' reliability \citep{kargerMulticlass,li2019exploiting,bucarelli2023leveraging}, namely $\p_{00},\p_{11}>0.5$, labeling binary tasks, we have that:
\label{eq:suff_cond_nu}
  \begin{align} 
  \pro(y_\oMAP=y)&=\pro(y_\MV=y)~\textrm{if and only if } \nonumber \\ 
   \left(\frac{\delta_{c}}{\delta_{1-c}} \right)^{\frac{H}{2}} \frac{1}{\sqrt{\rho}}  &< \frac{ \nu_{1-c}}{\nu_{c}} < \left(\frac{\delta_{c}}{\delta_{1-c}} \right)^{\frac{H}{2}} \sqrt{\rho}. &
    \end{align}
where
$
    \rho = \frac{T_{00}T_{11}}{(1-T_{00})(1-T_{11})} \text{ and }     \delta_c = \frac{T_{cc}}{1-T_{\bar{c}\bar{c}}}.
$
\end{theorem}

The proof of this theorem can be found in \cref{list:exact_matches_proofs}.
The condition in \cref{eq:suff_cond_nu} holds for $\nu_0$ if and only if it holds for $\nu_1$. Therefore, it is sufficient to verify the condition for just one of them.
We emphasize that for $T_{01} = T_{10}$ we recover the condition of \cref{thm:binary_classes_shared_coins} and the inferences from \cref{sec:sharedcoins} generalize to this case.

\subsection{Verification without an oracle}
\label{sec:estimated_quantities}
\cref{thm:binary_classes_shared_coins,thm:twocoins} characterise the optimality criterion for MV under the one and two-coin case respectively. However, they require true $T$ and $\nu$ for verifiability, both of which are usually unknown. 
Nevertheless, with only their estimates, we will still be able to verify the bounds with high confidence. 
We define
$f(T)\!=\!\left(\frac{\delta_{c}}{\delta_{1-c}} \right)^{\frac{H}{2}}\!\frac{1}{\sqrt{\rho}}$,
$h(T)\!=\!\left(\frac{\delta_{c}}{\delta_{1-c}} \right)^{\frac{H}{2}}\!\!\!\sqrt{\rho}$ and $g(\nu)\!=\!\frac{\nu_{1-c}}{\nu_{c}}$.

\begin{theorem}
\label{th:estimated_quantities}
Given an approximation of the noise transition matrix \( \widetilde{T}, \) such that \( || T - \widetilde{T} ||_2 \leq \epsilon \) holds with probability at least \( 1 - \gamma \), we define \( \tilde{\nu} = \widetilde{T}^{-1} \hat{\nu} \), where \( \hat{\nu} \) is an approximation of the noisy label distribution. \\
If $\nu_0 $ and $\nu_1$ are so that  \( \eta \leq  \nu_c \leq  1 - \eta \), for \( 0< \eta < 1 \), 
$ \frac{1}{2}  < T_{cc} \leq 1- \xi  $ for $ 0.5 < \xi < 1$ and the following inequalities hold:
\[
\begin{cases}
g(\tilde{\nu}) - f(\widetilde{T}) > \psi + \epsilon \chi \\
h(\widetilde{T}) - g(\tilde{\nu}) >  \psi + 4 \epsilon \chi
\end{cases}
\]
then it follows that:
\begin{equation*}
f(T) < g(\nu) < h(T)
\end{equation*}
with probability \( 1 - 2 \gamma - 2e^{-\epsilon^2 \ntasks}\), where:\\ 
$\psi=\frac{\epsilon}{\lambda_{\text{min}}(\widetilde{T})} \left[ \frac{1}{\lambda_{\text{min}}(\widetilde{T}) - \epsilon} + \sqrt{C} \right]\frac{1}{\min(\eta, 1 - \eta)^2} $ and\\
$\chi=\sqrt{ \max ( \frac{1}{2},  \frac{H-1}{2} - H \xi)^2 +  \max ( \frac{1}{2},  \frac{H+1}{2}  -\xi H )^2 } $.
\end{theorem}

The proof of this theorem is in \cref{proof_theorem_est_quantities}. This bound depends on the quality of $\widetilde{T} \text{ and } \hat{\nu} $ via the parameter $\epsilon$. 
The conditions on $\eta$ and $\xi$ describe the amount of imbalance in class distribution and the worker noise that can be tolerated for the theorem to hold. 
Specifically,~$\eta$ is lower bound on the fraction of samples from the minority class and annotators reliability should be better than random by a margin of~$\xi$.

This theorem states that given estimates $(\hat{\nu}, \widetilde{T})$, it is feasible to determine with high probability if $(\clim, \p)$ satisfies \cref{thm:twocoins} enabling a practical method for the verification of MV's optimality.
Suitable estimators $(\hat{\nu}, \widetilde{T})$ from literature with the required estimation guarantee can be chosen.
For example,~\citet{bonald} use agreement between worker triplets from a pool of workers with better than random average reliability with at least three informative workers.
Similarly,~\cite{bucarelli2023leveraging} leverage pairwise worker agreement assuming all workers are better than random.

\begin{figure}[t]
    \centering
        \includegraphics[width=.65\linewidth]{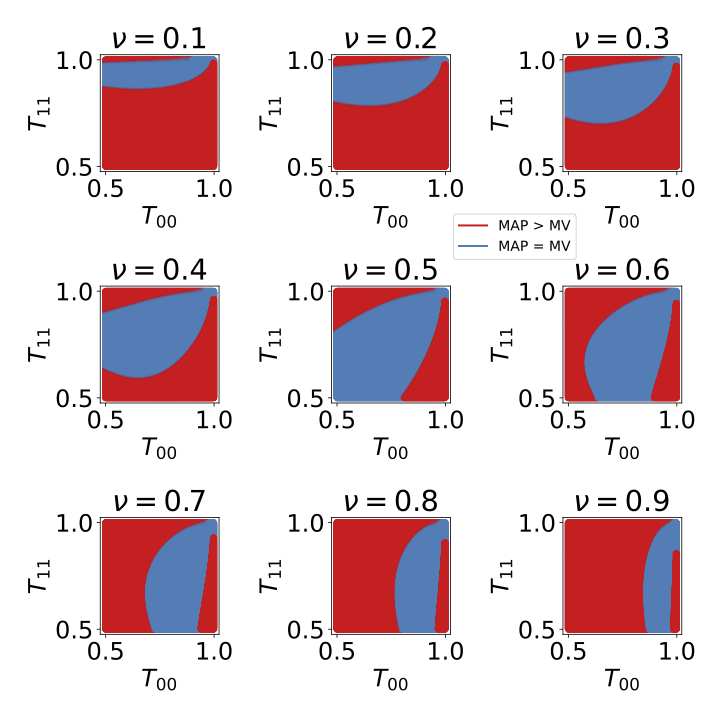}
    \caption{\looseness -1 Illustration of \cref{thm:twocoins} on simulations. We analyze the optimality of MV compared to oMAP 
    verifying if the condition in \cref{request_ineq} is satisfied for ten different $\nu_0$ values as $T_{11}$ and $T_{00}$ vary between $0.5$ and $1$. 
   Blue points denote where MV is equal to oMAP.}

   \label{fig:twocoin}
\end{figure}

\subsection{Beyond equal reliability condition}
\label{sec:relaxing_reliability_assumption}
In this section, we relax the equal reliability assumption and present two different settings.\\
\textbf{Uniformly perturbed $T$}.
In this setting annotators' noise transition matrices are not fixed but are sampled from a distribution. Specifically, for a fixed $\sigma$, the noise transition matrix of annotator $h$ is given by:
\begin{equation*}
T_h = \begin{bmatrix}
T_{00} & T_{01} \\
T_{10} & T_{11}
\end{bmatrix} + \sigma_h \begin{bmatrix}
-1 &1 \\
1 & -1
\end{bmatrix},\; \sigma_h \sim \text{Unif}[-\sigma, \sigma]
\end{equation*}
 We seek the conditions under which the expected performance of MV equals that of oMAP, or, more formally:
\begin{equation}
    \label{eq:check_in_expectation}
    \mathbb{E}[ \mathbb{P}(y^{\MV}=y)  ] = \mathbb{E}[ \mathbb{P}(  y^{\oMAP}=y) ].
\end{equation}
We take expectation over annotator distributions, so the derived conditions will depend only on the number of annotators and not on their specific matrices.
We can prove that, if $\sigma$ is small enough, the conditions of \cref{thm:twocoins} ensure  \cref{eq:check_in_expectation} holds. Specifically, we require:
$ \sigma\!\leq\!\frac{\log \rho}{H} R_c \min ( A_c\!-\! \floor{A_c},  1\!-\!A_c\!-\!\floor{A_c})$,\\
where $R_c = \left[  \frac{2}{T_{\bar{c}\bar{c}}} +   \frac{2}{ T_{c\bar{c}}} + \frac{1}{T_{cc}} + \frac{1}{ T_{\bar{c}c}} \right]^{-1}$ with $A_c$ is as defined in \cref{eq:T_map}. \\
\textbf{Annotators of two categories.}
Here we consider two groups of annotators A and B with different reliabilities with noise transition matrices respectively $T_A$ and $T_B$. 
\textit{W.l.o.g.} we can assume $|A|<|B|$ and $|A| < \ceil{\frac{ H}{2}}$.
Under the assumption $\rho_A = \rho_B$, (which holds for instance, when $(T_A)_{cc}=(T_B)_{\bar{c}\bar{c}}$), we derive conditions under which MV is equivalent to oMAP. These conditions are given by:
\begin{align*}
& \left( \frac{ (\delta_B)_c }{ (\delta_B)_{\bar{c}} } \right)^{ \frac{H}{2} } 
  \sqrt{ \frac{1}{\rho_B} } \zeta_{B,A} < \frac{ \nu_{\bar{c}} }{ \nu_c } \leq 
  \left( \frac{ (\delta_B)_c }{ (\delta_B)_{\bar{c}} } \right)^{ \frac{H}{2} } 
  \sqrt{ \rho_B } \zeta_{B,A}
\end{align*} 
with $\zeta_{B,A} = \left( \frac{ (\delta_B)_{\bar{c}} }{ (\delta_A)_{\bar{c}} } \right)^{ |A| }$ and  $ \delta_A$, $\delta_B$, $\rho_A$ and $\rho_B$ following the notation of \cref{thm:twocoins}. 
Notice that in the case of one set only we recover precisely the condition of \cref{thm:twocoins}.
For additional details refer to \cref{sec:appendix_expectation,sec:appendix_two_groups_workers}.

\subsection{Multiclass case with uniform noise}

Regarding the extension to the multiclass setting,  our current framework already provides the necessary ingredients to generalize the theory beyond the binary case. In particular, the results established in Theorem \ref{thm:binary_classes_shared_coins} naturally extend to multiclass classification.

In this setting, we have $C>2$.
Given the $C$-sized vector of the votes received for a sample by the various classes $n_c$ , MV selects class $c \in C$ iff.:
\begin{equation*}
n_c>n_j \forall j \neq c \Rightarrow n_c \geq n_j+1, 
\end{equation*}
while oMAP selects class $c \in C$ iff.:
\begin{align*}
\log(\nu_c) \! + \!\! \sum_{l=1}^C n_l log(T_{cl}) \! > \! \log(\nu_j) \! + \!\! \sum_{l=1}^C n_l log(T_{jl}) \, \forall j \neq c.
\end{align*}

\paragraph{Uniform Noise and Uniform Priors} Here, the $T$ matrix can be parametrized by a single parameter. 
$$
T_{ij} = \begin{cases}
1 - \rho & \text{if } i = j \\
\frac{\rho}{C-1} & \text{if } i \neq j
\end{cases} \quad \nu_i = \frac{1}{C} \quad \forall i \in \{1, \ldots, C\}
$$
Notice that in this case $\nu_j=\nu_c \; \forall j \neq c$. 
Apart from element $cc$ and $jj$ for all the other $\ell$, $T_{c\ell} = T_{j \ell}$, so the oMAP constraint becomes:
\begin{equation*}
\log\frac{(1-\rho)(C-1)}{\rho} \cdot (n_c - n_j) \geq 0
\end{equation*}
Since $\log\frac{(1-\rho)(C-1)}{\rho} \geq 1$ this becomes $n_c \geq n_j$. 
In this case oMAP and MV are equivalent (we are excluding ties). 
\paragraph{Uniform Noise and Non-Uniform Priors}

Now $\log\nu_j \neq \log\nu_c$. The constraint becomes:
\begin{equation*}
n_c \geq n_j + \frac{\log(\nu_j/\nu_c)}{\log((1-\rho)(C-1)/\rho)}
\end{equation*}
This means that MV is  equivalent to oMAP when:
\begin{equation*}
    1 < \frac{\nu_j}{\nu_c} \leq \left(\frac{(1-\rho)(C-1)}{\rho}\right) \quad \forall j \neq c, \; \forall c \in \{1, \ldots, C\}
\end{equation*}
This condition with $C=2$ gives the condition from \cref{thm:twocoins}.

Deriving and characterizing the equivalence conditions in the general multi-class setting represents a substantial theoretical contribution in its own right, as the constraints become considerably more complex with increasing $C$. While we provide the above derivation to demonstrate the extensibility of our framework, a complete characterization of the multi-class case, including the geometry of the feasible regions and optimality conditions, merits dedicated treatment and is left to future work. We emphasize that establishing these results rigorously even in the binary case represents a significant theoretical contribution.

\section{EXPERIMENTS}
\label{sec:results}
We empirically verify our results comparing \MV~against \oMAP~for various parameter configurations.
Simulations are used for validating \cref{thm:binary_classes_shared_coins,thm:twocoins} that require true data generating~$(\clim,\p)$.
\cref{th:estimated_quantities} that uses estimated~$(\tilde{\clim},\widetilde{\p})$ is verified on simulated data.
We compare estimates~$(\tilde{\clim}, \widetilde{\p})$ from different methods in the literature, namely, Dawid-Skene~\cite{dawid}, GLAD~\citep{glad},  MACE~\citep{mace},  IWMV~\citep{li2014error}, BWA \citep{li2019truth}, IAA~\citep{bucarelli2023leveraging} and LA \citep{yang2024lightweight}.
Annotator count $\nworkers$ is used as available for all real data while it is set to $\nworkers=3$ in all simulations unless stated otherwise. Here we report experiments with the same noise transition matrix shared by all annotators. Additional experiments that relax this assumption, as discussed in \cref{sec:relaxing_reliability_assumption}, are presented in \cref{sec:appendix_beyound_equal_rel}.

\paragraph*{Simulated symmetric noise with oracle  (RQ1).} 
We illustrate \cref{thm:binary_classes_shared_coins} in \cref{fig:main_theorem} marking points in the parameter space where MV performs optimally in blue distinguishing them from those where it underperforms  oMAP (in red). The plot obtained through simulations accurately reflects the theorem's conditions. 
Towards the left on $x$-axis are the low-noise regimes where MV is optimal even for skewed class balance.
As the noise increases to the right, MV's optimality is restricted to fewer~$\clim$ closer to~$0.5$.
This, however, is when we use the true generating noise that is unavailable for real-world cases. So, for comparison, we make similar plots with its estimate from two methods, IWMV and IAA, in \cref{fig:iwmv,fig:iaa} respectively.
Evidently, there is a distortion in the blue region where MV is optimal that grows with the estimation error.
While IWMV largely retains the shape, IAA distorts it further with a higher number of sub-optimal parameters falsely marked as optimal for MV.
We further inspect the actual difference in the probability scores of \MV~and \oMAP~from \cref{fig:main_theorem} on a heatmap in \cref{fig:heatmap}. 
Large magnitude differences are restricted to extreme cases of skewed class distribution with noisy labels on the top and bottom corners to the right. 
\looseness -1 Note that this is the worst-case scenario with $\nworkers=3$, and the blue region where MV is optimal grows larger with the size of the worker pool, however, with the caveat that we need exponentially more workers.
This is illustrated in \cref{fig:infinity} where we draw four distinct parameter configurations to plot the gap $\pro(\y_{\oMAP})-\pro(\y_{\MV})$ for increasing~$\nworkers$.
$(\clim,\p)$ for the orange curve satisfies \cref{thm:binary_classes_shared_coins} and is flat all through independent of~$\nworkers$ as it should be. 
MV is sub-optimal for the other three parameter configurations at lower $\nworkers$ with a gap relative to \oMAP~that vanishes asymptotically as predicted by~\cite{li2014error}.

\paragraph*{Simulated asymmetric noise (RQ1).} Similar visualisations for the two-parameter model with $\nworkers=3$ are plotted in \cref{fig:twocoin} for various values of $\nu$.
Blue colored $\{(\p_{00},\p_{11})\}$ pairs are MV optimal satisfying \cref{eq:suff_cond_nu} while red regions have MV underperforming oMAP.
Notice how the region where MV is optimal lying around the line $\p_{00}=\p_{11}$ are maximum for $\nu=0.5$ shrinking as we move away to $0.1$ or $0.9$ similar to the one-parameter model in \cref{fig:main_theorem}.

\begin{table*}[h]
\begin{subtable}[t]{.67\linewidth}
    \centering
    \resizebox{\linewidth}{!}{
    \begin{tabular}{rlcccc}
    \toprule
        Dataset & $\nu$ & $\ntasks$ & $v_{\textrm{tr}}$ & $\nworkers$ & $v_{\textrm{wr}}$\\
         \midrule
        SP & [49.9, 50.1] & 4999 & 203 & 5  & 5 \\
        SP\_amt & [49, 51] & 500 & 143 & 20 & 500 \\
        ZC\_in & [78.3, 21.7] & 2040 & 25 & 5 & 2125\\
        MS & [11, 10, 10,  9, 10, 10,  10, 10, 11,  9] & 700 & 44 & 4 & 67 \\
        CF\_amt & [19, 23, 24, 31, 3] & 300 & 110 & 20 & 55\\
        D-Product & [87.8, 12.2] & 8715 & 176 & 3 & 140\\
        \bottomrule
    \end{tabular}
    }
    \caption{Real data statistics.}
    \label{table:stats}
\end{subtable}
\hfill
\begin{subtable}[t]{.3\linewidth}
    \centering
    \resizebox{\linewidth}{!}{
    \begin{tabular}{rlcc}
    \toprule
        Dataset & $\nu$ & $T_{00}$ & $T_{11}$\\
         \midrule
        $\alpha$-Data & [20, 80] & 0.7 & 0.7\\
        $\beta$-Data & [50, 50]& 0.9 & 0.9\\
        $\gamma$-Data & [30, 70] & 0.6 & 0.75\\
        $\delta$-Data & [90, 10] & 0.6 & 0.9\\
        \bottomrule
    \end{tabular}
    }
    \caption{Synthetic data statistics.}
    \label{tab:synthparams}
\end{subtable}
\caption{Data statistics with number of tasks, annotators, label class distribution $\nu$ (\%), task rank ($v_{\textrm{tr}}$) and worker rank ($v_{\textrm{wr}}$) that are respectively, the number of annotators who provided labels for each sample and the number of samples each annotator was assigned to label. In \cref{tab:synthparams} $\alpha,\beta$-Data from one-parameter models satisfy and fail \cref{eq:onecoinresult} respectively. $\gamma,\delta$-Data from two-parameter models satisfy and fail \cref{eq:suff_cond_nu} respectively.}
\end{table*}

\begin{figure*}[ht] 
\centering 
    \begin{subfigure}[t]{0.2\linewidth}
        \includegraphics[width=\linewidth]{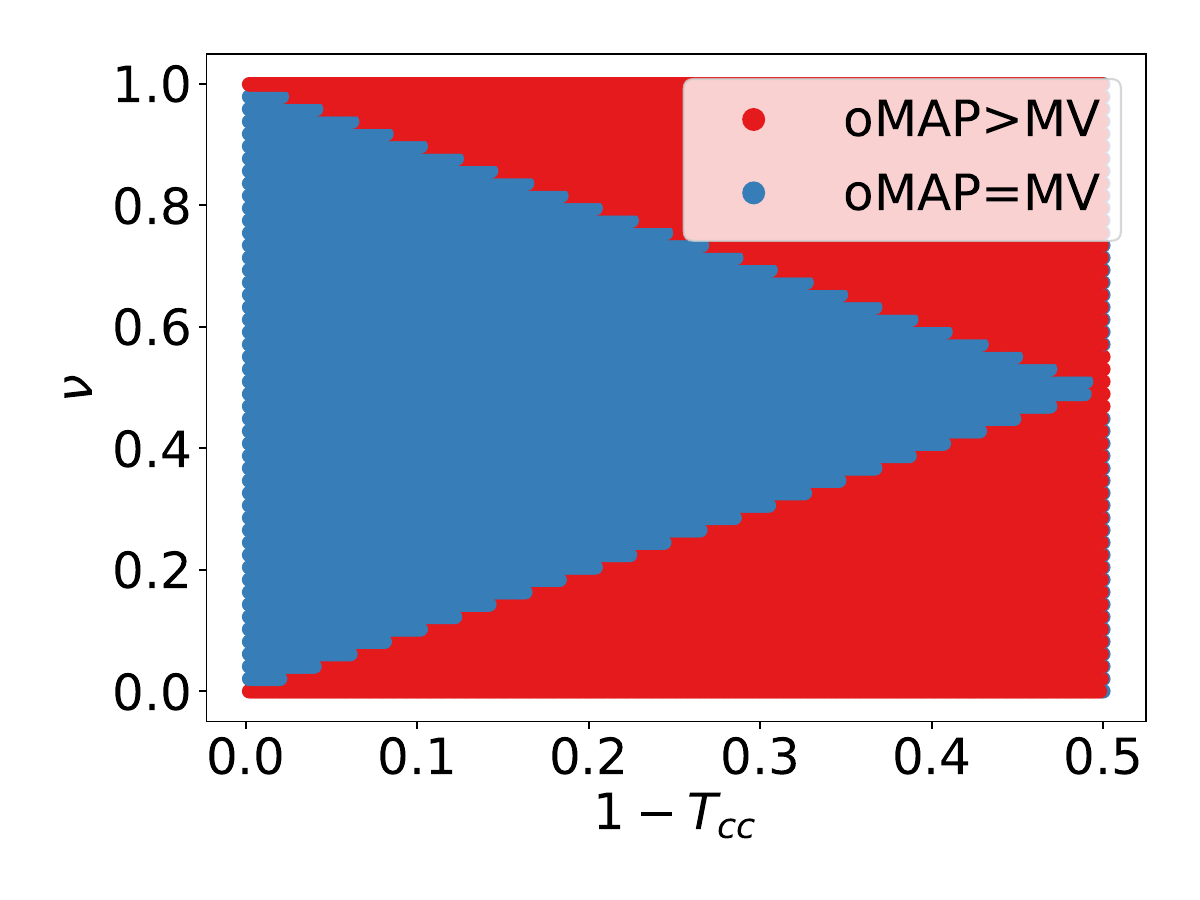}
        \subcaption{
        $\pro(\y_{\oMAP}) = \pro(\y_{\MV})$}
        \label{fig:main_theorem}
    \end{subfigure}
    \begin{subfigure}[t]{0.2\linewidth}
        \includegraphics[width=\linewidth]{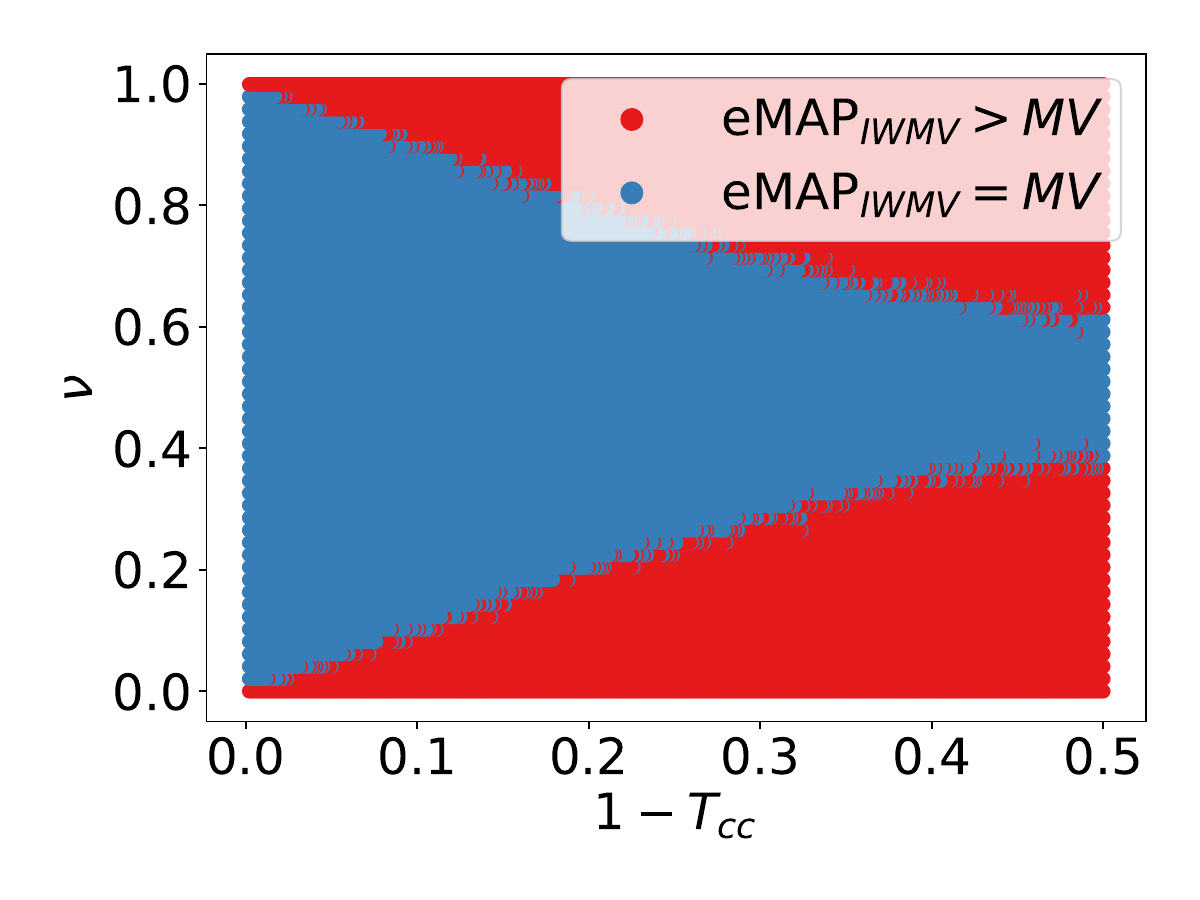}
        \caption{$\pro(\y_{\text{IWMV}}) = \pro(\y_{\MV})$}
        \label{fig:iwmv}
    \end{subfigure}
    \begin{subfigure}[t]{0.2\linewidth}
        \includegraphics[width=\linewidth]{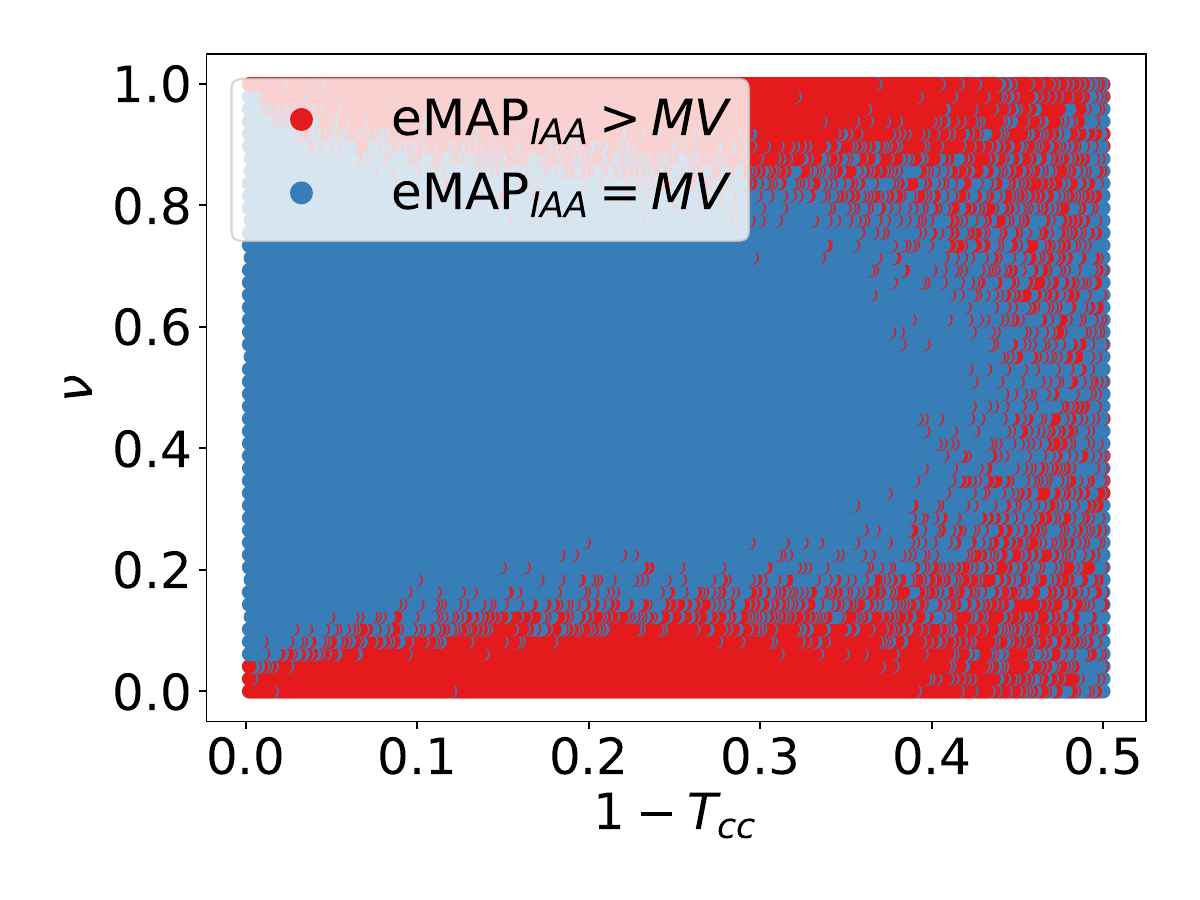}
        \caption{$\pro(\y_{\text{IAA}}) = \pro(\y_{\MV})$}
        \label{fig:iaa}
    \end{subfigure}
\resizebox{0.75\linewidth}{!}
    {
    \begin{subfigure}[t]{0.25\linewidth} \includegraphics[width=\linewidth]{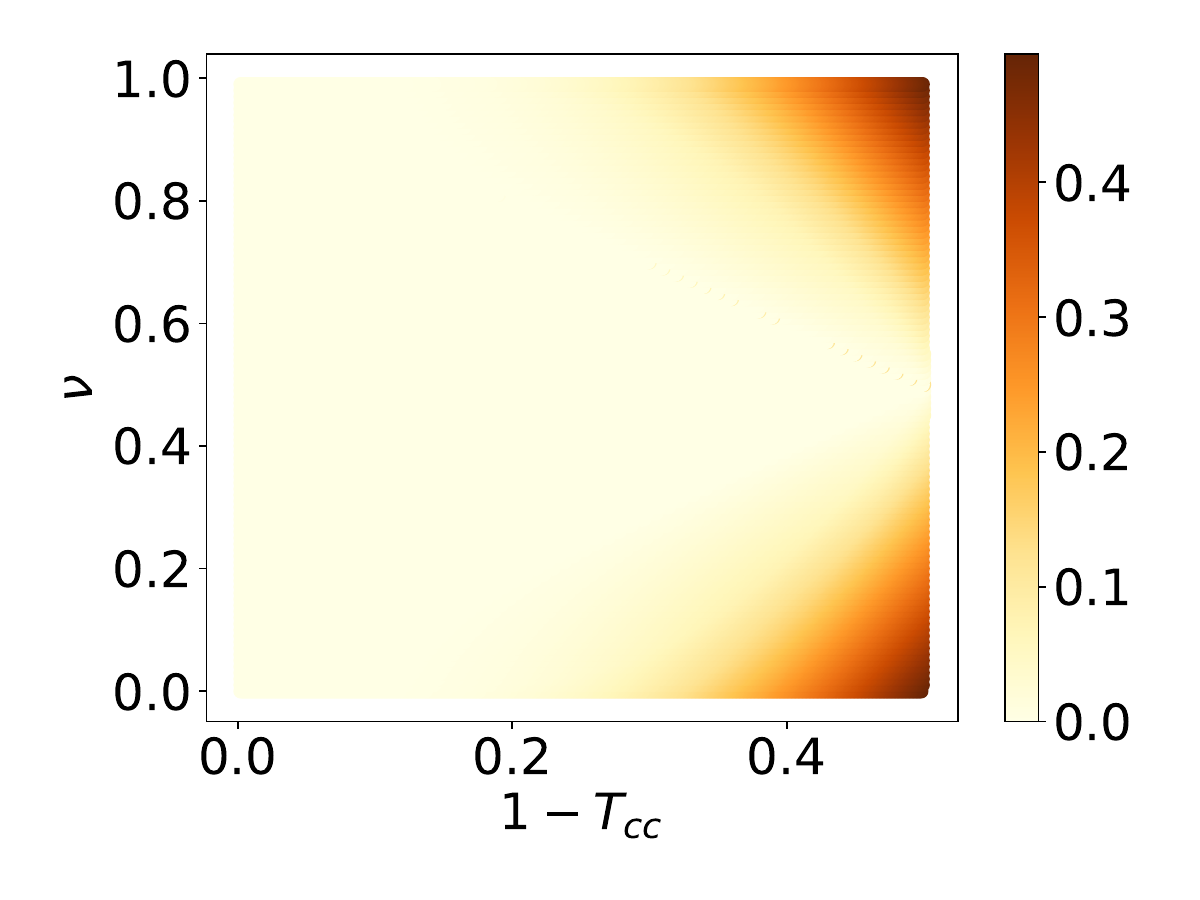} \caption{ $\pro(\y_{\oMAP})\!-\!\pro(\y_{\MV})$ heatmap } \label{fig:heatmap} \end{subfigure} 
    \hfill 
    \begin{subfigure}[t]{0.25\linewidth} \includegraphics[width=\linewidth]{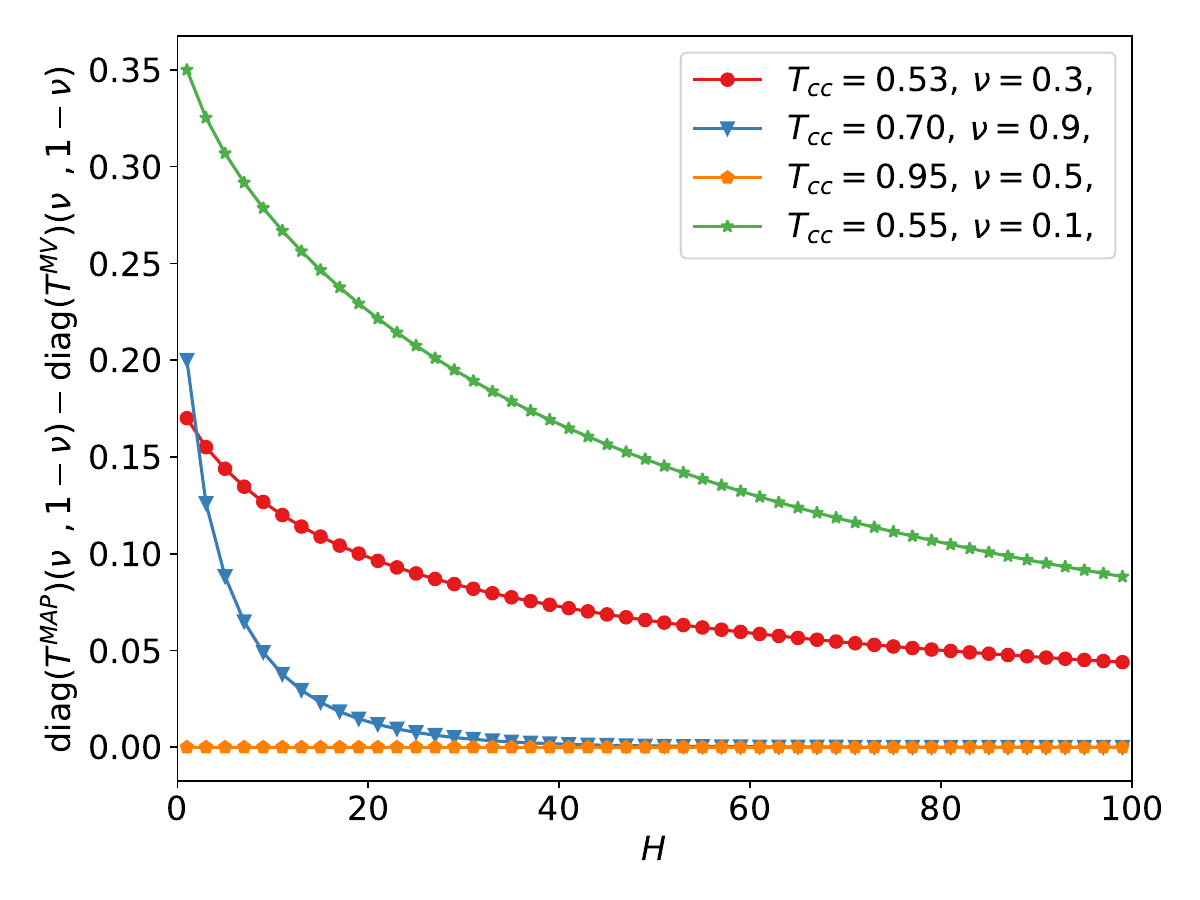} \caption{$\pro(\y_{\oMAP})\!-\!\pro(\y_{\MV})$ vs $\nworkers$. } \label{fig:infinity} \end{subfigure} 
    \hfill 
    \begin{subfigure}[t]{0.25\linewidth} \includegraphics[width=\linewidth]{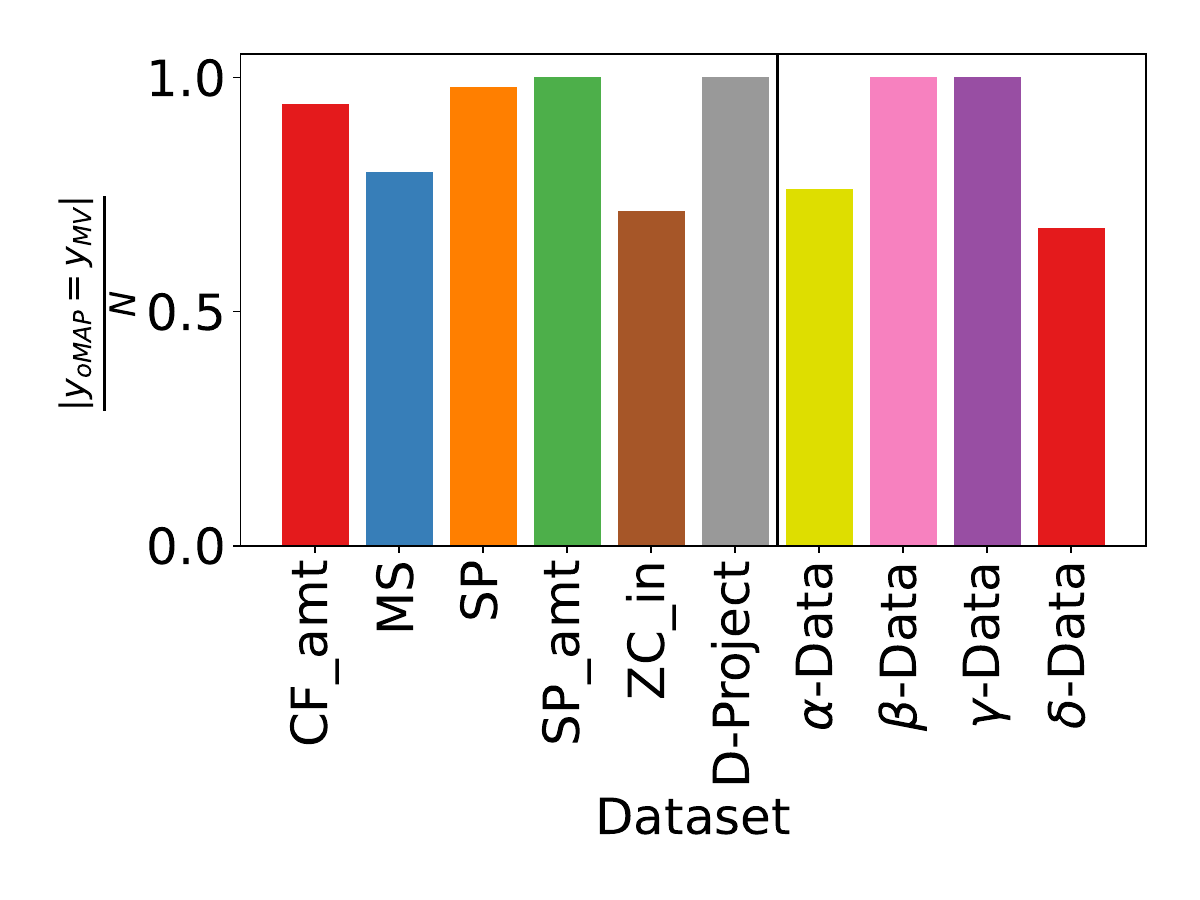} \caption{ $\sum\ind[y_{\oMAP}\!=\!y_\meth]$ histogram} \label{fig:histogram} \end{subfigure}
}
\caption{
Illustrations on simulated data of \cref{thm:binary_classes_shared_coins} comparing MV to \oMAP~are in~(\subref{fig:main_theorem}) with similar plots for IWMV in~(\subref{fig:iwmv}) and for IAA in~(\subref{fig:iaa}).
Heatmap in~(\subref{fig:heatmap}) is of MV vs \oMAP~plot in~(\subref{fig:main_theorem}).
Curves in~(\subref{fig:infinity}) show how that probability gap falls as~$\nworkers$ grows for four distinct parameter settings of which only orange one satisfies optimality condition for MV. 
Histogram~(\subref{fig:histogram}) shows the percentage of labels aggregated using MV equal to the the ones aggregated using oMAP.
}\label{fig:one-coin}
\end{figure*}

\paragraph*{Quantitative evaluation on simulation (RQ1).} We simulate data for four different configurations, see \cref{tab:synthparams}.
Two for one-parameter models of which $\alpha$-Data fails the optimality condition for MV in \cref{eq:onecoinresult} while $\beta$-Data satisfies it.
Similarly, for the two-parameter case $\gamma$-Data satisfies \cref{eq:suff_cond_nu} while $\delta$-Data does not.
\cref{table:main_results} reports the accuracy of various methods on these data.
As predicted by the theory, MV recovers labels optimally for $\beta,\gamma$-Data while it underperforms for $\alpha,\delta$-Data.
\cref{fig:histogram} plots the histogram with the fraction of instances for which the labels recovered using MV is exactly the same as that of oracle MAP.
There is an exact match in all samples of $\beta,\gamma$-Data  emphasizing instance optimality of the results.
Note that the shortfall in oMAP from perfect accuracy score is the irreducible error and it is larger if the noise is higher. Moreover, we can notice that in the settings for which oMAP equals MV, it implies that nearly every other aggregation method performs equally well in terms of effectiveness.  This again suggests that in such scenarios, employing a straightforward approach like MV, which only relies on individual sample-wise information, performs comparably to employing more complex methods that leverage data relationships across the entire dataset for aggregation.

\begin{table*}[t!]
    \centering
     \resizebox{0.9\linewidth}{!}
    {
    \begin{tabular}{lcccccc|cccc}
    \toprule
         & SP \xmark & SP\_amt \cmark & ZC\_in \xmark & MS \textit{?} & CF\_amt \textit{?} & D-Product \xmark & $\alpha$-Data \xmark & $\beta$-Data \cmark & $\gamma$-Data \cmark & $\delta$-Data \xmark \\         \midrule
         DS (a) & 0.502 & 0.632 & 0.575 & 0.103 & 0.257 & \textbf{0.940} & \underline{0.802} & 0.971 & 0.785 & 0.831\\
        GLAD (b) & 0.502 & 0.630 & 0.611 & 0.096 & 0.263 & 0.927 & 0.786 & 0.971 & 0.785 & 0.678 \\ 
        MACE (c) & 0.503 & \underline{0.632} & 0.631 & 0.097 & 0.270 & 0.695 & 0.786 & 0.971  & 0.785 & 0.678\\
        IWMV (d) & \underline{0.904} & {0.942} & 0.750 & \textbf{0.799} & 0.857 & \underline{0.928} & 0.807 & 0.971 & 0.785 &0.678\\
        BWA(e) & \textbf{0.917} & 0.942 & \underline{0.765} & 0.786 & 0.857 & 0.919  & 0.786 & 0.971 & 0.785 & \underline{0.901}\\ 
    LA(f) & 0.903 & 0.942 & 0.749 & \underline{0.797} & 0.857 & 0.926 & 0.784 & 0.971 & 0.785 & 0.678\\
        IAA (g) & {0.882} & {0.942} & {0.744} & 0.703 & 0.857 & 0.905 & 0.786 & 0.971 & 0.785 & 0.678\\
        MV (h) & 0.887 & 0.942 & 0.740 & 0.707 & \underline{0.857} & 0.905 & 0.786 & 0.971 & 0.785 & 0.678\\ 
        
        oMAP & 0.891\textsuperscript{abcfh} & \textbf{0.942}\textsuperscript{abc} & \textbf{0.784}\textsuperscript{abcdefgh} & 0.749\textsuperscript{abcdefgh} & \textbf{0.857}\textsuperscript{abc} & 0.905\textsuperscript{abcdef} & \textbf{0.841}\textsuperscript{abcdefgh} & 0.971 & 0.785 & \textbf{0.915}\textsuperscript{abcdefgh}\\ 
        \bottomrule
    \end{tabular}
    }
        \caption{
        Accuracy ($\uparrow$ is better) of different aggregation methods as measured against gold labels. 
        Real data followed by synthetic from left to right. $\cmark \text{ or } \xmark$ indicate if \cref{thm:twocoins} is satisfied or not, while \textit{?} refers to multi-class datasets. Super-scripted letters indicate the methods statistically significant with respect to oMAP, determined by Wilxocon tests ($p < 0.05$) with Bonferroni correction. Best results in \textbf{bold} and second-best \underline{underlined}.
        }
    \label{table:main_results}

\end{table*}

\paragraph*{Results on real-world data (RQ2).} We use benchmark crowd-sourced datasets from Active Crowd Toolkit \citep{venanzi} and Crowdsourcing datasets repository\footnote{\url{https://dbgroup.cs.tsinghua.edu.cn/ligl/crowddata/}} for our evaluation. 
Data statistics including the number of samples, annotators, classes, and their distributions are reported in \cref{table:stats}.
We do not have access to oracle's true $\p$ of annotators, hence, we use a widely adopted approach~\citep{xia2019anchor}  leveraging carefully curated gold labels or anchor points.
We compute element $i,j$ of the matrix $\p$ as the fraction of anchor points with gold label $i$ that were marked as class $j$.
This approximation assumes noise characteristics are consistent across workers and assigns to each one the average noise estimated from all annotations.
We use this as the oracle $T$ in experiments and report its numbers as $\oMAP$ in \cref{table:main_results}.
Predicted labels are compared against gold labels and accuracy is reported in \cref{table:main_results}.

Classes are balanced in SP\_amt and $T_{00}=T_{11}=0.64$ satisfying \cref{eq:onecoinresult} for optimality of MV.
As a consequence, MV and Anchor oMAP achieve the same accuracy of $0.942$ with exact match of labels for all instances as seen in histogram of \cref{fig:histogram_samplewise}.
Data in SP also has balanced class distribution but $\p_{00}(=0.57) \neq\p_{11}(=0.65)$.
Correspondingly, MV achieves an accuracy of $0.887$ that is slightly lower than that of Anchor oMAP at $0.891$ with a small drop from perfect match in all instances.
ZenCrowd\_in (ZC\_in) and D-Product have skewed class balance with $80$-$20$ and $87$-$13$ data splits respectively.
However, ZC\_in with $\p_{00}(=0.58) \neq\p_{11}(=0.51)$ satisfies neither condition for MV optimality while D-Product with $\p_{00}(=0.72) \neq\p_{11}(=0.55)$ satisfies \cref{eq:suff_cond_nu}.
Correspondingly, D-Product has MV accuracy on par with that Anchor oMAP, such as ZC\_in.
GLAD and MACE could be underperforming due to model misspecifications given that they make very specific assumptions about workers.
It's important to note that Anchor oMAP isn't consistently the top-performing method, primarily because our estimation of matrix $\p$ is merely an approximation. \looseness -1 
\cref{appendix:computeT_cc_omap} shows the noise transition matrices estimated from the data using this method.
CF\_amt and MS (denoted with \textit{?} in \cref{table:main_results}) are both multi-class data and, hence, the results from this work do not apply. However, for CF\_amt data we empirically observe MV match the label recovery accuracy of Anchor oMAP, while it is sub-optimal for MS data.

\begin{figure}[t!]
    \centering
    \includegraphics[width=\linewidth]{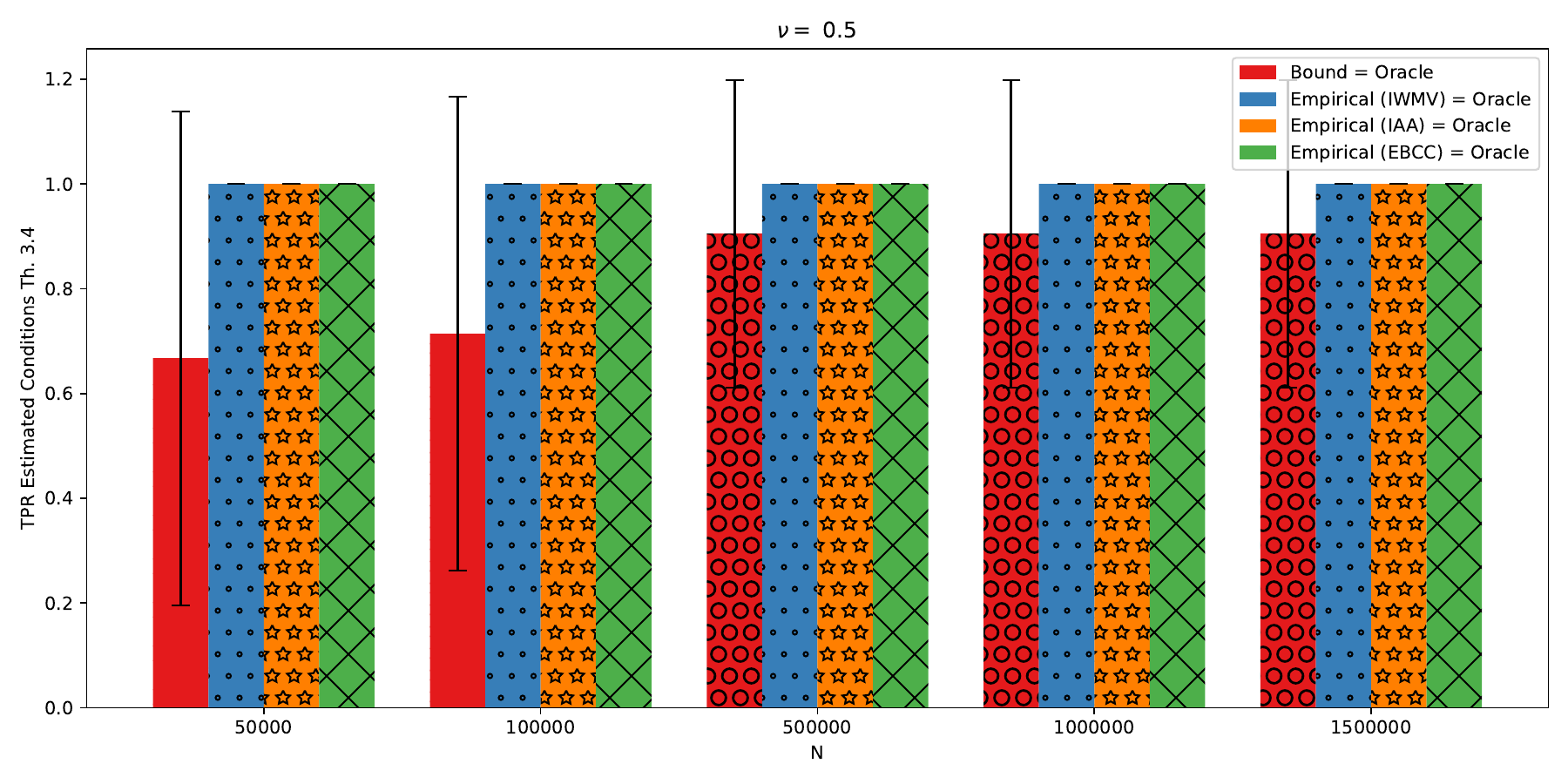}
    \caption{
    Non-red bars show the fraction of experiments where verification of \cref{thm:twocoins} with estimated parameters from the candidate methods aligns with that of \cref{thm:twocoins} using the true $(\clim,\p)$, considering cases where the theorem is verified with true parameters.
    Red bars indicate cases where \cref{th:estimated_quantities} aligns with \cref{thm:twocoins} using true parameters.
    Synthetic data have various sample sizes~$N$, and the average True Positive Rate is plotted over multiple $T$ values, with $\nu_0=0.5$. }
    \label{fig:histogram_samplewise}
\end{figure}

\paragraph*{Optimality verification (RQ2).}
To apply our theorems in practical settings, one can verify optimality conditions by directly plugging in estimates $(\tilde{\clim},\widetilde{\p})$ from any method directly into \cref{thm:twocoins}. 
\cref{th:estimated_quantities} provides guarantees on this approach, showing that when the assumptions of \cref{th:estimated_quantities} are met, the correctness of the MV optimality check is assured with high probability.
Empirically, we observed using estimates from three different methods (EBCC, IWMV and IAA) that even if the conditions of \cref{th:estimated_quantities} are not strictly met, checking the optimality conditions by substituting estimated quantities into \cref{thm:twocoins}
often suffices. In these cases, the estimated conditions frequently match those of the unknown true quantities (see \cref{fig:iwmv,fig:iaa}).
\cref{fig:histogram_samplewise} illustrates this trade-off. The plot displays bars representing the average number of instances where the estimated quantities (each from different estimation method) are plugged into \cref{thm:twocoins} and match the condition computed with \cref{thm:twocoins} using the true, but unknown, noise rate parameters and data distribution. We consider only cases
where the theorem is verified with true parameters, namely we computed the True Positive Rate (TPR). These are compared with red bars illustrating the cases where \cref{th:estimated_quantities} aligns with \cref{thm:twocoins} when calculated based on these true parameters and data distribution. Although \cref{th:estimated_quantities} achieves perfect sensitivity, it often rejects MV optimality in cases where it actually holds, as the conditions it requires for confirmation are quite strict.
We further show similar observation on a simulation with all four cases of true/false-positive/negative from three different estimation methods (IWMV, IAA and EBCC) on various sample sizes in \cref{appendix:table_our_estimation_approach} from \cref{appendix:sec_other_exps}.
For reference, we also report accuracy when aggregated using the three methods along with run time of each which is lowest for MV as expected.

\section{DISCUSSION AND CONCLUSION}
\label{sec:conclude}
MV is often considered naive for crowdsourced labeling and, before this paper, its optimality remained unexplored. 
It was unknown whether it could reach the theoretically optimal oMAP bound, which uses complete knowledge of annotators' noise, and the conditions under which it would match that optimal bound.
We address this problem for identical workers annotating binary tasks.
We identify sufficient and necessary conditions for MV to be equivalent to oMAP.
Our major finding is that when classes are well balanced, MV  is robust to a wide range of worker noise levels exactly matching the oMAP bound and as the class distribution skews, worker reliability becomes increasingly important for MV to be optimal.
Furthermore, we also extended our findings to some of the scenarios involving annotators with varying reliabilities.

Experiments show that verifying the conditions using estimated quantities is often sufficient, thereby making our findings applicable to real-world scenarios.
Finding a suitable aggregation method currently relies on evaluating multiple hypotheses through trial-and-error on a \textit{dev set} with \textit{expert} generated \textit{gold} labels that suffer from the same annotator uncertainty.
It is particularly challenging for critical tasks like medical diagnostics where the accuracy of individual labels is paramount that the instance-level optimality derived in this work guarantees.
This also serves as a guide for future research to focus effort on those regions of the parameter space where MV is not optimal.
\paragraph{Limitations} 
Areas for improvement include extending to arbitrary noise models including adversarial workers and multiclass tasks that are theoretically challenging cases to analyse but have wider applicability.

\acknowledgments{Maria Sofia Bucarelli acknowledges partial support from the French government, through the 3IA Cote d’Azur Investments in the project managed by the National Research Agency (ANR) with the reference number ANR-23-IACL-0001. 

Antonio Purificato and Fabrizio Silvestri acknowledge project FAIR (PE0000013), under the MUR National Recovery and Resilience Plan funded by the European Union - NextGenerationEU.}

\bibliographystyle{apalike}
\bibliography{sample}

\section*{Checklist}

\begin{enumerate}

  \item For all models and algorithms presented, check if you include:

  \begin{enumerate}

    \item A clear description of the mathematical setting, assumptions, algorithm, and/or model. [Yes]. The manuscript contains a clear description of the mathematical setting and assumptions in \cref{sec:problem}.
   \item An analysis of the properties and complexity (time, space, sample size) of any algorithm. 
   
   [Yes]. The properties of each algorithm are analyzed in \cref{sec:results}.
   \item (Optional) Anonymized source code, with specification of all dependencies, including external libraries. 
   
   [Yes]. The code can be downloaded from the Additional Material. All the external libraries are clearly specified.

  \end{enumerate}

  \item For any theoretical claim, check if you include:

  \begin{enumerate}
   \item Statements of the full set of assumptions of all theoretical results. 

   [Yes]. Each theoretical claim contains as statement with the full set of assumptions.
   \item Complete proofs of all theoretical results. 

   [Yes]. Each theoretical claim is related to a proof, which can be found in the main manuscript or in the Appendix.
   \item Clear explanations of any assumptions. 

   [Yes]. Each assumption is clearly explained before the corresponding claim.
 \end{enumerate}

  \item For all figures and tables that present empirical results, check if you include:

   \begin{enumerate}
   \item The code, data, and instructions needed to reproduce the main experimental results (either in the supplemental material or as a URL). 

   [Yes]. The code to reproduce all the experiment is fully available and all the istructions are clearly explained in the README.
   \item All the training details (e.g., data splits, hyperparameters, how they were chosen). 

   [Yes]. All the experimental details can be found or in  \cref{sec:results} from the main manuscript or in \cref{sec:sup_exps}.
         \item A clear definition of the specific measure or statistics and error bars (e.g., with respect to the random seed after running experiments multiple times). 

         [Yes]. All the statistics are clearly explained in the Appendix.
         \item A description of the computing infrastructure used. (e.g., type of GPUs, internal cluster, or cloud provider). 

         [Yes]. The description of the configuration is given in \cref{sec:sup_exps} from the Additional Material.
 \end{enumerate}

  \item If you are using existing assets (e.g., code, data, models) or curating/releasing new assets, check if you include:

   \begin{enumerate}
   \item Citations of the creator If your work uses existing assets. 

   [Yes]. All the citations to code or datasets creators are in the text.
   \item The license information of the assets, if applicable. 

   [Yes]. The license of the code or datasets are clearly stated.
   \item New assets either in the supplemental material or as a URL, if applicable.

   [Yes]. The code the reproduce all the experiments is released as Supplementary Material.
   \item Information about consent from data providers/curators. 

   [Not Applicable]. All the used data don't require consent from data curators.
   \item Discussion of sensible content if applicable, e.g., personally identifiable information or offensive content. 
   
   [Not Applicable]. There is no sensible content from the experimental data.
 \end{enumerate}

  \item If you used crowdsourcing or conducted research with human subjects, check if you include:

   \begin{enumerate}
   \item The full text of instructions given to participants and screenshots. 

   [No]. Also if this paper is about crowdsourcing, no participants were involved. Only synthetic or publicly available datasets were used.
   \item Descriptions of potential participant risks, with links to Institutional Review Board (IRB) approvals if applicable.

   [No]. Also if this paper is about crowdsourcing, no participants were involved.
   \item The estimated hourly wage paid to participants and the total amount spent on participant compensation. 

   [No]. Also if this paper is about crowdsourcing, no participants were involved.
 \end{enumerate}

\end{enumerate}

\appendix

\onecolumn

\section{Proofs}

\subsection{Proof of \cref{lemma:noise_transition_matrix_MAP}. Computation of $T_{cc}^{MAP}$}
\label{appendix:computeT_cc_omap}
Given that the analysis remains consistent for both $T_{00}$ and $T_{11}$, let us proceed by considering a general $T_{cc}$, where $c \in \{0,1\}$.
\begin{equation}
    T_{cc}^{oMAP} = \mathbb{P}({y_{oMAP}=c | y=c}) = \mathbb{P}({\text{argmax} \, p_i = c})
\end{equation}

Let us define  $m \sim \text{Bin}(H,T_{cc})$ and  $H-m \sim \text{Bin}(H,1-T_{cc})$. We recall that the element $i$-th of the posterior probabilities vector, $ p_{i|\tilde{c}} = \nu_{i} \prod_{c'=1}^C T_{i c' }^{n_{c'|\tilde{c}}}$
where $n_{i|\tilde{c}}$ is the number of annotators  that vote class $i$ given that the true class is $\tilde{c}$.\\ 
The ${\text{argmax}}_{i \in \{1,0 \}}$ of $p_i$ is $c$ if $ p_{c|c} > p_{1-c|c}$ i.e.:
\begin{equation}
    \nu_c (T_{cc})^m{(1-T_{cc})}^{H-m} > \nu_{1-c}(1-T_{1-c 1-c})^m(T_{1-c 1-c})^{H-m}
\end{equation}
Where $m $ is the number of annotators that vote class $c$ given that the true class is $c$. 
\begin{equation}
    {\left(\frac{T_{cc}}{1-T_{1-c1-c}}\right)}^m {\left(\frac{1-T_{cc}}{T_{1-c 1-c}}\right)}^{H-m} > \frac{\nu_{1-c}}{\nu_c}
\end{equation}

\begin{equation*}
    {\left(\frac{T_{00}T_{11}}{(1-T_{00})(1-T_{11})}\right)}^{m} > \frac{1-\nu}{\nu} {\left(\frac{T_{11}}{1-T_{00}}\right)}^H
 \iff 
    m >  \frac{\log{\frac{1-\nu}{\nu}} + H \log{\frac{T_{11}}{1-T_{00}}}}{\log{\frac{T_{00}T_{11}}{(1-T_{00})(1-T_{11})}}}
\end{equation*}

By defining:
\begin{align*}
    A_c = \frac{\log{\frac{\nu_{1-c }}{\nu_c}} + H \log{\frac{T_{1-c,1-c}}{1-T_{cc}}}}{\log{\frac{T_{cc}T_{1-c,1-c}}{(1-T_{cc})(1-T_{1-c,1-c})}}}
\end{align*}

We obtain that, if $A_{c} \notin \mathbb{N}$: 
\begin{align*}
    T^{oMAP}_{cc} &= \mathbb{P}({y_{oMAP}=c | y=c}) = \mathbb{P}(m > A_c )  = \sum_{k=\ceil{A_c} }^H \binom{H}{k}T_{cc}^k (1-T_{cc})^{H-k}.
\end{align*}

While if $A_c \in \mathbb{N}$:
\begin{align*}
    T^{oMAP}_{cc} &= \mathbb{P}({y_{oMAP}=c | y=c}) = \mathbb{P}(m > A_c)  = \sum_{k=A_c +1 }^H \binom{H}{k}T_{cc}^k (1-T_{cc})^{H-k}.
\end{align*}

For the sake of simplicity, we opt to exclude the scenario where $A_c \in \mathbb{N}$.

\subsection{Possible scenarios for exact matches. Proof of  \cref{thm:twocoins,thm:binary_classes_shared_coins}}
\label{list:exact_matches_proofs}
We want to solve \cref{eq:main_eq} that we rewrite here to make the reading easier: 

\begin{equation}
\begin{multlined}
\label{eq:initial}
    \mathbb{P}(y_{MV} = y | y=0)\mathbb{P}(y=0) + {P}(y_{MV} = y | y=1)\mathbb{P}(y=1) =  \\ \mathbb{P}(y_{oMAP} = y | y=0)\mathbb{P}(y=0) + \mathbb{P}(y_{oMAP} = y | y=1)\mathbb{P}(y=1)
\end{multlined}
\end{equation}

We switch to \cref{request_ineq}, namely:
\begin{equation*}
    T_{00}^{\MV} + T_{11}^{\MV}\frac{1 -\nu}{\nu} = T_{00}^{\oMAP} + T_{11}^{\oMAP}\frac{1 -\nu}{\nu}.
\end{equation*}

The following list elaborates on the possible scenarios in for sign of the inequality in \cref{request_ineq}:
\begin{enumerate}

    \item If $T_{00}^{oMAP} > T_{00}^{MV}$ and $T_{11}^{MAP} >  T_{11}^{MV} \Rightarrow \mathbb{P}(y_{oMAP} = y) > \mathbb{P}(y_{MV} = y)$;
    \item If $T_{00}^{oMAP}> T_{00}^{MV}$ and $T_{11}^{oMAP} < T_{11}^{MV}$ \\ $T_{00}^{oMAP} - T_{00}^{MV} > (T_{11}^{MV} - T_{11}^{MAP})\left(\frac{1-\nu}{\nu}\right)$ $\Rightarrow \mathbb{P}(y_{MAP} = y) > \mathbb{P}(y_{MV} = y)$;
    \item If $T_{11}^{MAP}> T_{11}^{MV}$ and $T_{00}^{oMAP} < T_{00}^{MV} \Rightarrow$  $T_{00}^{MV} - T_{00}^{MV} < (T_{11}^{oMAP} - T_{11}^{MV})\left(\frac{1-\nu}{\nu}\right)$;
    \item Otherwise $\mathbb{P}(y_{oMAP} = y) < \mathbb{P}(y_{MV} = y)$;
\end{enumerate}

In general if $x, y \in \mathbb{R}$  $x\leq  y $ it's sufficient to have  $ \ceil{x} \leq  \ceil{y}$.
If one wants the strict inequality for the ceiling, a sufficient condition is $x \leq y-1$.

From the previous sections we know that:

\begin{equation*}
T_{cc}^{MV} = \sum_{i=\ceil{\frac{H}{2}}}^{H} \binom{H}{i} T_{cc}^{i}(1- T_{cc})^{H-i} \quad \text{and}\quad 
    T^{oMAP}_{cc} = \sum_{k=\ceil{A_c} }^H \binom{H}{k}T_{cc}^k (1-T_{cc})^{H-k}.
\end{equation*}

\begin{lemma}
\label{lemma:conditions_on_T_cc}
    Let $T$ be the workers' confusion matrix, and $(\nu_0,\nu_1)$ the class distribution:
\begin{align}
 \label{eq:t_cc_MaP_greater_T_cc_MV}
         &T^{oMAP}_{cc} > T^{MV}_{cc} \Longleftrightarrow      \frac{\nu_{c}}{\nu_{1-c}} \geq \biggl(\frac{\delta_{1-c}}{\delta_{c}} \biggr)^{\frac{H}{2}} \sqrt{ \rho} \\
         \label{eq:t_cc_MaP_equal_T_cc_MV}
         &T^{oMAP}_{cc} =  T^{MV}_{cc} \Longleftrightarrow   \biggl(\frac{\delta_{1-c}}{\delta_{c}} \biggr)^{\frac{H}{2}} \frac{1}{\sqrt{ \rho}}  < \frac{\nu_{c}}{\nu_{1-c}} < \biggl(\frac{\delta_{1-c}}{\delta_{c}} \biggr)^{\frac{H}{2}} \sqrt{ \rho} 
    \end{align}
\end{lemma}

Where:
\begin{equation}
\label{eq:notation_rho_delta}
\delta_c = \frac{T_{cc}}{1-T_{1-c,1-c}}
 \,\,\,\,
 \rho = \frac{T_{cc} T_{1-c,1-c}}{(1-T_{cc})(1-T_{1-c,1-c})}
\end{equation}

\textbf{Proof of statement in \cref{eq:t_cc_MaP_equal_T_cc_MV}}.
The  condition $ T_{cc}^{oMAP} = T_{cc}^{MV} $ is equivalent to $ \ceil{A_{c}} = \frac{H+1}{2} $, that is satisfied when: 
\begin{align*}
        \frac{H-1}{2} <     \frac{- \log{\frac{\nu_{c}}{\nu_{1-c}} + H \log{\frac{T_{1-c,1-c}}{1-T_{cc}}}}}{\log{\frac{T_{00}T_{11}}{(1-T_{00})(1-T_{11})}}}  \leq \frac{H+1}{2}
\end{align*}
Using the notation from \cref{eq:notation_rho_delta}.

We remove the denominator, by multiplying both sides for it, notice we can do it without changing the direction of the inequality since $\rho>1$ so $\log(\rho) >1$. We  can rewrite the equation above as: 
\begin{align*}
        \frac{H-1}{2} < &  \frac{   - \log{\frac{\nu_{c}}{\nu_{1-c}} + H \log{\delta_{1-c}}}}{\log{\rho} }  \leq \frac{H+1}{2} \\ & \quad \quad \quad \Updownarrow  \\
        H \left( \frac{\log{\rho}}{2} - \log \delta_{1-c} \right)  - \frac{\log{\rho}}{2} & < -   \log{\frac{\nu_{c}}{\nu_{1-c}}}   \leq H \left( \frac{\log{\rho}}{2} - \log \delta_{1-c} \right) +  \frac{\log{\rho}}{2}\\
     & \quad \quad \quad \Updownarrow  \\
H  \left( \frac{\log{\rho}}{2} - \log \delta_{1-c} \right)  - \frac{\log{\rho}}{2} < &    - \log{\frac{\nu_{c}}{\nu_{1-c}}}  \leq H \left( \frac{\log{\rho}}{2} - \log \delta_{1-c} \right) +  \frac{\log{\rho}}{2}. 
\end{align*}

Noticing that:
$$ \log{ \sqrt{\rho}} - \log{\delta_{1-c}} = \log{ \sqrt{ \frac{ \delta_{c}}{\delta_{1-c}}}}, $$
It follows that:
$$ \left(\frac{\delta_{c}}{\delta_{1-c}} \right)^{\frac{H}{2}} \frac{1}{\sqrt{\rho}}  < \frac{ \nu_{1-c}}{\nu_{c}} \leq \left(\frac{\delta_{c}}{\delta_{1-c}} \right)^{\frac{H}{2}} \sqrt{\rho} $$

This concludes the proof of statement in \cref{eq:t_cc_MaP_equal_T_cc_MV}.

In case we wanted to derive a condition for the single class distributions $\nu_c$ we could derive: 
$$
  \left(\frac{\delta_c}{\delta_{1-c}} \right)^{\frac{H}{2}} \frac{1}{\sqrt{\rho}} +1< \frac{1}{\nu_c} \leq  \left(\frac{\delta_{c}}{\delta_{1-c}} \right)^{\frac{H}{2}} \sqrt{\rho} +1 $$
From which: 
$$  \frac{1}{ \left(\frac{\delta_c}{\delta_{1-c}} \right)^{\frac{H}{2}} \sqrt{\rho} +1} 
\leq  {\nu_c}<  \frac{1}{ \left(\frac{\delta_c}{\delta_{1-c}} \right)^{\frac{H}{2}} \frac{1}{\sqrt{\rho}} +1} $$

\textbf{Proof of statement in \cref{eq:t_cc_MaP_greater_T_cc_MV} }
\begin{align*}
    T^{oMAP}_{c,c} > T^{MV}_{c,c} &\Longleftrightarrow \ceil{A_{c}} < \ceil{\frac{H}{2}} = \frac{H+1}{2} \Longleftrightarrow A_c \leq  \frac{H-1}{2} \\
   &\Longleftrightarrow     \frac{- \log{\frac{\nu_{c}}{\nu_{1-c}} + H \log{\frac{T_{1-c,1-c}}{1-T_{cc}}}}}{\log{\frac{T_{00}T_{11}}{(1-T_{00})(1-T_{11})}}}  \leq \frac{H-1}{2}
\end{align*}

Using the notation from \cref{eq:notation_rho_delta}:

\begin{align*}
    - \log{\frac{\nu_{c}
    }{\nu_{1-c}}} +  {H} \log \delta_{1-c}   \leq  \frac{H}{2} \log \rho -  \frac{1}{2} \log \rho 
 \Leftrightarrow
     - \log{\frac{\nu_{c}
    }{\nu_{1-c}}} \leq H\biggl[-\log \delta_{1-c} + \log(\sqrt{\rho}) \biggr] -   \log(\sqrt{\rho}) \\
 \Leftrightarrow
   - \log{\frac{\nu_{c}
    }{\nu_{1-c}}} \leq \frac{H}{2} \biggl[ \log \left( \frac{\delta_c}{\delta_{1-c}} \right) \biggr] -  \log \sqrt{\rho}
 \Leftrightarrow
        {\frac{\nu_{1-c}}{\nu_c}} \leq   \left( \frac{\delta_c}{\delta_{1-c}} \right)^{\frac{H}{2}}  \frac{1}{\sqrt{\rho}}  \Leftrightarrow  \frac{\nu_{c}
        }{\nu_{1-c}} \geq \left( \frac{\delta_{1-c}}{\delta_{c}} \right)^{\frac{H}{2}}  {\sqrt{\rho}}
\end{align*}

\begin{remark}
 Notice that from the Theorem above it follows directly that it can never happen that both $T_{00}^{MAP} > T_{00}^{MV}$ and    $T_{11}^{MAP} > T_{11}^{MV}$ indeed we have that:
 $$  \frac{\nu_{0}}{\nu_{1}} \geq \biggl(\frac{\delta_{1}}{\delta_{0}} \biggr)^{\frac{H}{2}} \sqrt{ \rho} \quad  \textrm{ and } \quad \frac{\nu_{1}}{\nu_{0}} \geq \biggl(\frac{\delta_{0}}{\delta_{1}} \biggr)^{\frac{H}{2}} \sqrt{ \rho}  \iff \sqrt{ \rho} = \frac{1}{\sqrt{ \rho}} \iff T_{cc} = 1-T_{cc},$$
which is against our assumption.
\end{remark}

We now state a sufficient condition to have that oMAP is equivalent to MV, precisely we have that if their confusion matrices are the same, the two methods are equivalent, the following Theorem states the condition under which this happens.
\begin{theorem}
\label{thm:suff_condition_appenidx}
Using the notation from \cref{eq:notation_rho_delta}, we have that if:
\begin{align}
    \label{eq:condition_sufficency_appendix}
    \biggl(\frac{\delta_0}{\delta_{1}} \biggr)^{\frac{H}{2}} \frac{1}{\sqrt{ \rho}}  < \frac{\nu_{1}}{\nu_0} < \biggl(\frac{\delta_0}{\delta_{1}} \biggr)^{\frac{H}{2}} \sqrt{ \rho}
\end{align}
Then $T^{oMAP}_{00} = T_{00}^{MV} $ and that $T_{11}^{oMAP } = T_{11}^{MV}$ from which it follows that under the conditions in described in \cref{eq:condition_sufficency_appendix} oMAP is equivalent to MV instance wise. 
\end{theorem}
\begin{proof}
From \cref{lemma:conditions_on_T_cc} we obtained that have $T_{00}^{oMAP} = T^{MV}_{00}$ we need 
    \begin{align}
        \begin{cases}
          \biggl(\frac{\delta_0}{\delta_{1}} \biggr)^{\frac{H}{2}} \frac{1}{\sqrt{ \rho}}  < \frac{\nu_{1}}{\nu_0} \leq  \biggl(\frac{\delta_0}{\delta_{1}} \biggr)^{\frac{H}{2}} \sqrt{ \rho}\\
            \biggl(\frac{\delta_1}{\delta_{0}} \biggr)^{\frac{H}{2}} \frac{1}{\sqrt{ \rho}}  < \frac{\nu_{0}}{\nu_1} \leq \biggl(\frac{\delta_1}{\delta_{0}} \biggr)^{\frac{H}{2}} \sqrt{ \rho} 
        \end{cases}
        \iff 
            \biggl(\frac{\delta_0}{\delta_{1}} \biggr)^{\frac{H}{2}} \frac{1}{\sqrt{ \rho}}  < \frac{\nu_{1}}{\nu_0} < \biggl(\frac{\delta_0}{\delta_{1}} \biggr)^{\frac{H}{2}} \sqrt{ \rho}
    \end{align}
\end{proof}

Notice that the same condition described by the theorem hold for $  \frac{\nu_0}{\nu_1}$.

\begin{corollary}
\label{corollary:one_parameter_equivalence}
    In the case the noise rate of the two classes is symmetric, i.e. $T_{00}=T_{11}$ we have that if: 
$$ 1- T_{00} < \nu_0 < T_{00}$$
it follows that $T_{00}^{oMAP} = T_{00}^{MV} $ and $T_{11}^{oMAP} = T_{11}^{MV}$.
\end{corollary}

\begin{proof}
    The statement follows from \cref{thm:suff_condition_appenidx} noticing that in the case $T_{00}= T_{11}$ we have that  $ \frac{\delta_0}{\delta_1} =1 $ and $ \sqrt{\rho} = \frac{T_{00}}{1-T_{00}}$. 
    From this we obtain that: 
        $$  \frac{1-T_{00}}{T_{00}}< \frac{\nu_{1}}{\nu_0} < \frac{T_{00}}{1-T_{00}}  \iff \frac{1}{T_{00}} < \frac{1}{\nu_0} < \frac{1}{1-T_{00}}$$ 
\end{proof}

\begin{corollary}
    In the case the noise rate of the two classes is symmetric, i.e. $T_{00}=T_{11}$ we have that $T_{cc}^{oMAP} > T_{cc}^{MV} $  if and only if: 
$$ \nu_{c} \geq T_{00}$$
\end{corollary}

\begin{proof}
    The statement follows substituting  $ \frac{\delta_0}{\delta_1} =1 $ and $ \sqrt{\rho} = \frac{T_{00}}{1-T_{00}}$ in \cref{eq:t_cc_MaP_greater_T_cc_MV}. 
\end{proof}

Let us see what happens in the case where we have that $T_{cc}^{oMAP}$ is greater than $T_{cc}^{MV}$. 

\begin{theorem}
\label{th:necessary_only_appendix}
    If:
    $$  \frac{\nu_{1}}{\nu_0} < \biggl(\frac{\delta_0}{\delta_{1}} \biggr)^{\frac{H}{2}} \frac{1}{\sqrt{ \rho}} \quad \quad 
    \textrm{ or }
    \quad \quad 
    \frac{\nu_{1}}{\nu_0}>  \biggl(\frac{\delta_0}{\delta_{1}} \biggr)^{\frac{H}{2}} \sqrt{ \rho}$$
    oMAP is strictly better than MV.  
\end{theorem}

\begin{proof}
    Without loss of generality we can consider the case in which $T_{00}^{oMAP}>T_{00}^{MV} $ and $T_{11}^{oMAP}<T_{11}^{MV} $, similar reasoning can be applied in the other case. 
    MAP is better than MV if: 
    $$ \nu T_{00}^{oMAP} + (1-\nu) T_{11}^{oMAP} > \nu T_{00}^{MV} + (1-\nu) T_{11}^{MV} $$
    That is true if:
    $$ \sum_{k=\ceil{A_{0}}}^{\frac{H-1}{2}} \binom{H}{k}  T_{00}^{k}(1-T_{00}^{H-k}) - \left( \frac{1-\nu}{\nu} \right)  \sum_{k=\frac{H-1}{2}}^{\ceil{A_{1}}} \binom{H}{k}  T_{11}^{k}(1-T_{11}^{H-k}) > 0 $$
    Denoting by $\eta = \left( \frac{1-\nu}{\nu} \right)$, that happens when: 
    $$\sum_{k=\ceil{A_{0}}}^{\frac{H-1}{2}} \binom{H}{k} \left[  T_{00}^{k}(1-T_{00}^{H-k}) - \eta   T_{11}^{H- k}(1-T_{11}^{k}) \right] > 0  $$
We know inspect the terms of the sum, if they are all positive than for sure the sum is positive. 
Namely, we want now to check if: 
$$ T_{00}^{k}(1-T_{00}^{H-k}) - \eta   T_{11}^{H- k}(1-T_{11}^{k}) >0. $$
$$ \Updownarrow$$
$$ \eta < \frac{T_{11}^{H- k}(1-T_{11}^{k})}{T_{00}^{k}(1-T_{00}^{H-k})} = \frac{\delta_{0}^k}{\delta_{1}^{H-k}} = \frac{\rho^{k}}{\delta_1^{H}}$$
$$ \Updownarrow$$
$$ k > \frac{H \log \delta_1 + \log \eta }{ \log \rho} = A_0.$$

The index $k$ of the sum goes from $k= \ceil{A_{0}} $ to $k = \frac{H-1}{2}$. Moreover we were assuming $\ceil{A_0} \notin \mathbb{N}$, this means that for all terms of the sum it is satisfied that $ k > A_{0}$. 

\end{proof}

Putting \cref{thm:suff_condition_appenidx} and \cref{th:necessary_only_appendix} together we are able to prove  a necessary and sufficient condition to have oMAP equivalent to MV.

\begin{theorem*}[\cref{thm:twocoins} in the main paper. Necessary and sufficent condition for equivalence between MAP and MV]
Using the notation in \cref{eq:notation_rho_delta}, we have that oMAP is equivalent to MV instance wise if and only if:
\begin{align}
    \biggl(\frac{\delta_0}{\delta_{1}} \biggr)^{\frac{H}{2}} \frac{1}{\sqrt{ \rho}}  < \frac{\nu_{1}}{\nu_0} < \biggl(\frac{\delta_0}{\delta_{1}} \biggr)^{\frac{H}{2}} \sqrt{ \rho}
\end{align} 
\end{theorem*}

\begin{proof}

The Theorem is a corollary of \cref{th:necessary_only_appendix} and \cref{thm:suff_condition_appenidx}.

\end{proof}

\cref{thm:binary_classes_shared_coins} is a Corollary of the Theorem above and follows immediately considering that in that case $ \delta_1 = \delta_0$ and $ \sqrt{\rho} = \frac{T_{00}}{1-T_{00}}$.

\subsection{Including estimation of $T$}
\label{sup:estimatedTProof}

\begin{theorem*}[Gap with one-parameter $\hatp$ for binary tasks]
Given the model assumptions as in \cref{thm:binary_classes_shared_coins} an an estimate $\hatp$ of $\p$ satisfying $\| \hatp - \p \|_\infty<\epsilon$, we have with at least probability $(1-\delta)$ for any arbitrarily small $\delta$:
\begin{equation}
    \vert\pro(\y_{\eMAP}=
    \y)-\pro(\y_{\MV}=\y)\vert \leq  \frac{2 \epsilon H (1+ \epsilon)^{H-1}}{\nu} + \pro(\y_{\oMAP}=\y) - \pro(\y_{\MV}=\y).
\end{equation}
\end{theorem*}

\begin{proof}
Suppose we have that with probability higher than $1- \delta$ it happens that $ || \widehat{T}- T ||_{\infty} < \epsilon$. 
It follows that with probability at least $ 1- \delta$: 

$$ \mathbb{P}( y_{eMAP}=y) - \mathbb{P}(y_{MV} = y ) =  [ \mathbb{P}(y_{eMAP}=y )- \mathbb{P}( y_{{MAP}}=y) ]  +  [\mathbb{P}( y_{{MAP}}=y) - \mathbb{P}(y_{MV} = y )]$$
So: 
\begin{equation}
\begin{split}
 |\mathbb{P}( y_{eMAP}=y) - \mathbb{P}(y_{MV} = y )| &\leq   | \mathbb{P}(y_{eMAP}=y )- \mathbb{P}( y_{{MAP}}=y)| + | \mathbb{P}( y_{{MAP}}=y) - \mathbb{P}(y_{MV} = y )| \\
 & \leq \frac{\epsilon}{\nu} + | \mathbb{P}( y_{{MAP}}=y) - \mathbb{P}(y_{MV} = y )| 
 \end{split}
 \end{equation}

Indeed 
$ \mathbb{P}( y_{{MAP}}=y) -  \mathbb{P}(y_{eMAP}=y ) =    T_{00}^{MAP} - T_{11}^{MAP}\frac{1 -\nu}{\nu} -  T_{00}^{\widehat{MAP}} + T_{11}^{\widehat{MAP}}\frac{1 -\nu}{\nu}  $

It follows that:

$$ |\mathbb{P}( y_{{MAP}}=y) -  \mathbb{P}(y_{eMAP}=y )| \leq | T_{00}^{MAP} - T_{00}^{\widehat{MAP}} | + \left( \frac{1-\nu}{\nu} \right)| T_{11}^{MAP} - T_{11}^{\widehat{MAP}} |  $$

We recall that:

$$    T^{\oMAP}_{cc}  = \sum_{k=\ceil{A_c} }^H \binom{H}{k}T_{cc}^k (1-T_{cc})^{H-k} \text{ and }  T^{\eMAP}_{cc}  = \sum_{k=\ceil{\widehat{A_c}} }^H \binom{H}{k} \widehat{T}_{cc}^k (1-\widehat{T}_{cc})^{H-k}  $$

Let's first show that the error in the noise transition matrix $\epsilon $, is small enough $ \ceil{A_c} = \ceil{ \widehat{A_c}} $.

 $\chi = (\nu_{0}, 1-\nu_{0})$ as the vector having as element class distributions, $ || \rho || = 1$. 
If we are able to recover the actual noise transition matrix up to an error of $\epsilon$, as a consequence we are also able to recover the data distribution up to an $\epsilon$, indeed, the class distribution after the noise introduction is $ \chi' = T \chi  $
and $ \chi = T^{-1} \chi'$.
 So we can retrieve  $ \hat{\chi} = \hat{T}^{-1} \rho'$ and:

\begin{equation*}
    ||\hat{\chi} - \chi|| = | T^{-1} \chi' - \hat{T}^{-1} \nu'| \leq || T^{-1} - \hat{T}^{-1}|| \cdot ||\chi'|| \leq \epsilon \cdot 1.    
\end{equation*}

If $|x- \hat{x}| \leq \varepsilon$ it follows that 
$1-  \frac{\varepsilon}{\hat{x}} \leq \frac{x}{\hat{x}} \leq 1 +  \frac{\varepsilon}{\hat{x}} $, i.e. 
$ |\frac{x}{\hat{x}}-1| \leq \frac{\varepsilon}{\hat{x}}   $
We assume, wlog that $ \log{\frac{x}{1-x}} > \log{\frac{\hat{x}}{1- \hat{x}}}$, namely $ \frac{x(1- \hat{x})}{\hat{x} (1-x)} > 1$ if this is not the case we can swap the two terms:

\begin{align*}
 \log{\frac{x}{1-x}} - \log{\frac{\hat{x}}{1- \hat{x}}}  & = \log{\frac{x(1- \hat{x})}{\hat{x} (1-x)}} \leq \log({1+ \varepsilon}) \leq \epsilon
\end{align*}

Moreover:
$$ \frac{T_{00} T_{11}}{(1- T_{11})(1- T_{00})} -  \frac{\widehat{T}_{00} \widehat{T}_{11}}{(1- \widehat{T}_{11})(1- \widehat{T}_{00})} \leq \frac{[1 - \min{(T_{00}, T_{11})]} \varepsilon}{(1- \max{(T_{11}, T_{00}))^4 - \varepsilon}} $$

And:

$$ \log{\frac{T_{00} T_{11}}{(1- T_{11})(1- T_{00})}}  \log{\frac{\widehat{T}_{00} \widehat{T}_{11}}{(1- \widehat{T}_{11})(1- \widehat{T}_{00})}} \geq  \left[ \log{\frac{T_{00} T_{11}}{(1- T_{11})(1- T_{00})}}  \right]^2 - \varepsilon \left[ \log{\frac{T_{00} T_{11}}{(1- T_{11})(1- T_{00})}}  \right]   $$
It follows that:

$$ | \hat{A_c} - A_c | \leq \frac{\log(1+ \varepsilon) + \log{ \left( 1+ \varepsilon \frac{T_{1-c, 1-c}}{ (1-T_{cc})^2 -\varepsilon} \right)} }{\left[ \log{\frac{T_{00} T_{11}}{(1- T_{11})(1- T_{00})}}  \right]^2 - \varepsilon \left[ \log{\frac{T_{00} T_{11}}{(1- T_{11})(1- T_{00})}}  \right]} $$

Meaning that if $\varepsilon$ is small enough, we can have that:

$ \ceil{\hat{A_c}} -\ceil{ A_c} $.

Now, using the fact that $ |ab - \hat{a} \hat{b}| \leq b | a - \hat{a}| + \hat{a} | b - \hat{b}|$:

\begin{align}
\label{eq: difference_element_wises_T_estiamtion}
T^{\oMAP}_{cc}  - T^{\eMAP}_{cc}  &\leq \sum_{k=\ceil{A_c} }^H  \binom{H}{k} \{ (1- \widehat{T}_{cc})|T_{cc}^k  -  \widehat{T}_{cc}^k |   + T_{cc}[ (1-T_{cc})^{H-k} - (1 - \widehat{T}_{cc})^{H-k} ] \}  \\
& = \underbrace{ \hspace{-1.8pt} \sum_{k=\ceil{A_c}}^H \hspace{-2pt} \binom{H}{k} \hspace{-0.8pt} ( \hspace{-0.8pt} 1 \hspace{-0.8pt}-\hspace{-0.8pt} \widehat{T}_{cc})^{H-k}\hspace{-0.8pt}|\hspace{-0.8pt}T_{cc}^k - \widehat{T}_{cc}^k |   }_{\text{Term 1}} +  \underbrace{\hspace{-1.8pt} \sum_{k=\ceil{A_c}}^H 
\hspace{-0.9pt} \binom{H}{k} \hspace{-0.8pt} T_{cc}^{k}[ (1-T_{cc})^{H-k} \hspace{-0.8pt}-\hspace{-0.8pt} (1 \hspace{-0.8pt} - \hspace{-0.8pt}\widehat{T}_{cc})^{H-k} ]}_{\text{Term 2}}
\nonumber
\end{align}

Now using the fact that $ (x^n - y^n) = (x -y )(x^{n-1} + x^{n-2}y + x^{n-3}y^2 \dots x^2y^{n-3} + x y^{n-2} + y^{n-1})$ we can obtain that if $ |T_{cc} - \widehat{T}^{cc} | < \epsilon$ it follows that: 
$$ || T_{cc}^k - \widehat{T_{cc}}^{k}| \leq |T_{cc} - \widehat{T}_{cc} | \max(T_{cc}, \widehat{T}_{cc})^{k-1}(k-1) $$

So: 
\begin{align}
\label{eq:bound_power_k_T_estimation}
|| T_{cc}^k - \widehat{T_{cc}}^{k}| &\leq |T_{cc} - \widehat{T}^{cc} | \max(T_{cc}, \widehat{T}_{cc})^{k-1}(k-1)
\nonumber
\\ &\leq \epsilon (\widehat{T}_{cc} + \epsilon)^{k-1} (k-1)  \\
&\leq \epsilon (\widehat{T}_{cc} + \epsilon)^{k-1}k 
\nonumber
\end{align}

And:
\begin{align}
\label{eq:bound_power_H_k_T_estimation}
    | (1- T_{cc})^{H-k} - (1- \widehat{T_{cc}})^{H-k}| & \leq  |T_{cc} - \widehat{T}_{cc} | \max(1-T_{cc}, 1- \widehat{T}_{cc})^{H-k-1}(H-k-1)   \nonumber \\ 
   &  \leq \epsilon (H-k-1)  (1- \widehat{T}_{cc} + \epsilon)^{H-k-1} \\
   & \leq \epsilon(H-k)  (1- \widehat{T}_{cc} + \epsilon)^{H-k-1} 
\nonumber
\end{align}

 We can now use the \cref{eq:bound_power_k_T_estimation,eq:bound_power_H_k_T_estimation} To bound Term 1 and Term 2 of \cref{eq: difference_element_wises_T_estiamtion} it follows that: 

 \begin{align*}
     \text{Term 1} &\leq \epsilon \sum_{k=\ceil{A_c}}^H \binom{H}{k}  (1 -  \widehat{T}_{cc})^{H-k} (\widehat{T}_{cc} + \epsilon)^{k-1}k \\
     &= \epsilon H \sum_{k=\ceil{A_c} -1 }^{H-1} \binom{H-1}{h}  (1 -  \widehat{T}_{cc})^{H-1-h} (\widehat{T}_{cc} + \epsilon)^{h}\\
     & \leq \epsilon H (1+ \epsilon)^{H-1}
 \end{align*}

 \begin{align*}
     \text{Term 2} &\leq \epsilon \sum_{k=\ceil{A_c}}^{H} \binom{H}{k}  (H-k) 
     T_{cc}^k   (1- \widehat{T}_{cc} + \epsilon)^{H-k-1}   
     \\
    = & \sum_{k=\ceil{A_c}}^{H-1} \binom{H}{k}  (H-k) 
     T_{cc}^k   (1- \widehat{T}_{cc} + \epsilon)^{H-k-1}   \text{ [we are  using that $H-H = 0$]} \\
     & = \epsilon H \sum_{k =\ceil{A_c}}^{H-1} \binom{H-1}{k}  
     T_{cc}^{k}   (1- \widehat{T}_{cc} + \epsilon)^{H-1-k} \\
     & \leq \epsilon H (1+ \epsilon)^{H-1}
 \end{align*}

 As a consequence 
$ T^{\oMAP}_{cc}  - T^{\eMAP}_{cc}  \leq 2\epsilon H (1+ \epsilon)^{H-1} $

From this it follows that:

\begin{align*}
|\mathbb{P}( y_{{oMAP}}=y) -  \mathbb{P}(y_{eMAP}=y )| &\leq 2\epsilon H (1+ \epsilon)^{H-1}  + \left( \frac{1-\nu}{\nu} \right)| 2\epsilon H (1+ \epsilon)^{H-1}\\
& = \frac{2 \epsilon H (1+ \epsilon)^{H-1}}{\nu} 
\end{align*} 

\end{proof}

\subsection{Proof of \cref{th:estimated_quantities}}
\label{proof_theorem_est_quantities}

\begin{lemma}
Given $N$ noisy samples and an approximation of the noise transition matrix \( \tilde{T}, \) such that the inequality \( || T - \tilde{T} ||_2 \leq \epsilon \) holds with probability at least \( 1 - \gamma \), we define \( \tilde{\nu} = \tilde{T}^{-1} \hat{\nu}_{\text{noisy}} \), where \( \hat{\nu}_{\text{noisy}} \) is an approximation of the noisy label distribution. Under these conditions, it follows that, with probability \( 1 - \alpha \), the following bound holds:
\[
|| \nu - \tilde{\nu} ||_2 \leq \frac{\epsilon}{\lambda_{\text{min}}(\tilde{T})} \left[ \frac{1}{\lambda_{\text{min}}(T) - \epsilon} + \sqrt{C} \right].
\]
where \( \alpha = 1 - 2e^{-2 \epsilon^2 N} - \gamma \).
\label{lemma on approximation of nu}
\end{lemma}

\begin{proof}
If with probability at least $  1 - \gamma(\epsilon) $ is true that $|| T- \tilde{T}|| < \epsilon $. From the definition of noise transition matrix, we have that $\nu_{\text{noise}} = T \nu $. 

Now $ \nu_{\text{noisy}}$ is something we can approximate looking at the distribution of the noisy data. 

We call this approximation $ \hat{\nu}_{\text{noisy}}  $ and it is true that (we can use Dvoretzky-Kiefer-Wolfowitz inequality \citep{Dvoretzky1956AsymptoticMC}):
$$ \mathbb{P}( || \hat{\nu}_{\text{noisy}} - \nu_{\text{noisy}}||_{\infty} > \epsilon ) \leq 2 e^{-2 \epsilon^2 n}.$$

Let us consider $\tilde{T}$ the approximation of $T$ and $ \hat{\nu}$ the approximation of $\nu$, we define $ \tilde{\nu} = \tilde{T}^{-1} \hat{\nu}_{\text{noisy}}$. 

Since the distribution of the classes must sum up to 1, $\sum_{i=1}^C \nu_{\text{noisy}}^i = 1$ so $ ||\nu_{\text{noisy}}||_2 \leq  1$.

\begin{align*}
    \nu - \tilde{\nu} &= T^{-1}\nu_{\text{noisy}} - \tilde{T}^{-1}\hat{\nu}_{\text{noisy}}  \\
    & = T^{-1}\nu_{\text{noisy}} - \tilde{T}^{-1} \nu_{\text{noisy}} + \tilde{T}^{-1} \nu_{\text{noisy}} -  \tilde{T}^{-1}\hat{\nu}_{\text{noisy}}    \\
    & = (T^{-1} - \tilde{T}^{-1}) \nu_{\text{noisy}}  + \tilde{T}^{-1}(\nu_{\text{noisy}} - \hat{\nu}_{\text{noisy}} ) 
\end{align*}
So: 
\begin{align*}
    ||\nu - \tilde{\nu}||_{2} \leq || T^{-1} - \tilde{T}^{-1}||_2 ||\nu_{\text{noisy}}||_2 + ||\tilde{T}^{-1}||_2 ||\nu_{\text{noisy}} - \hat{\nu}_{\text{noisy}} ||_2
\end{align*}

Notice that $\nu_{\text{noisy}} $ and $\hat{\nu}_{\text{noisy}}  $ are vectors in $ \mathbb{R}^C$. It follows that $||\nu_{\text{noisy}} - \hat{\nu}_{\text{noisy}} ||_2 \leq \sqrt{C}  ||\nu_{\text{noisy}} - \hat{\nu}_{\text{noisy}} ||_{\infty}$.

Moreover, from the definition of noise transition matrix,  
$\tilde{T}$ is symmetric so $ \tilde{T}^{-1}$ symmetric, it follows that:
\begin{align*}
||\tilde{T}^{-1}||_2 = \lambda_{\text{max}}(\tilde{T}^{-1}) = \frac{1}{\lambda_{\text{min}}(T)} \\
 T^{-1} - \tilde{T}^{-1} = T^{-1} ( \tilde{T} - T) \tilde{T}^{-1} \\
 || T^{-1} - \tilde{T}^{-1}||_2 \leq \frac{\epsilon}{\lambda_{\text{min}}(T) \lambda_{\text{min}}(\tilde{T}) }.   
\end{align*}

Using Weyl's inequality \citep{weyl1912asymptotische} on perturbed eigenvalues:
\begin{equation*}
\lambda_{\text{min}}(\tilde{T}) \leq \lambda_{\text{min}}(T) + \lambda_{\text{max}}(T - \tilde{T}) = \lambda_{\text{min}}(T) + || T- \tilde{T}||_2 
\end{equation*}
It follows that  $ \lambda_{\text{min}}(T) > \lambda_{\text{min}}(\tilde{T}) - \epsilon $. 

Using Boole's inequality \cite{boole1847mathematical} we have that with probability $ 1- 2e^{-2 \epsilon^2 N} - \delta$: 
\begin{equation*}
|| \nu - \tilde{\nu}||_{2} \leq \frac{\epsilon}{\lambda_{\text{min}}(\tilde{T}) [\lambda_{\text{min}}(\tilde{T}) - \epsilon)]} + \epsilon \frac{\sqrt{C}}{\lambda_{\text{min}}(\tilde{T})} = \frac{\epsilon}{ \lambda_{\text{min}}(\tilde{T})} \left[ \frac{1}{\lambda_{\text{min}}(\tilde{T}) - \epsilon} + \sqrt{C} \right]
\end{equation*}
\end{proof}

\begin{lemma}
Let the function \( A(x, y) \) be defined as 
$ A(x,y) = x^{a} (1- x)^{b} y^{c}(1-y)^{d} $.\\
In the interval $\frac{1}{2}<x<1-\xi $ and $\frac{1}{2} < y < 1 - \xi$, where $\xi$ is a positive constant smaller than 1, the function $A(x, y)$ is Lipschitz continuous with Lipschitz constant\\
$(1-\xi)^{a+c-1} (\frac{1}{2})^{b+d-1} \sqrt{ \max ( |\frac{a-b}{2}|,  |- b + \xi(a +b)|) +  \max ( |\frac{c-d}{2}|,  |- d+  \xi(c +d)|) }$.
\label{lemma_lips_one}
\end{lemma}

\begin{proof}
in the interval $   \frac{1}{2} < x < 1- \xi, \; 
      \frac{1}{2} <y < 1- \xi$ is differentiable  and its gradient is

    \[
\nabla A(x, y) =
\begin{pmatrix}
 \left( a - ax - bx \right) x^{a-1} (1-x)^{b-1}y^{c}(1-y)^{d}  \\
  \left( c-cy-dy \right) x^{a} (1-x)^{b}y^{c-1}(1-y)^{d-1}
\end{pmatrix}.
\]

Now using that $ \frac{1}{2} < x < 1- \xi $ and $ \frac{1}{2} < y < 1- \xi  $, it follows that $\xi < 1-x < \frac{1}{2} $ and $ \xi < 1-y < \frac{1}{2}$. As a consequence:
\begin{align*}
|x|< 1-\xi \text { and } |1-x| < \frac{1}{2}\\
- b+ \xi(a +b) < a- (a+b) x < \frac{a-b}{2} \\
-d + \xi (c+d) < c-(c+d)y < \frac{c-d}{2} \\
|a- (a+b) x| < \max ( |\frac{a-b}{2}|,  |- b + \xi(a +b)|)\\
|c-(c+d)y | <  \max ( |\frac{c-d}{2}|,  |- d+  \xi(c +d)|)\\
\end{align*}
It follows that:
\begin{align*}
|\nabla A(x, y)_x| < \max ( |\frac{a-b}{2}|,  |- b + \xi(a +b)|) (1-\xi)^{a+c-1} (\frac{1}{2})^{b+d-1}\\
| \nabla A(x, y)_y |=  \max ( |\frac{c-d}{2}|,  |- d+  \xi(c +d)|)  (1-\xi)^{a+c-1} (\frac{1}{2})^{b+d-1}
\end{align*}

Finally:
\begin{equation*}
    \| \nabla A(x,y) \|_2 \leq  (1-\xi)^{a+c-1} (\frac{1}{2})^{b+d-1} \sqrt{ \max ( |\frac{a-b}{2}|,  |- b + \xi(a +b)|) +  \max ( |\frac{c-d}{2}|,  |- d+  \xi(c +d)|) }
\end{equation*}

\end{proof}

\begin{lemma}
Let the function $A(x,y)$ be defined as 
$ A(x,y) = x^{\frac{H-1}{2}} (1- x)^{\frac{H+1}{2}} y^{-\frac{H+1}{2}}(1-y)^{-\frac{H-1}{2}}$.\\

In the interval \( \frac{1}{2} < x < 1 - \xi \) and \( \frac{1}{2} < y < 1 - \xi \), where \( \xi \) is a positive constant smaller than 1, the function \( A(x, y) \) is Lipschitz continuous with Lipschitz constant
$(1-\xi)^{-2} (\frac{1}{2})^{1} \sqrt{ \max ( \frac{1}{2},  \frac{H+1}{2} - H \xi)^2 +  \max ( \frac{1}{2},  \frac{H-1}{2}  -\xi H )^2 }$.
\label{primo lemma Lips}
\end{lemma}

The proof follows straightforwardly from \cref{lemma_lips_one}.

\begin{lemma}
Let the function \( B(x, y) \) be defined as 
$ B(x,y) = x^{\frac{H+1}{2}} (1- x)^{\frac{H-1}{2}} y^{-\frac{H-1}{2}}(1-y)^{-\frac{H+1}{2}} $
In the interval \( \frac{1}{2} < x < 1 - \xi \) and \( \frac{1}{2} < y < 1 - \xi \), where \( \xi \) is a positive constant smaller than 1, the function \( B(x, y) \) is Lipschitz continuous with Lipschitz constant
\(
 4(1-\xi)\sqrt{ \max ( \frac{1}{2},  \frac{H-1}{2} - H \xi)^2 +  \max ( \frac{1}{2},  \frac{H+1}{2}  -\xi H )^2 }. 
\)
\label{secondo lemma lips}
\end{lemma}

The proof follows straightforwardly from \cref{lemma_lips_one}.

Assume we are able to verify the condition of \cref{thm:twocoins}, given the definition of $f(T), g(\nu), h(T)$ from the same Theorem. When dealing with real-world data and estimated quantities experiments quantities $\tilde{T}$ and $\tilde{\nu}$, a question arises: under what conditions does the following implication hold?
\begin{equation*}
f(\tilde{T}) < g(\tilde{\nu}) < h(\tilde{T}) \Rightarrow f(T) < g(\nu) < h(T)
\end{equation*}

From previous Lemmas, there exist bounds \( \epsilon_1 \), \( \epsilon_2 \), and \( \epsilon_3 \), such that:
\[
| g(\nu) - g(\tilde{\nu}) | < \epsilon_2, \quad | f(T) - f(\tilde{T}) | < \epsilon_1, \quad \text{and} \quad | h(T) - h(\tilde{T}) | < \epsilon_3.
\] with probability at least $ 1- \beta$ for some $\beta $ we can derive. 

Therefore, if the approximations \( \tilde{T} \) and \( \tilde{\nu} \) are sufficiently accurate, the inequalities of the estimated conditions can be used to infer the true conditions. 

If the following inequalities hold:
\[
\begin{cases}
g(\tilde{\nu}) - f(\tilde{T}) > \epsilon_2 + \epsilon_1 \\
h(\tilde{T}) - g(\tilde{\nu}) > \epsilon_3 + \epsilon_2
\end{cases}
\]
then it follows that:
\[
f(T) < g(\nu) < h(T).
\]

Given the function \( g_0(\nu_0) = \frac{1 - \nu_0}{\nu_0} \), defined on the interval \( \eta < \nu_0 < 1 - \eta \), where \( 0< \eta < 1 \). The function is differentiable, with its derivative given by 
\(
g_0'(\nu_0) = - \frac{1}{\nu_0^2}.
\)
Furthermore, the magnitude of the derivative is bounded as 
\(
|g_0'(\nu_0)| \leq \frac{1}{\min(\eta, 1 - \eta)^2}.
\)
Thus, if \( ||\tilde{\nu} - \nu||_2 \leq \epsilon_2 \), it follows that 
\(
|g_0(\nu_0) - g_0(\tilde{\nu}_0)| < \frac{\epsilon_2}{\min(\eta, 1 - \eta)^2}
\).

By applying \cref{lemma on approximation of nu}, we obtain that, with probability at least \( 1 - \gamma - 2e^{-2 \epsilon^2 N} \), the following inequality holds:
\begin{equation*}
|g_0(\nu_0) - g_0(\tilde{\nu}_0)| < \frac{\epsilon}{\lambda_{\text{min}}(\tilde{T})} \left[ \frac{1}{\lambda_{\text{min}}(\tilde{T}) - \epsilon} + \sqrt{C} \right] \frac{1}{\min(\eta, 1 - \eta)^2}
\end{equation*}

Assuming that  $ 0.5  < T_{cc} \leq 1- \xi  $ for some positive $ \xi < 1$ and that we have an approximation of the noise transition matrix $\tilde{T}$ so that with probability $ 1- \gamma$ , $ || T - \tilde{T}||_2 < \epsilon $, from \cref{primo lemma Lips} and \cref{secondo lemma lips} it follows that with probability  $ 1- \gamma$ holds that:
\begin{align*}
    |f_0(T) - f_0(\tilde{T})| \leq  \epsilon (1-\xi)^{-2} \frac{1}{2} \sqrt{ \max ( \frac{1}{2},  \frac{H+1}{2} - H \xi)^2 +  \max ( \frac{1}{2},  \frac{H-1}{2}  -\xi H )^2 } \\
    |h_0(T) - h_0(\tilde{T})| \leq 4 \epsilon \sqrt{ \max ( \frac{1}{2},  \frac{H-1}{2} - H \xi)^2 +  \max ( \frac{1}{2},  \frac{H+1}{2}  -\xi H )^2 }
\end{align*}

\section{Experiments}
\label{sec:sup_exps}
\subsection{Experimental Setup}
\label{sec:exp_setup}
All experiments were performed on an Amazon Web Services (AWS) Elastic Compute Cloud (EC2) instance with an AMD EPYC 7R32 CPU with 8 cores and 16 threads and no GPU.

All the code was written in Python 3.10.9. To implement aggregation methods as Dawid-Skene, GLAD and MACE was used the code from the Toloka library\footnote{\url{https://crowd-kit.readthedocs.io/en/latest/}}. To implement BWA, EBCC, IWMV, and LA\textsuperscript{twopass}  the code was inspired by \cite{10.1145/3630102}. 
For all the methods, we use the parameters presented in their respective papers.

Regarding the datasets, all of them, except D-Product, can be downloaded from the ActiveCrowdToolkit\footnote{\url{https://orchidproject.github.io/active-crowd-toolkit/}} page. The description of data they were aggregated can be found in \cite{venanzi}. The D-Product dataset was downloaded from the Crowdsourcing datasets repository\footnote{\url{https://dbgroup.cs.tsinghua.edu.cn/ligl/crowddata/}}.

\subsection{Optimality conditions with estimated quantities} \label{appendix:sec_other_exps}

\begin{table*}[ht]
    \centering
     \resizebox{0.9\linewidth}{!}{
    \begin{tabular}{ccccccccccccccc}
    \toprule
    $N$ & $n_e$ & Method & $\mathbb{I}_o(\clim,\p)$ & $\mathbb{I}_o(\tilde{\clim},\widetilde{\p})$  & $\mathcal{A}_\textrm{oMAP}$ & $\mathcal{A}_\textrm{MV}$ & $\mathcal{A}_\textrm{EBCC}$ & $\mathcal{A}_\textrm{LA}$ & $\mathcal{A}_\textrm{IWMV}$ & $\mathcal{T}_\textrm{MV}$ & $\mathcal{T}_\textrm{EBCC}$ & $\mathcal{T}_\textrm{LA}$ & $\mathcal{T}_\textrm{IWMV}$ & $\mathcal{T}_\textrm{o}$\\
    \midrule
    5000 & 0.05 & IWMV & \xmark & \cmark & 0.776 & 0.695 & 0.695 & 0.695 & 0.681 & \textbf{0.011} & 3.251 & \underline{0.028}  & 0.034 & 0.010 \\
    10000 & 0.05 & IAA & \cmark & \cmark & 0.810 & 0.810 & 0.810 & 0.810 & 0.810 & \textbf{0.024} & 6.549 & \underline{0.058} & 0.069 & 0.026 \\ 
    20000 & 0.1 & IAA & \xmark & \xmark & 0.901 & 0.720 & 0.774 & 0.720 & 0.716 & \textbf{0.047} & 36.203 & \underline{0.114} & 0.149 & 0.057 \\
    50000 & 0.1 & EBCC & \cmark & \cmark & 0.735 & 0.735 & 0.735 & 0.735 & 0.735 & \textbf{0.121} & 37.834 & \underline{0.294} & 0.375 & 8.122 \\
    65000 & 0.1 & EBCC & \cmark & \xmark & 0.783 & 0.783 & 0.783 & 0.783 & 0.783 & \textbf{0.158} & 101.664 & \underline{0.382} & 0.489 & 42.972 \\
    100000 & 0.05 & IWMV & \cmark & \cmark & 0.701 & 0.701 & 0.701 & 0.701 & 0.701 & \textbf{0.239} & 118.871 & \underline{0.626} & 0.761 & 0.177 \\
        \bottomrule
    \end{tabular}
    }
        \caption{\looseness -1 Different sets of annotations with $N$ samples are synthetically generated of which only $n_e$ fraction are used to estimate $(\tilde{\clim},\widetilde{\p})$. $\mathbb{I}_o(\cdot)$ indicates if optimality condition in \cref{thm:twocoins} is satisfied using real/estimated parameters.  
        $\mathcal{A}_\varphi$ and $\mathcal{T}_\varphi$ report accuracy and run time, in seconds, for aggregation method~$\varphi$. $\mathcal{T}_o$ is time to verify $\mathbb{I}_o(\tilde{\clim},\widetilde{\p})$.
        }
    \label{appendix:table_our_estimation_approach}
\end{table*}

\begin{table}[ht]
    \centering
    \begin{tabular}{cccc}
    \toprule
    $T$ & $\tilde{T}$ & $\nu$ & $\tilde{\nu}$  \\
    \midrule

    $\begin{bmatrix}
        0.55 & 0.45 \\0.2 & 0.8 \\
    \end{bmatrix}$ & $\begin{bmatrix}
        0.71 & 0.29 \\0.29 & 0.71 \\
    \end{bmatrix}$ & $\begin{bmatrix}
        0.65 & 0.35\\
    \end{bmatrix}$ & $\begin{bmatrix}
        0.54 & 0.46\\
    \end{bmatrix}$ \\

    $\begin{bmatrix}
        0.75 & 0.25 \\0.35 & 0.65 \\
    \end{bmatrix}$ & $\begin{bmatrix}
        0.68 & 0.32 \\0.32 & 0.68 \\
    \end{bmatrix}$ & $\begin{bmatrix}
        0.7 & 0.3\\
    \end{bmatrix}$ & $\begin{bmatrix}
        0.67 & 0.33\\
    \end{bmatrix}$ \\

    $\begin{bmatrix}
        0.65 & 0.35 \\0.35 & 0.65 \\
    \end{bmatrix}$ & $\begin{bmatrix}
        0.75 & 0.25 \\0.25 & 0.77 \\
    \end{bmatrix}$ & $\begin{bmatrix}
        0.9 & 0.1\\
    \end{bmatrix}$ & $\begin{bmatrix}
        0.98 & 0.02\\
    \end{bmatrix}$ \\
    
    $\begin{bmatrix}
        0.55 & 0.45 \\0.20 & 0.80 \\
    \end{bmatrix}$ & $\begin{bmatrix}
        0.67 & 0.33 \\0.33 & 0.67 \\
    \end{bmatrix}$ & $\begin{bmatrix}
        0.5 & 0.5\\
    \end{bmatrix}$ & $\begin{bmatrix}
        0.37 & 0.63\\
    \end{bmatrix}$  \\

    $\begin{bmatrix}
        0.70 & 0.30 \\0.30 & 0.70 \\
    \end{bmatrix}$ & $\begin{bmatrix}
        0.72 & 0.28 \\0.28 & 0.72 \\
    \end{bmatrix}$ & $\begin{bmatrix}
        0.68 & 0.32 \\
    \end{bmatrix}$ & $\begin{bmatrix}
        0.78 & 0.23\\
    \end{bmatrix}$  \\
    
    $\begin{bmatrix}
        0.70 & 0.30 \\0.45 & 0.55 \\
    \end{bmatrix}$ & $\begin{bmatrix}
        0.77 & 0.23 \\0.23 & 0.77 \\
    \end{bmatrix}$ & $\begin{bmatrix}
        0.68 & 0.32\\
    \end{bmatrix}$ & $\begin{bmatrix}
        0.64 & 0.36\\
    \end{bmatrix}$  \\
    \bottomrule
    \end{tabular}    \caption{Statistics of data from \cref{appendix:table_our_estimation_approach}. This Table shows the noise transition matrix ($\p$) and the distribution of classes ($\nu$) real and estimated for each experiment by using the methods shown in \cref{appendix:table_our_estimation_approach}.}
    \label{table:additional_info_our_estimation_approach}
\end{table}

\cref{appendix:table_our_estimation_approach} shows all four cases of
true/false-positive/negative from three different estimation methods (IWMV, IAA and EBCC). As expected, when both $\mathbb{I}_o(\clim,\p)$ and $\mathbb{I}_o(\tilde{\clim},\tilde{\p})$ are $\cmark$, then all the methods behave as good as oMAP. When $\mathbb{I}_o(\clim,\p)$ is $\xmark$, then oMAP will always be the best solution.
For reference, we also report the computational time for each aggregation algorithm and MV always remains the fastest, with the second place for \cite{10.1145/3630102}. Additional statistics about the data can be found in \cref{table:additional_info_our_estimation_approach}, where are shown the synthetically generated noise transition matrices ($\p$) and class distributions ($\nu$) and their corresponding estimates ($\tilde{T}$ and $\tilde{\nu})$. From \cref{fig:histogram} from the main paper the empirical approach is working perfectly in all cases. However, this approach does not always work perfectly, as also shown in \cref{appendix:table_our_estimation_approach}.

\cref{img:confusion_matrices} shows the confusion matrices of the the empirical approach with the IAA and EBCC methods, respectively. It can be seen how the quality of the estimation influences the performance of the estimation method. When using the EBCC methods, the empirical approach is making really a small number of mistakes. This is also confirmed by the quality of the estimated matrices, which are really close to the original ones.

\begin{figure*}[h] 
\centering  
    \begin{subfigure}[t]{0.25\linewidth}
        \includegraphics[width=\linewidth]{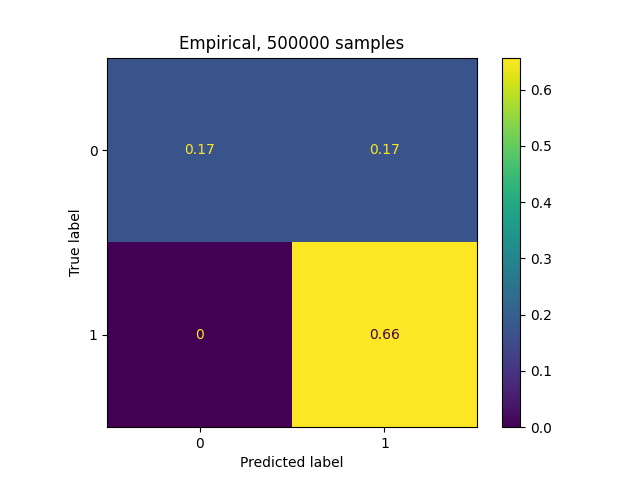}
        \subcaption{
        IAA}
    \end{subfigure} 
    \begin{subfigure}[t]{0.25\linewidth}
        \includegraphics[width=\linewidth]{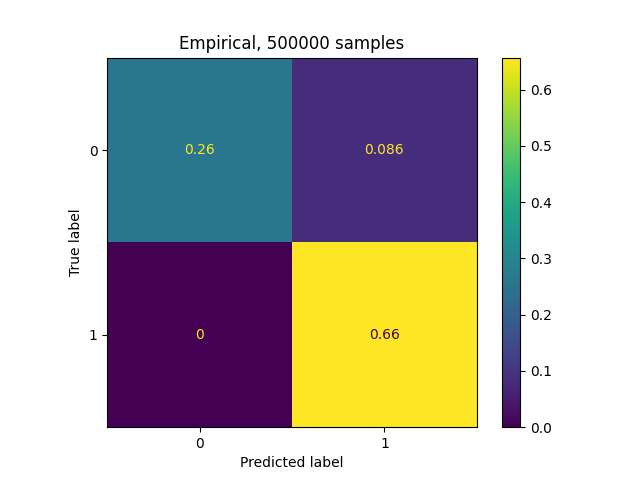}
        \subcaption{
        EBCC}
    \end{subfigure}\\
    \begin{subfigure}[t]{0.25\linewidth}
        \includegraphics[width=\linewidth]{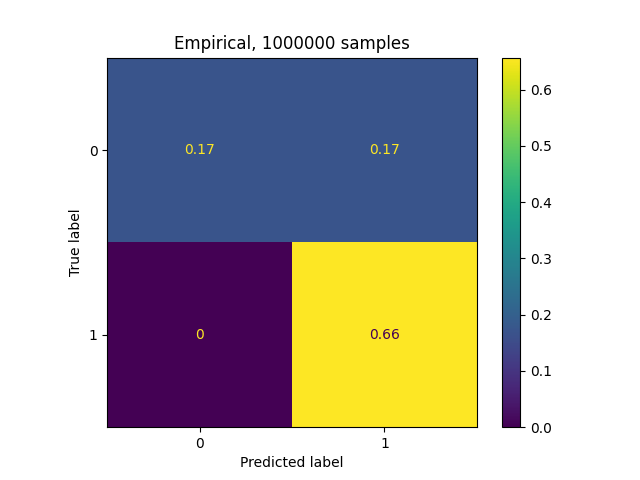}
        \subcaption{
        IAA}
    \end{subfigure}
    \begin{subfigure}[t]{0.25\linewidth}
        \includegraphics[width=\linewidth]{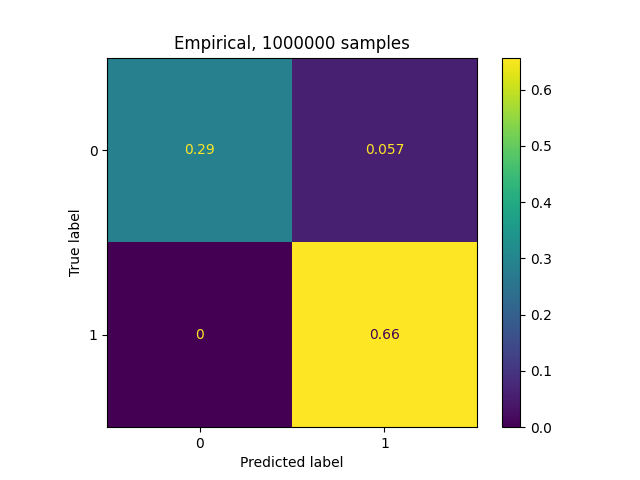}
        \caption{EBCC}
    \end{subfigure} 
\caption{Confusion matrices describing the performance of the empirical method comparing it with the oracle results. Different $T$ matrices and class distributions $\nu$ are used to perform the experiments. These results are based on $H=3$ and $N=10^6$. With this value of $N$, the empirical approach has good performance with the EBCC method, which slightly decrease with the IAA estimation approach.}
    \label{img:confusion_matrices}
\end{figure*}

\begin{figure*}[h!] 
\centering 
    \begin{subfigure}[t]{0.3\linewidth}
        \includegraphics[width=\linewidth]{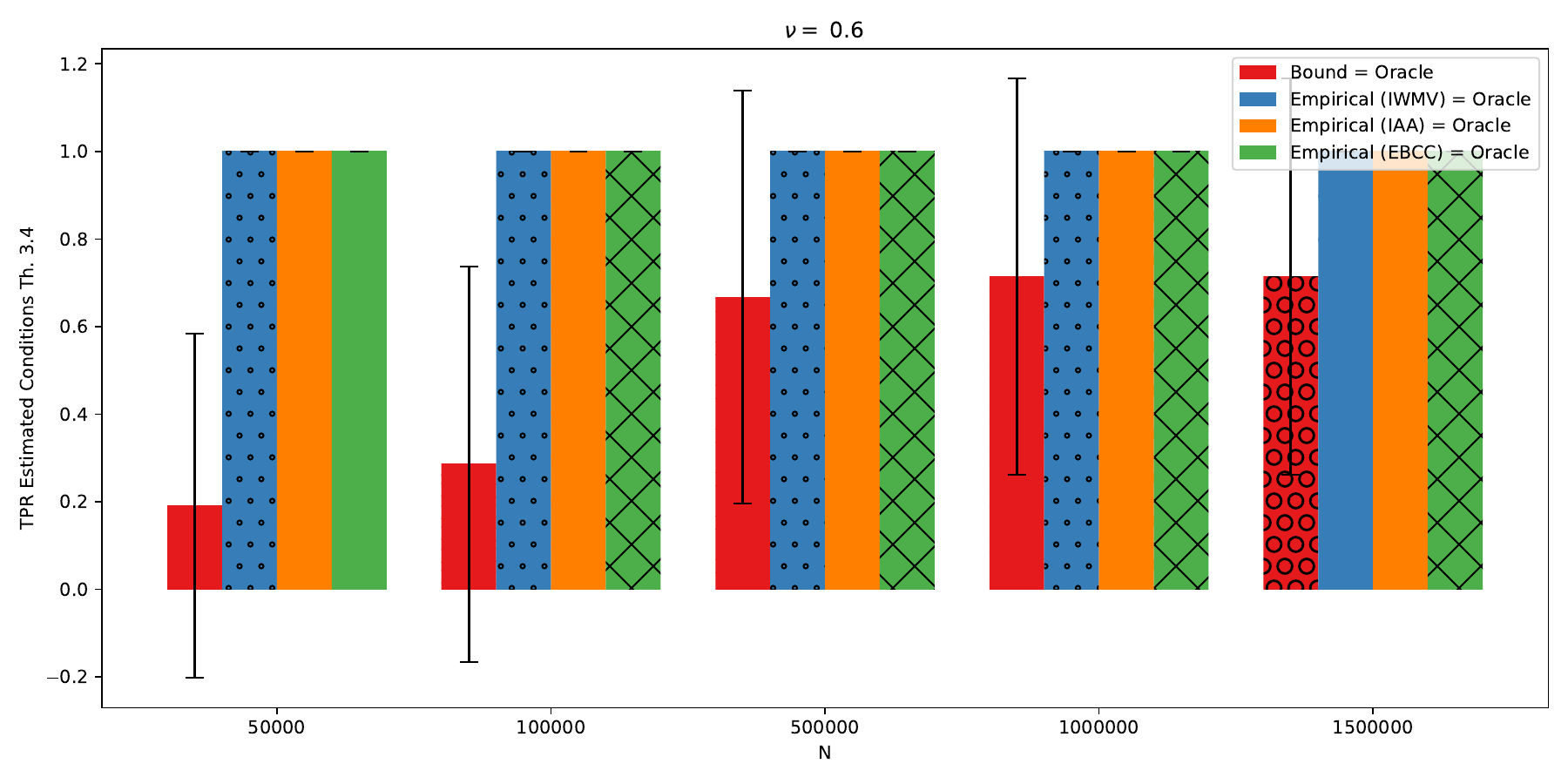}
        \subcaption{
        $\nu_0=0.6$}
    \end{subfigure}
    \begin{subfigure}[t]{0.3\linewidth}
        \includegraphics[width=\linewidth]{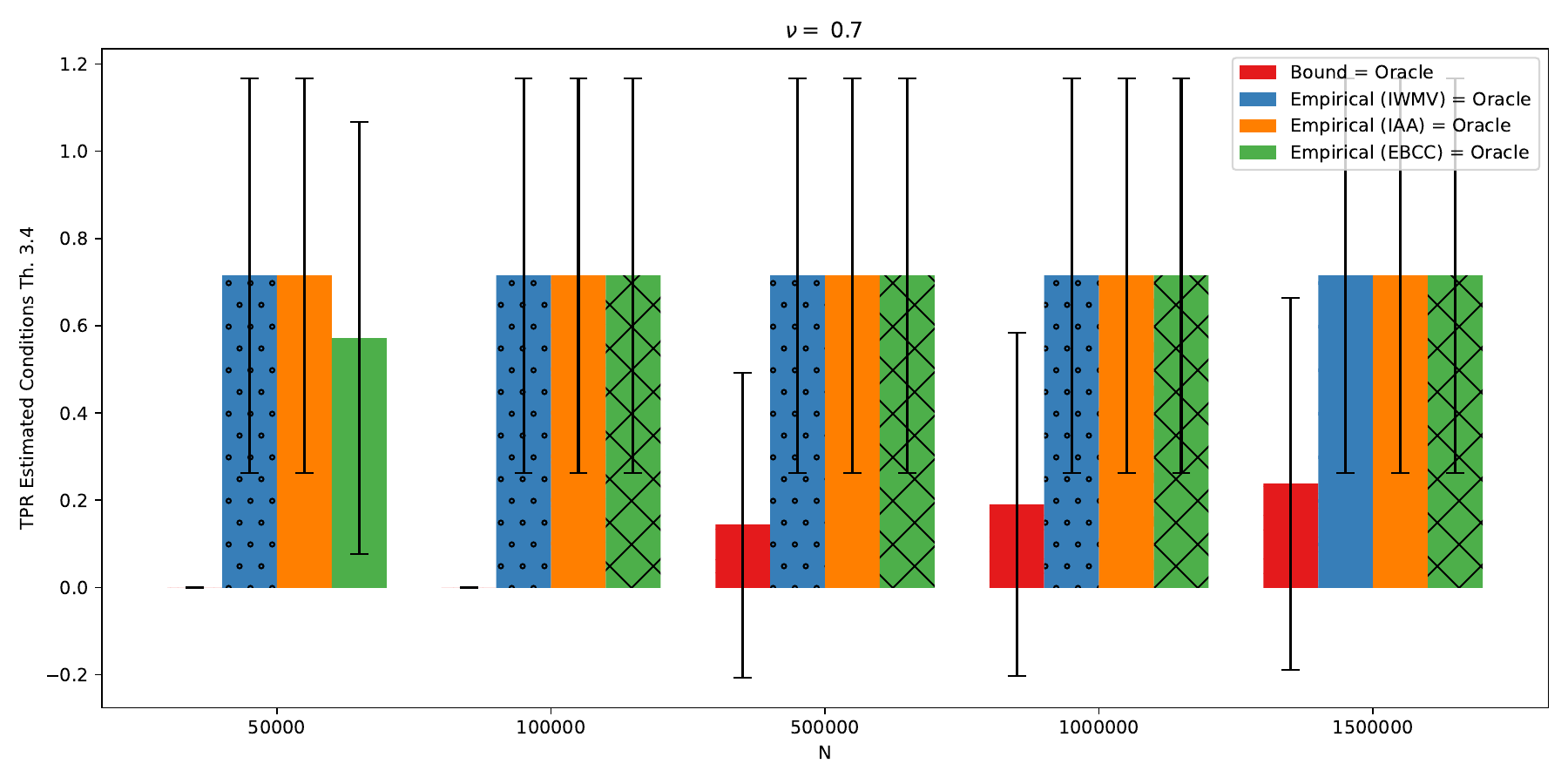}
        \caption{$\nu_0=0.7$}
    \end{subfigure} \\
    \begin{subfigure}[t]{0.3\linewidth}
        \includegraphics[width=\linewidth]{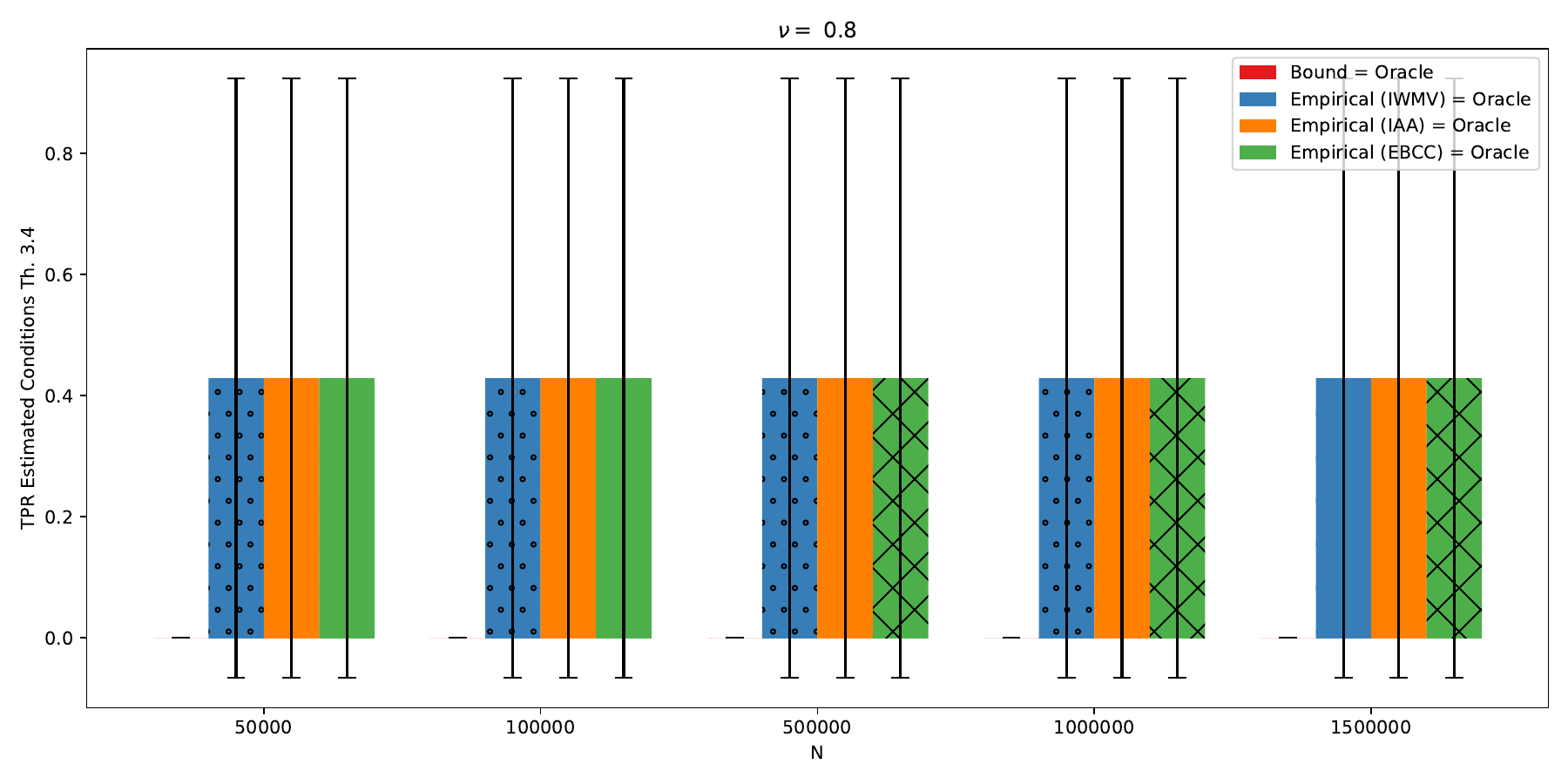}
        \caption{$\nu_0=0.8$}
    \end{subfigure}
\begin{subfigure}[t]{0.3\linewidth}
        \includegraphics[width=\linewidth]{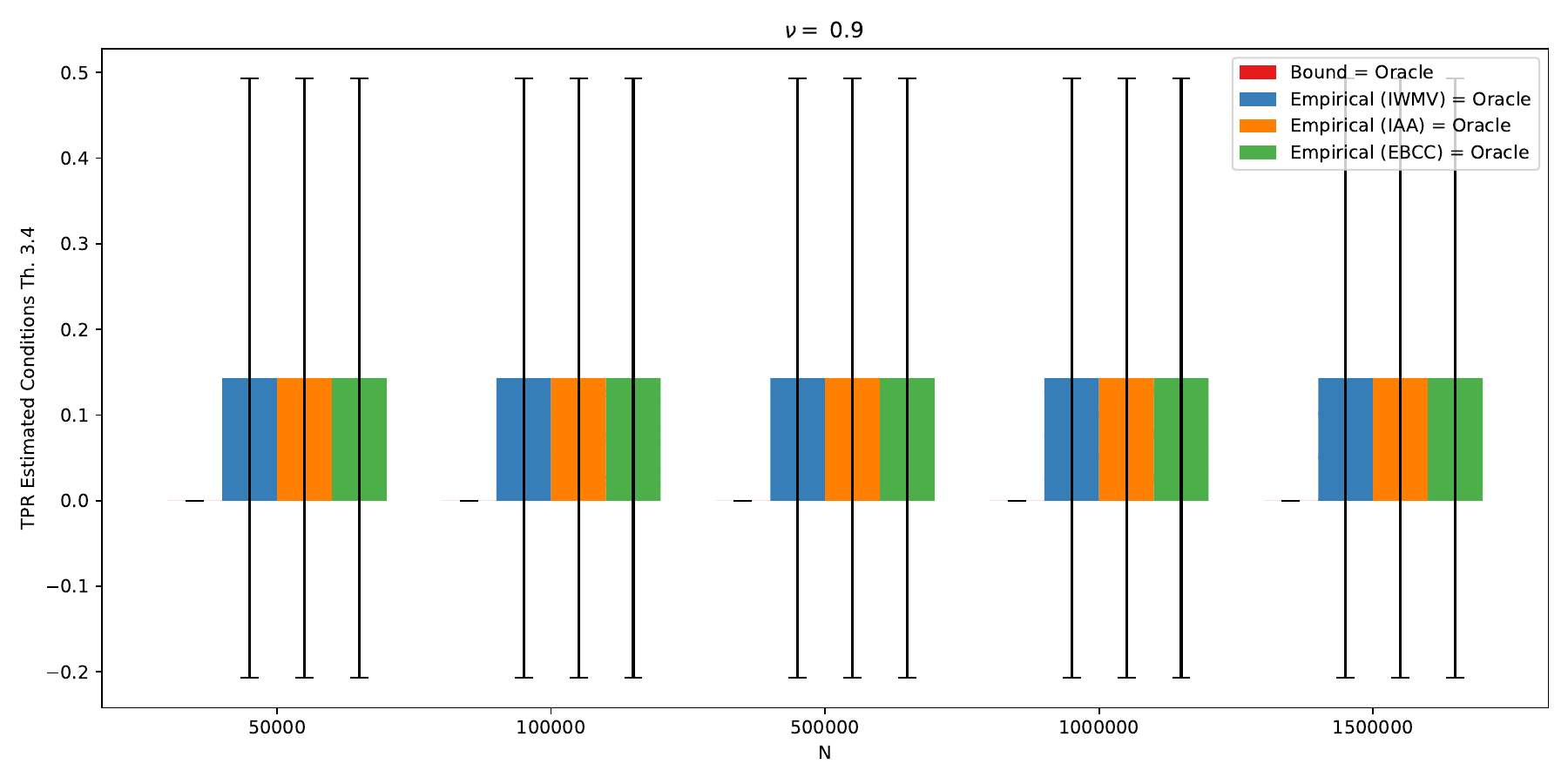}
        \caption{$\nu_0=0.9$}
    \end{subfigure}
\caption{Non-red bars show the fraction of experiments where verification of \cref{thm:twocoins} with estimated parameters from the candidate methods aligns with that of \cref{thm:twocoins} using the true $(\clim,\p)$, considering cases where the theorem is verified with true parameters.
    Red bars indicate cases where \cref{th:estimated_quantities} aligns with \cref{thm:twocoins} using true parameters.
    Synthetic data have various sample sizes~$N$, and the average True Positive Rate is plotted.}
    \label{img:bound_vs_empirical}
\end{figure*}

\cref{img:bound_vs_empirical} shows with non-red bars the proportion of experiments where the verification of \cref{thm:twocoins} using the estimated parameters from the candidate methods yields the same conclusion as the verification of \cref{thm:twocoins} using the true $(\clim,\p)$ values. This is calculated for the cases where the theorem is verified as true with the actual parameters. The red bars indicate the cases where \cref{th:estimated_quantities} and \cref{thm:twocoins} with true parameters lead to the same conclusion. The experiments were conducted with synthetic data of varying sample sizes~$N$ and with values of $\nu_0$ ranging from $0.6 \text{ to } 0.9$. The case with $\nu_0=0.5$ is in \cref{fig:histogram_samplewise} from the main paper. The percentage of data used for estimation is always equal to 10\%.
\subsection{Real-data Noise Transition Matrices}
\label{sec:real_data_t_matrices}
Anchor Map $T$ matrix computed using anchor points for datasets with 2 classes (SP, SP\_amt, ZenCrowd\_in, D-Product) are reported in \cref{table:t_matrices_anchor_map}.
\begin{table}[t]
    \centering
    \begin{tabular}{cc}
    \toprule
        Dataset name & anchor Map $T$ matrix\\
         \midrule
        SP & $\begin{bmatrix} 0.57 & 0.43\\ 0.35 & 0.65 \end{bmatrix}$\\[15pt]
        SP\_amt &  $\begin{bmatrix} 0.64 & 0.36\\ 0.36 & 0.64 \end{bmatrix}$\\ [15pt]
        ZC\_in & $\begin{bmatrix} 0.58 & 0.42\\ 0.49 & 0.51 \end{bmatrix}$\\[15pt]
        D-Product & $\begin{bmatrix} 0.72 & 0.28\\ 0.45 & 0.55 \end{bmatrix}$\\[10pt]
        \bottomrule
    \end{tabular}
    \caption{Noise transition matrices $T$ computed using anchor points for SP, SP\_amt, ZenCrowd\_in, D-Product, that are the binary classes real-world dataset we are using in our experiments.}
    \label{table:t_matrices_anchor_map}
\end{table}

Anchor Map $T$ matrix computed using anchor points for MS real-world dataset is reported in \cref{table:t_matrices_anchor_map_MS}:
\begin{align}
    \widetilde{T}_{MS} = \begin{bmatrix} 0.41 & 0.05 & 0.05 & 0.13 & 0.07 & 0.05 & 0.09 & 0.06 & 0.05 & 0.04 \\ 0.05 & 0.41 & 0.05 & 0.04 & 0.07 & 0.05 & 0.12 & 0.08 & 0.05 & 0.07 \\ 0.08 & 0.05 & 0.31 & 0.21 & 0.05 & 0.03 & 0.09 & 0.05 & 0.05 & 0.07 \\ 0.06 & 0.07 & 0.06 & 0.39 & 0.08 & 0.08 & 0.07 & 0.03 & 0.08 & 0.08 \\ 0.04 & 0.08 & 0.06 & 0.08 & 0.41 & 0.06 & 0.07 & 0.06 & 0.07 & 0.08 \\ 0.06 & 0.06 & 0.08 & 0.23 & 0.08 & 0.26 & 0.11 & 0.06 & 0.05 & 0.04 \\ 0.15 & 0.08 & 0.05 & 0.09 & 0.09 & 0.06 & 0.31 & 0.06 & 0.05 & 0.05 \\ 0.06 & 0.05 & 0.07 & 0.07 & 0.07 & 0.03 & 0.05 & 0.49 & 0.09 & 0.03 \\ 0.05 & 0.06 & 0.04 & 0.06 & 0.07 & 0.28 & 0.08 & 0.05 & 0.28 & 0.03 \\ 0.06 & 0.06 & 0.05 & 0.10 & 0.10 & 0.08 & 0.05 & 0.03 & 0.07 & 0.41 \end{bmatrix} 
    \label{table:t_matrices_anchor_map_MS}
\end{align}

Anchor Map $T$ matrix  computed using anchor points for  CF\_amt real-world dataset is reported in \cref{table:t_matrices_anchor_CFAMT}:
\begin{align}
\label{table:t_matrices_anchor_CFAMT}
      \widetilde{T}_{CF\_amt} =  \begin{bmatrix} 0.33 & 0.21 & 0.24 & 0.08 & 0.14 \\ 0.14 & 0.31 & 0.22 & 0.15 & 0.18 \\ 0.12 & 0.20 & 0.49 & 0.10 & 0.09 \\ 0.14 & 0.22 & 0.20 & 0.29 & 0.15 \\ 0.18 & 0.22 & 0.21 & 0.12 & 0.27 \end{bmatrix}
\end{align}

\subsection{Heatmap visualization of \cref{thm:twocoins}}

\begin{figure}[h]
    \centering
    \includegraphics[width=.5\linewidth]{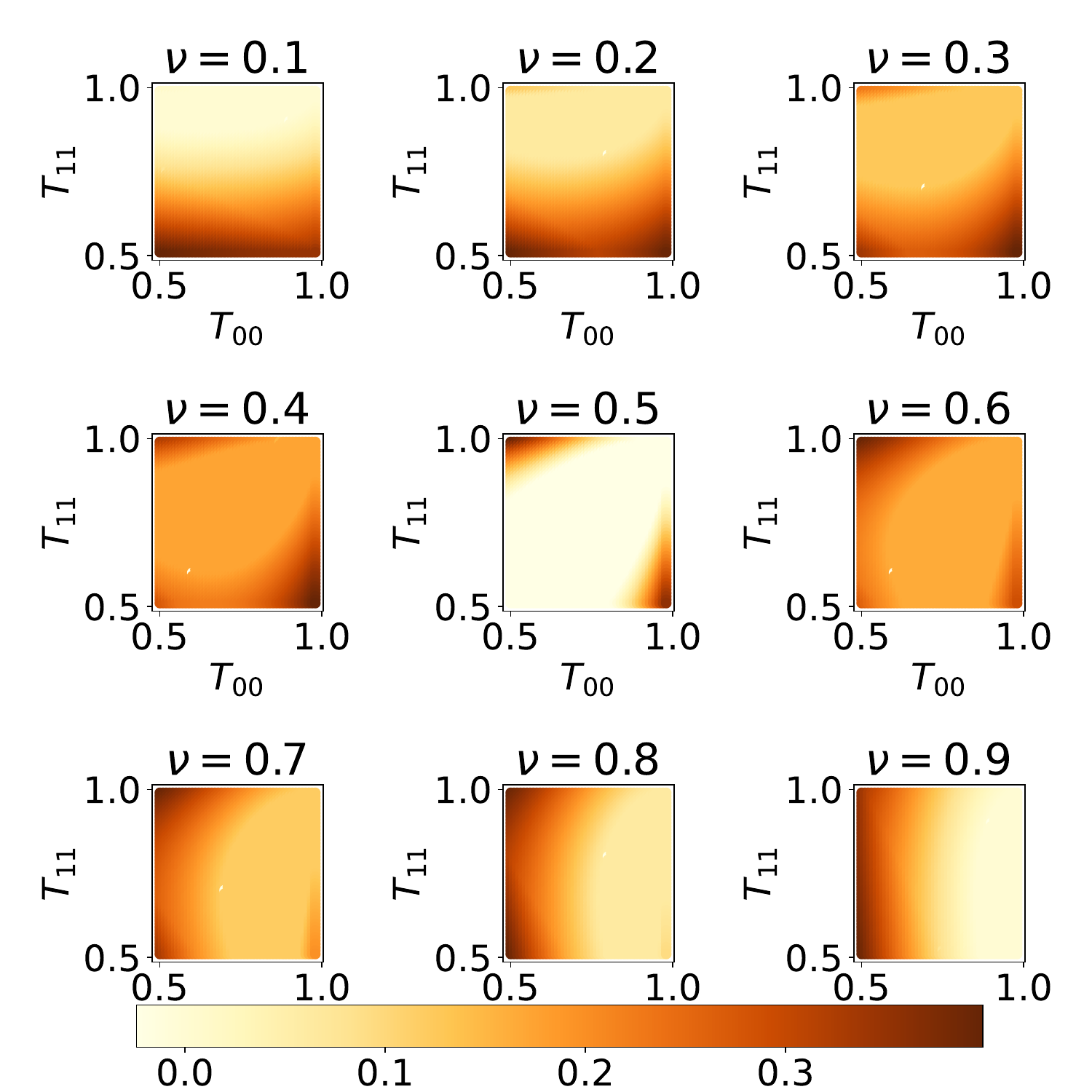}
    \caption{
    Heatmap visualization of the gap between oMAP and MV in the two-coin case.}
    \label{fig:heatmap_two_coin}
\end{figure}

The heatmap in \cref{fig:heatmap_two_coin} shows the gap between MV and oMAP. The experiment is obtained via simulations with different $\nu$ values ranging from $0.1$ to $0.9$. The plot obtained through
simulations accurately reflects the theorem’s conditions.

\newpage
\section{Beyond equal reliability assumption}
\label{sec:appendix_beyound_equal_rel}
\subsection{Uniformly perturbed T}
\label{sec:appendix_expectation}
In this setting we imagine that instead of being fixed, the annotators noise transition matrix are sampled form a distribution so that:\begin{equation}
\label{matrix_T_h}
  T_h = \begin{bmatrix}&T_{00} - \sigma_h &T_{01} + \sigma_h \\& T_{10} + \sigma_h &T_{11}- \sigma_h
\end{bmatrix}   \; \; \text{ with } \sigma_h \sim \text{Unif}[-\sigma, \sigma]
\end{equation} 
For a given $\sigma$.

Suppose we want the quantity to not depend of the particular pool of annotators we choose but simply on their number. What we can do in this case, having the knowledge annotators are sampled from that distribution is to use them to answer the question:

``\textit{Given annotators sampled around that distribution, not having the need of actually them being exactly equal, neither the necessity of knowing the exact reliability of each annotator, what is the expectation, on annotator distribution, of the probability of MV match the theoretical optima upper bound of the estimation given by oMAP?}''
In this case we're trying to answer the above question:

$$\mathbb{E}_{\sigma}[ \mathbb{P}(y^{MV}=y) |\sigma_1, \dots, \sigma_h) ] = \mathbb{E}_{\sigma}[ \mathbb{P}(  y^{oMAP}=y|\sigma_1, \dots, \sigma_h) ]$$

\begin{lemma}
Let $H$ be the number of annotators and $T_h$ their noise transition matrix, samples from the distribution defined in \cref{matrix_T_h}. For $c \in \{0,1\}$, it holds that: 
\[ \mathbb{E}[T^{MV}_{cc}] = \sum_{s=\ceil{\frac{H}{2}}}^{H} \binom{H}{s} T_{cc}^s (1- T_{cc})^{H-s}   \]
\end{lemma}

\begin{proof}
    
Denoting by $[H]  = \{1, \dots, H \}$, in the setting we just described we have that:
\begin{align*}
     T^{MV}_{cc} &= \mathbb{P} (y_{MV} = c| y=c) \\
     & = \mathbb{P} ( \text{at least } \ceil{\frac{H}{2}} \text{annotators vote class } c | y=c )\\
     & =  \sum_{s=\ceil{\frac{H}{2}}}^{H} \mathbb{P} ( \text{exactly  } s \text{ annotators vote class }c | y=c )\\
    & = \sum_{s=\frac{H+1}{2}}^{H} \sum_{\substack{ A \subseteq [H] \\  \text{s.t. }|A|=s}} \prod_{h \in A} (T_h)_{cc} \prod_{k \in H \ A} (T_{k})_{c 1-c}
\end{align*}
And the expected value of:\begin{align}
    \mathbb{E}[T^{MV}_{cc}] = \mathbb{E}\left[ \sum_{s= \frac{H+1}{2}}^{H} \sum_{\substack{ A \subseteq [H] \\  \text{s.t. }|A|=s}} \prod_{h \in A} (T_{h})_{cc} \prod_{k \in [H] \setminus A} (T_{k})_{c 1-c}  \right]
\end{align} 
The random variables $T^h$ are independent since we assume the $\sigma^h$ to be independent. So we have that:
   \begin{align}
    \mathbb{E}[T^{MV}_{cc}] & =  \sum_{s=\frac{H+1}{2}}^{H} \sum_{\substack{ A \subseteq [H] \\  \text{s.t. }|A|=s}}  \prod_{h \in A} \mathbb{E}\left[(T_{h})_{cc} \right] \prod_{k \in [H] \setminus A} \mathbb{E}\left[ (T_{k})_{c 1-c}  \right]\\
    &= \sum_{s=\frac{H+1}{2}}^{H} \sum_{\substack{ A \subseteq [H] \\  \text{s.t. }|A|=s}}  \prod_{h \in A} T_{cc} \prod_{k \in [H] \setminus A}  T_{c 1-c} \\
    &= \sum_{s=\ceil{\frac{H+1}{2}}}^{H} \binom{H}{s} T_{cc}^s (1- T_{cc})^{H-s} 
\end{align} 
\end{proof}

Let us now derive the formulation for oMAP. The following lemma hold. 

\begin{lemma}
    Let $H$ be the number of annotators and $T_h$ their noise transition matrix, samples from the distribution defined in \cref{matrix_T_h}, with :
\begin{equation}
       \sigma  \leq \frac{\log \rho}{H}     \left[  \frac{2}{T_{1-c1-c}} +   \frac{2}{ T_{c1-c}} + \frac{1}{T_{cc}} + \frac{1}{ T_{1-cc}} \right]^{-1} \min ( A_c- \floor{A_c},  1- A_c- \floor{A_c})
\end{equation}
Where:
$$A_c = \left[ \log \rho \right]^{-1}
 \left[ \log \left( \frac{\nu_{1-c}}{\nu_c} \right) + H\log \delta_{1-c} \right] $$
for $c \in \{0,1\}$. Then, it holds that:  
\[ \mathbb{E}[T^{oMAP}_{cc}] = \sum_{k=\ceil{A_c} }^H \binom{H}{k}T_{cc}^k (1-T_{cc})^{H-k}.  \]

\end{lemma}

\begin{proof}
Let us denote by $C_c$ the set of annotators that correctly voted annotators $c$ when the true label was $c$ and by  $W_c$ the set of annotators that incorrectly voted class $1-c$ when the true class was $c$. 
Notice that $W_c = [H] \setminus C_c$
We recall that the element $i$-th of the posterior probabilities vector, $ p_{i|\tilde{c}} = \nu_{i} \prod_{h \in C_{\tilde{c}}}  (T_h)_{i c} \prod_{k \in W_{\tilde{c}}} (T_k)_{i 1-c}  $. \\
The ${\text{argmax}}_{i \in \{1,0 \}}$ of $p_i$ is $c$ if $ p_{c|c} > p_{1-c|c}$ i.e.:

\begin{align}
  \nonumber
    \nu_{c} \prod_{h \in C_{c}}  (T_h)_{c c} \prod_{k \in W_{c}} (T_k)_{c 1-c} &>  \nu_{1-c} \prod_{h \in C_{c}}  (T_h)_{1-c c} \prod_{k \in [H] \setminus C_{{c}}} (T_k)_{1-c 1-c} \\
      \nonumber
    &\Updownarrow\\
    \nonumber
       \prod_{h \in C_{c}}  \frac{(T_h)_{c c}}{(T_h)_{1-c c}} \prod_{k  \in [H] \setminus C_{{c}}} & \frac{(T_k)_{c 1-c}}{ (T_k)_{1-c 1-c}} >  \frac{\nu_{1-c} }{  \nu_{c}}\\ 
         \nonumber
 &\Updownarrow\\
   \label{eq:condition_set}
 \sum_{h \in C_{c}}  \log\left (\frac{(T_h)_{c c}}{(T_h)_{1-c c}} \right) + \sum_{k  \in [H] \setminus C_{{c}}} & \log \left( \frac{(T_k)_{c 1-c}}{ (T_k)_{1-c 1-c}} \right) >  \log \left(\frac{\nu_{1-c} }{  \nu_{c}} \right)\\
   \nonumber
  &\Updownarrow\\
  \nonumber
\sum_{h \in H} m_h \log\left (\frac{(T_h)_{c c}}{(T_h)_{1-c c}} \right)  + (1-m_h) \log & \left(  \frac{(T_h)_{c 1-c}}{ (T_h)_{1-c 1-c}} \right)  > \log \left(\frac{\nu_{1-c} }{  \nu_{c}} \right) \text{ with } m_h= \mathds{1}_{h \in C_c}\\
  \nonumber
  &\Updownarrow\\
    \nonumber
\sum_{h \in H} m_h \left[  \log\left (\frac{(T_h)_{c c}}{(T_h)_{1-c c}} \right) -  \log \left( \frac{(T_h)_{c 1-c}}{ (T_h)_{1-c 1-c}} \right) \right]  & > \log \left(\frac{\nu_{1-c} }{  \nu_{c}} \right)+  \sum_{h \in H}  \log\left (\frac{(T_h)_{1-c 1-c}}{(T_h)_{c 1-c}} \right)\\
\nonumber
 &\Updownarrow\\
    \nonumber
    \sum_{h \in H} m_h \left[  \log\left (\frac{(T_h)_{c c} (T_h)_{1-c 1-c}}{(T_h)_{1-c c} (T_h)_{c 1-c} }  \right) \right]  & > \log \left(\frac{\nu_{1-c} }{  \nu_{c}} \right) + \sum_{h \in H} \log\left (\frac{(T_h)_{1-c 1-c}}{(T_h)_{c 1-c}} \right)\\
    \nonumber
     &\Updownarrow\\
     \nonumber
    \nonumber
      \sum_{h \in H} m_h  \log\left( \rho_h  \right)   > \log \left(\frac{\nu_{1-c} }{  \nu_{c}} \right)  -  \sum_{h \in H} & \log\left( (\delta_h)_c \right)  \text{ with } m_h= \mathds{1}_{h \in C_c}  \\
\end{align}

Denoting by $\rho_h = \frac{(T_h)_{c c} (T_h)_{1-c 1-c}}{(T_h)_{1-c c} (T_h)_{c 1-c} }  \; \text{and } (\delta_h)_c = \frac{(T_h)_{1-c 1-c}}{(T_h)_{c 1-c}}$.

The only random variable in equation \eqref{eq:condition_set} is the set $C_c$, with this writing we transferred the randomness to the variable $m_{h}$ and $m_h$ is a Bernoulli with parameter $(T_h)_{cc}$. 

Let us denote by $ \alpha_c = \log \left(\frac{\nu_{1-c} }{  \nu_{c}} \right)  -  \sum_{h \in H}  \log\left( (\delta_h)_c \right)  \text{ with } m_h= \mathds{1}_{h \in C_c}    $ and by $ \beta_h = \log(\rho_h)$. 
We are so saying that the probability that:
\[ \mathbb{P} (y^{oMAP}=y | y=c)  =  \mathbb{P} \left(   \sum_{h \in H} m_h  \beta_h   > \alpha_c   \right)  \]  with $ m_h \sim $ Bern$((T_{h})_{cc})$.

Now, we have discrete set of possible values for this sum $ \sum_{h \in H} m_h  \beta_h  $, indeed it can take values in this set $ V = \left\{ \sum_{h \in S} \beta_h \ \bigg| \ S \subseteq H \right\}
$ and  the probability of: 
\[ \mathbb{P} \left( \sum_{h \in H} m_h \beta_h =  \sum_{h \in S} \beta_h \right) =  \prod_{h \in S} (T_h)_{cc} \prod_{k \in [H] \setminus S} \left( 1 - (T_k)_{cc} \right)  \]
It follows that:
\begin{align*}
 T_{cc}^{oMAP} = \mathbb{P} \left(   \sum_{h \in H} m_h  \beta_h   > \alpha_c   \right)  = \sum_{\substack{ S \subseteq [H] \\  \text{s.t. }\sum_{s \in S} \beta_s> \alpha_c} } \prod_{h \in S} (T_h)_{cc} \prod_{k \in H \setminus S} \left( 1 - (T_k)_{cc} \right)
    \end{align*}

Without loss of generality, we can consider $c=0$, we have that: 
\[
\beta_h = \log \left( \frac{(T_{00} - \sigma_h)(T_{11} - \sigma_h)}{(T_{10} + \sigma_h)(T_{01} + \sigma_h)} \right)
\]

Substitute \( T_{ij} \pm \sigma_h \) and rewrite:

\begin{align*}
    \beta_h &= \log \left( \frac{T_{00} T_{11}}{T_{10} T_{01}} \cdot \frac{1 - \frac{\sigma_h}{T_{00}}}{1 + \frac{\sigma_h}{T_{10}}} \cdot \frac{1 - \frac{\sigma_h}{T_{11}}}{1 + \frac{\sigma_h}{T_{01}}} \right)\\
    &= \log \left( \frac{T_{00} T_{11}}{T_{10} T_{01}} \right) + \log \left( \frac{1 - \frac{\sigma_h}{T_{00}}}{1 + \frac{\sigma_h}{T_{10}}} \cdot \frac{1 - \frac{\sigma_h}{T_{11}}}{1 + \frac{\sigma_h}{T_{01}}} \right). 
\end{align*}

Similarly: 
\[
\log \left( \frac{T_{11} + \sigma_h}{T_{01} - \sigma_h} \right) = \log \left( \frac{T_{11}}{T_{01}} \right) + \log \left( \frac{1 + \frac{\sigma_h}{T_{11}}}{1 - \frac{\sigma_h}{T_{01}}} \right)
\]
So: $$ \alpha_0 = \log \left( \frac{\nu_1}{\nu_0} \right) + H\log \left( \frac{T_{11}}{T_{01}} \right) + \sum_{h \in [H]}\log \left( \frac{1 + \frac{\sigma_h}{T_{11}}}{1 - \frac{\sigma_h}{T_{01}}} \right)  $$

Substituting, we have that:  

\begin{align*}
    \sum_{s in S} \beta_s &> \alpha_0 \\
    &\Updownarrow\\
    |S|\log \left( \frac{T_{00} T_{11}}{T_{10} T_{01}} \right) + \sum_{h \in S} \log \left( \frac{1 - \frac{\sigma_h}{T_{00}}}{1 + \frac{\sigma_h}{T_{10}}} \cdot \frac{1 - \frac{\sigma_h}{T_{11}}}{1 + \frac{\sigma_h}{T_{01}}} \right) & > \log \left( \frac{\nu_1}{\nu_0} \right) + H\log \left( \frac{T_{11}}{T_{01}} \right) + \sum_{h \in [H]}\log \left( \frac{1 + \frac{\sigma_h}{T_{11}}}{1 - \frac{\sigma_h}{T_{01}}} \right)\\
    \Updownarrow\\
    |S|\log \left( \frac{T_{00} T_{11}}{T_{10} T_{01}} \right)   > \log \left( \frac{\nu_1}{\nu_0} \right) + H\log \left( \frac{T_{11}}{T_{01}} \right) & + \sum_{h \in [H]}\log \left( \frac{1 + \frac{\sigma_h}{T_{11}}}{1 - \frac{\sigma_h} {T_{01}}} \right) -  \sum_{h \in S} \log \left( \frac{1 - \frac{\sigma_h}{T_{00}}}{1 + \frac{\sigma_h}{T_{10}}} \cdot \frac{1 - \frac{\sigma_h}{T_{11}}}{1 + \frac{\sigma_h}{T_{01}}} \right)\\
     \Updownarrow\\
     |S|  > \log \left( \frac{T_{10} T_{01}}{T_{00} T_{11}} \right) 
 \left[ \log \left( \frac{\nu_1}{\nu_0} \right) + H\log \left( \frac{T_{11}}{T_{01}} \right) \right]  & + \log \left( \frac{T_{10} T_{01}}{T_{00} T_{11}} \right)\left[\sum_{h \in [H]}\log \left( \frac{1 + \frac{\sigma_h}{T_{11}}}{1 - \frac{\sigma_h} {T_{01}}} \right) -  \sum_{h \in S} \log \left( \frac{1 - \frac{\sigma_h}{T_{00}}}{1 + \frac{\sigma_h}{T_{10}}} \cdot \frac{1 - \frac{\sigma_h}{T_{11}}}{1 + \frac{\sigma_h}{T_{01}}} \right) \right]
\end{align*}

Let us denote by: $$ \xi = + \log \left( \frac{T_{10} T_{01}}{T_{00} T_{11}} \right)\left[\sum_{h \in [H]}\log \left( \frac{1 + \frac{\sigma_h}{T_{11}}}{1 - \frac{\sigma_h} {T_{01}}} \right) -  \sum_{h \in S} \log \left( \frac{1 - \frac{\sigma_h}{T_{00}}}{1 + \frac{\sigma_h}{T_{10}}} \cdot \frac{1 - \frac{\sigma_h}{T_{11}}}{1 + \frac{\sigma_h}{T_{01}}} \right) \right]$$

We already defined in the previous section $A_0 $ as: 
$$A_0 = \left[ \log \left( \frac{T_{00} T_{11}}{T_{10} T_{01}} \right) \right]^{-1}
 \left[ \log \left( \frac{\nu_1}{\nu_0} \right) + H\log \left( \frac{T_{11}}{T_{01}} \right) \right] $$

Where $|S|$ is an integer.

 so $|S| > A_0 + \xi  \iff |S| \geq \ceil{A_0 + \xi} $

 If $\xi$ is so small that $\ceil{A_c + \xi}  = \ceil{A_c } $ we have that:
 \begin{equation}
 \label{eq:ceil_cond}
     \ceil{A_0 + \xi} = \ceil{A_c } \iff - (A_0- \floor{A_0}) \leq \xi \leq 1- (A_0- \floor{A_0}) 
 \end{equation}

 We can always seletc $\sigma_h$ so that the maximum value of the $\sigma_h$ is small so that the condition in \cref{eq:ceil_cond} is met. 
We assumed, $ |\sigma_h|< \sigma$.

  \begin{align*}
    |\xi| & \leq \left[ \log \left( \frac{T_{00} T_{11}}{T_{01} T_{01}} \right)  \right]^{-1}   \left[ \left| \sum_{h \in [H]}\log \left( \frac{1 + \frac{\sigma_h}{T_{11}}}{1 - \frac{\sigma_h} {T_{01}}} \right) \right| +   \left| \sum_{h \in S} \log \left( \frac{1 - \frac{\sigma_h}{T_{00}}}{1 + \frac{\sigma_h}{T_{10}}} \cdot \frac{1 - \frac{\sigma_h}{T_{11}}}{1 + \frac{\sigma_h}{T_{01}}} \right) \right| \right] \\
    & \leq H \left[ \log \left( \frac{T_{00} T_{11}}{T_{01} T_{01}} \right)  \right]^{-1}
   \left[    \log \left( \frac{1 + \frac{\sigma}{T_{11}}}{1 - \frac{\sigma} {T_{01}}} \right)   + \log \left( \frac{1 + \frac{\sigma}{T_{00}}}{1 - \frac{\sigma}{T_{10}}} \cdot \frac{1 + \frac{\sigma}{T_{11}}}{1 - \frac{\sigma}{T_{01}}} \right) \right]  \\
   & \leq H \left[ \log \left( \frac{T_{00} T_{11}}{T_{01} T_{01}} \right)  \right]^{-1}
   \left[    2\log \left( \frac{1 + \frac{\sigma}{T_{11}}}{1 - \frac{\sigma} {T_{01}}} \right)   + \log \left( \frac{1 + \frac{\sigma}{T_{00}}}{1 - \frac{\sigma}{T_{10}}}  \right) \right] \\
   & = H \left[ \log \left( \frac{T_{00} T_{11}}{T_{01} T_{01}} \right)  \right]^{-1}
   \left[    2\log \left( 1 + \frac{\sigma}{T_{11}} \right) -2 \log \left( {1 - \frac{\sigma} {T_{01}}} \right)   + \log \left( 1 + \frac{\sigma}{T_{00}} \right) - \log \left({1 - \frac{\sigma}{T_{10}}}  \right) \right] \\
H \left[ \log \left( \frac{T_{00} T_{11}}{T_{01} T_{01}} \right)  \right]^{-1} 
      \end{align*} 

We can use that $ \frac{x}{1+x} \leq \ln( 1+x) \leq x $ for all $x >-1$ to obtain that: 
\begin{align*}
    &2\log \left( 1 + \frac{\sigma}{T_{11}} \right) -2 \log \left( {1 - \frac{\sigma} {T_{01}}} \right)   + \log \left( 1 + \frac{\sigma}{T_{00}} \right) - \log \left({1 - \frac{\sigma}{T_{10}}}  \right) \\
    & \leq 2 \frac{\sigma}{T_{11}} +  2 \frac{\sigma}{ \sigma +T_{01}} + \frac{\sigma}{T_{00}} + \frac{\sigma}{ \sigma +T_{10}}\\
    & \leq 2 \frac{\sigma}{T_{11}} +  2 \frac{\sigma}{ T_{01}} + \frac{\sigma}{T_{00}} + \frac{\sigma}{ T_{10}}\\
    &= \sigma \left[  \frac{2}{T_{11}} +   \frac{2}{ T_{01}} + \frac{1}{T_{00}} + \frac{1}{ T_{10}} \right]
\end{align*}

Putting everything together: 
\begin{align*}
    |\xi| & \leq \sigma  H \left[ \log \left( \frac{T_{00} T_{11}}{T_{01} T_{01}} \right)  \right]^{-1}  \left[  \frac{2}{T_{11}} +   \frac{2}{ T_{01}} + \frac{1}{T_{00}} + \frac{1}{ T_{10}} \right]  \leq \sigma  H \left[ \log \left( \frac{T_{00} T_{11}}{T_{01} T_{01}} \right)  \right]^{-1} 
\end{align*}

The following conditions implies \cref{eq:ceil_cond}:

\begin{align*}
   \sigma  H \left[ \log \left( \frac{T_{00} T_{11}}{T_{01} T_{01}} \right)  \right]^{-1}  \left[  \frac{2}{T_{11}} +   \frac{2}{ T_{01}} + \frac{1}{T_{00}} + \frac{1}{ T_{10}} \right] \leq  \min ( A_0- \floor{A_0},  1- A_0- \floor{A_0}) \\
   \sigma  \leq \frac{1}{H} \left[ \log \left( \frac{T_{00} T_{11}}{T_{01} T_{01}} \right)  \right]   \left[  \frac{2}{T_{11}} +   \frac{2}{ T_{01}} + \frac{1}{T_{00}} + \frac{1}{ T_{10}} \right]^{-1} \min ( A_0- \floor{A_0},  1- A_0- \floor{A_0})\\
   \sigma  \leq \frac{\log \left( \rho_0 \right)}{H}     \left[  \frac{2}{T_{11}} +   \frac{2}{ T_{01}} + \frac{1}{T_{00}} + \frac{1}{ T_{10}} \right]^{-1} \min ( A_0- \floor{A_0},  1- A_0- \floor{A_0})
\end{align*}

Now, assuming condition \cref{eq:ceil_cond} is true:
\begin{align}
 T^{oMAP}_{cc} =  
 \sum_{\substack{ S \subseteq [H] \\  \text{s.t. }   \sum_{l \in S} \beta_l> \alpha_c} } \prod_{h \in S} (T_h)_{cc} \prod_{k \in H \setminus S} \left( 1 - (T_k)_{cc} \right) \\
    T^{oMAP}_{cc} =  
    \sum_{\substack{ S \subseteq [H] \\  \text{s.t. } |S| > \ceil{a} }} \prod_{h \in S} (T_h)_{cc} \prod_{k \in H \setminus S} \left( 1 - (T_k)_{cc} \right)  
\end{align} 

Again, using that the $T_h$ are independent, we obtain that:

\begin{align}
    \mathbb{E}[ T^{oMAP}_{cc}] & =  
    \sum_{\substack{ S \subseteq [H] \\  \text{s.t. } |S| > \ceil{a} }}  \prod_{h \in S} \mathbb{E} \left[(T_h)_{cc} \right]  \prod_{k \in H \setminus S} \left( 1 - \mathbb{E} \left[ (T_k)_{cc}  \right] \right)  \\
    &=\sum_{\substack{ S \subseteq [H] \\  \text{s.t. } |S| > \ceil{a} }}  T_{cc}^{|S|} \left( 1 -  (T_k)_{cc} \right)^{H-|S|} \\
    &= \sum_{s=\ceil{A_c}}^{H} \binom{H}{s} T_{cc}^s (1- T_{cc})^{H-s} 
\end{align}

\end{proof}

Thanks to the previous Lemmas we are ready to prove the following Theorem.

\begin{theorem}
Using the notation from \cref{eq:notation_rho_delta}.
Let $H$ be the number of annotators and $T_h$ their noise transition matrix, samples from the distribution defined in \cref{matrix_T_h}, with: 
\begin{equation}
       \sigma  \leq \frac{\log \rho}{H}     \left[  \frac{2}{T_{1-c1-c}} +   \frac{2}{ T_{c1-c}} + \frac{1}{T_{cc}} + \frac{1}{ T_{1-cc}} \right]^{-1} \min ( A_c- \floor{A_c},  1- A_c- \floor{A_c})
\end{equation}
for $c \in \{0,1\} $ and with 
$$A_c = \left[ \log \rho \right]^{-1}
 \left[ \log \left( \frac{\nu_{1-c}}{\nu_c} \right) + H\log \delta_{1-c} \right]. $$

 we have that if: 
\begin{align}
    \biggl(\frac{\delta_0}{\delta_{1}} \biggr)^{\frac{H}{2}} \frac{1}{\sqrt{ \rho}}  < \frac{\nu_{1}}{\nu_0} < \biggl(\frac{\delta_0}{\delta_{1}} \biggr)^{\frac{H}{2}} \sqrt{ \rho}
\end{align}
Then $\mathbb{E}[T^{oMAP}_{cc}] = \mathbb{E}[T_{cc}^{MV}] $ and in expectations, over the distribution of the annotators oMAP is equivalent to MV instance wise. 
\end{theorem}

\begin{proof}
\begin{align*}
 \mathbb{E}[ \mathbb{P}(y_{MV} = y)] &= \mathbb{E}[\mathbb{P}(y_{oMV} = y)]  \\
 & \Updownarrow\\
    \mathbb{E}[ \mathbb{P}(y_{MV} = y | y=0)\mathbb{P}(y=0) &+ {P}(y_{MV} = y | y=1)\mathbb{P}(y=1)]  \\
    = \mathbb{E}[  \mathbb{P}(y_{oMAP} = y | y=0)\mathbb{P}(y=0) &+ \mathbb{P}(y_{oMAP} = y | y=1)\mathbb{P}(y=1)]
     & \Updownarrow\\
    \mathbb{E} \left[    T_{00}^{\MV} + T_{11}^{\MV}\frac{1 -\nu}{\nu} \right] &= \mathbb{E} \left[ T_{00}^{\oMAP} + T_{11}^{\oMAP}\frac{1 -\nu}{\nu}  \right] 
\end{align*}
Using the two Lemmas proved before we can just rely on the proof of \cref{thm:suff_condition_appenidx} to conclude this proof.
\end{proof}

\subsubsection{Experiments}

\begin{table}[h]
    \centering
    \begin{tabular}{cccccc}
    \toprule
    $T_{00}$ & $T_{11}$ & $\nu_0$ & \cref{eq:check_in_expectation} satisfied & oMAP Acc. & MV Acc.  \\
    \midrule
    0.7 & 0.8 & 0.6 & \cmark & 0.8328 & 0.8328 \\
    0.7 & 0.8 & 0.9 & \xmark & 0.9261 & 0.7961 \\
    0.55 & 0.55 & 0.6 & \xmark & 0.6099 & 0.5677 \\
    0.55 & 0.55 & 0.9 & \xmark & 0.8992 & 0.5785 \\
    0.9 & 0.7 & 0.6 & \cmark & 0.8974 & 0.8974 \\
    0.9 & 0.7 & 0.9 & \cmark & 0.9521 & 0.9521 \\
    0.6 & 0.6 & 0.6 & \cmark & 0.6499 & 0.6499 \\
    0.6 & 0.6 & 0.9 & \cmark & 0.9030 & 0.6485 \\
    
    \bottomrule
    \end{tabular}
    \caption{Each annotator $h$ has a different $T_h$ matrix, obtained via a perturbation of the original $T$. The condition on \cref{eq:check_in_expectation} is checked and then the Accuracy of oMAP and MV is presented.  This table uses $H=3 \text{ and } N=10000$.}
\label{table:annotators_diff_rel}
\end{table}

\cref{table:annotators_diff_rel} confirms the theoretical results obtained in \cref{sec:relaxing_reliability_assumption} and demonstrated in this section. This table uses $H=3 \text{ and } N=10000$. To obtain the results each annotator $h$ has a different $T_h$ matrix, obtained via a perturbation of the original $T$ (shown in the Table). The condition on \cref{eq:check_in_expectation} is checked and then the Accuracy of oMAP and MV is presented. As expected, all the times there is the $\cmark \,$ symbol the accuracy of oMAP matches the one of MV, showing the correctness of the theoretical results. The experiments are averaged across multiple runs with different seed values.

\subsection{Annotators of two categories}
\label{sec:appendix_two_groups_workers}
Let us now consider the case in which we have to group of annotators with different reliabilities. Specifically we have group $A$, with noise transition matrix $T_A$ and group $B$ with noise tranistion matrix $T_B$. 
Let us derive what are $T_{cc}^{MV}$ and $T_{cc}^{oMAP}$ in this case. 
Wlog we can assume $|A|<|B|$ and $|A| < \ceil{\frac{ H}{2}}$.
\begin{align*}
     T^{MV}_{cc}  & = \mathbb{P} ( \text{at least } \ceil{\frac{H}{2}} \text{annotators vote class } c | y=c )\\
     & =  \sum_{k=1}^{|A|} \sum_{l = \ceil{\frac{ H}{2}} - k}^{|B|} \binom{|A|}{k} \binom{|B|}{l} (T_A)_{cc}^k (1-(T_{A})_{cc})^{|A| -k} (T_B)_{cc}^l (1-(T_{B})_{cc})^{|B| -l}.
\end{align*}

We define:

$ \alpha_c = \log \left( \frac{\nu_{1-c}}{\nu_c} \right) + |B| \log (\delta_B)_{1-c}  +  |A| \log (\delta_A)_{1-c}  = \log \left( \frac{\nu_{1-c}}{\nu_c} \right) + H \log (\delta_B)_{1-c}  +  |A| \log \frac{(\delta_A)_{1-c}}{(\delta_B)_{1-c} } 
 $, for oMAP: 
\begin{align*}
     T^{oMAP}_{cc}  & =  \sum_{k=1}^{|A|} \sum_{l = \ceil{ \frac{\alpha}{\log \rho_B}- k \frac{\log \rho_A}{\log \rho_B} }}^{|B|} \binom{|A|}{k} \binom{|B|}{l} (T_A)_{cc}^k (1-(T_{A})_{cc})^{|A| -k} (T_B)_{cc}^l (1-(T_{B})_{cc})^{|B| -l}.
\end{align*}

\begin{theorem}
Let us assume the annotators belong to two groups $A$ and $B$, with different reliabilities. Specifically group $A$ has noise transition matrix $T_A$ and group $B$ has noise transition matrix $T_B$.
Denoting by $
    \rho_A = {\frac{(T_A)_{cc}(T_A)_{1-c,1-c}}{(1-(T_A)_{cc})(1-(T_A)_{1-c,1-c})}} $ and $ (\delta_A)_c = \frac{(T_A)_{cc}}{1-(T_A)_{1-c,1-c}} $, if $\rho_A = \rho_B$ we have that if 
\begin{align}
\label{eq:condition_sufficency_appendix_two_groups}
 \left( \frac{(\delta_B)_c}{(\delta_B)_{1-c}}\right)^{\frac{H}{2} }  \sqrt{ \frac{1}{\rho_B} } \left( \frac{(\delta_B)_{1-c}}{(\delta_A)_{1-c} } \right)^{|A|}   <   \frac{\nu_{1-c}}{\nu_c} \leq  \left( \frac{(\delta_B)_c}{(\delta_B)_{1-c}}\right)^{\frac{H}{2} }  \sqrt{ {\rho_B} }  \left( \frac{(\delta_B)_{1-c}}{(\delta_A)_{1-c} } \right)^{|A|}  
\end{align}
Then $T^{oMAP}_{00} = T_{00}^{MV} $ and that $T_{11}^{oMAP } = T_{11}^{MV}$ from which it follows that under the conditions in described in \cref{eq:condition_sufficency_appendix} oMAP is equivalent to MV instance wise. 
\end{theorem}

\begin{proof}
Now, 
$  T^{MV}_{cc} = T^{oMAP}_{cc}  \iff  \forall k \in \{ 1, \dots, |A| \} \ceil{\frac{H}{2}} - k = \ceil{ \frac{\alpha_c}{\log \rho_B} - k\frac{\log \rho_A}{\log \rho_B}} $ 

$\iff \forall k \in \{ 1, \dots, |A| \} \frac{H+1}{2} - k -1 <  \frac{\alpha_c}{\log \rho_B} - k\frac{\log \rho_A}{\log \rho_B} \leq \frac{H+1}{2} - k 
$

$\iff \forall k \in \{ 1, \dots, |A| \} \frac{H+1}{2} - k \left( 1- \frac{\log \rho_A}{\log \rho_B}  \right) -1 <  \frac{\alpha_c}{\log \rho_B}  \leq \frac{H+1}{2} - k \left( 1- \frac{\log \rho_A}{\log \rho_B}  \right)
$

$$\iff \frac{H+1}{2} - k \left( 1- \frac{\log \rho_A}{\log \rho_B}  \right) -1 <  \frac{\log \left( \frac{\nu_{1-c}}{\nu_c} \right) + H \log (\delta_B)_{1-c}  +  |A| \log \frac{(\delta_A)_{1-c}}{(\delta_B)_{1-c} }  }{\log \rho_B}  \leq \frac{H+1}{2} - k \left( 1- \frac{\log \rho_A}{\log \rho_B}  \right)
$$

$$ \frac{(H-1)\log \rho_B}{2} - k \left( \log \frac{\rho_B}{\rho_A} \right)  <  \log \left( \frac{\nu_{1-c}}{\nu_c} \right) + H \log (\delta_B)_{1-c}  +  |A| \log \frac{(\delta_A)_{1-c}}{(\delta_B)_{1-c} } \leq \frac{(H+1)\log \rho_B}{2} - k \left( \log \frac{\rho_B}{\rho_A} \right)
$$

$$ \Updownarrow$$
$$ \frac{(H-1)\log \rho_B}{2} - k \left( \log \frac{\rho_B}{\rho_A} \right) - H \log (\delta_B)_{1-c}  -  |A| \log \frac{(\delta_A)_{1-c}}{(\delta_B)_{1-c} }   $$
$$ <  \log \left( \frac{\nu_{1-c}}{\nu_c} \right) 
$$
$$\leq \frac{(H+1)\log \rho_B}{2} - k \left( \log \frac{\rho_B}{\rho_A} \right) - H \log (\delta_B)_{1-c}  -  |A| \log \frac{(\delta_A)_{1-c}}{(\delta_B)_{1-c} } 
$$

$$ \Updownarrow$$
$$ \frac{(H-1)\log \rho_B}{2} - k \left( \log \frac{\rho_B}{\rho_A} \right) - H \log (\delta_B)_{1-c}  -  |A| \log \frac{(\delta_A)_{1-c}}{(\delta_B)_{1-c} }   $$
$$ <  \log \left( \frac{\nu_{1-c}}{\nu_c} \right) 
$$
$$\leq \frac{(H+1)\log \rho_B}{2} - k \left( \log \frac{\rho_B}{\rho_A} \right) - H \log (\delta_B)_{1-c}  -  |A| \log \frac{(\delta_A)_{1-c}}{(\delta_B)_{1-c} } 
$$

Noticing that: 
$$ \log{ \sqrt{\rho}} - \log{\delta_{1-c}} = \log{ \sqrt{ \frac{ \delta_{c}}{\delta_{1-c}}}}, $$

$$ \Updownarrow$$ 
$$ \frac{H}{2} \log \left( \frac{(\delta_B)_c}{(\delta_B)_{1-c}}\right)  - \frac{1}{2} \log \rho_B - k \left( \log \frac{\rho_B}{\rho_A} \right) -  |A| \log \frac{(\delta_A)_{1-c}}{(\delta_B)_{1-c} }   $$
$$ <  \log \left( \frac{\nu_{1-c}}{\nu_c} \right) 
$$
$$\leq \frac{H}{2} \log \left( \frac{(\delta_B)_c}{(\delta_B)_{1-c}}\right)  + \frac{1}{2} \log \rho_B - k \left( \log \frac{\rho_B}{\rho_A} \right) -  |A| \log \frac{(\delta_A)_{1-c}}{(\delta_B)_{1-c} }
$$

$$ \Updownarrow$$ 
$$ \left( \frac{(\delta_B)_c}{(\delta_B)_{1-c}}\right)^{\frac{H}{2} }  \sqrt{ \frac{1}{\rho_B} } \left( \frac{\rho_A}{\rho_B} \right)^k  \left( \frac{(\delta_B)_{1-c}}{(\delta_A)_{1-c} } \right)^{|A|}   $$
$$ <  \frac{\nu_{1-c}}{\nu_c} 
$$
$$\leq \left( \frac{(\delta_B)_c}{(\delta_B)_{1-c}}\right)^{\frac{H}{2} }  \sqrt{ {\rho_B} } \left( \frac{\rho_A}{\rho_B} \right)^k  \left( \frac{(\delta_B)_{1-c}}{(\delta_A)_{1-c} } \right)^{|A|}  
$$

$$ \Updownarrow$$ 
$$ \left( \frac{(\delta_B)_c}{(\delta_B)_{1-c}}\right)^{\frac{H}{2} }  \sqrt{ \frac{1}{\rho_B} } \left( \frac{\rho_A}{\rho_B} \right)^k  \left( \frac{(\delta_B)_{1-c}}{(\delta_A)_{1-c} } \right)^{|A|}   <   \frac{\nu_{1-c}}{\nu_c} \leq  \left( \frac{(\delta_B)_c}{(\delta_B)_{1-c}}\right)^{\frac{H}{2} }  \sqrt{ {\rho_B} } \left( \frac{\rho_A}{\rho_B} \right)^k  \left( \frac{(\delta_B)_{1-c}}{(\delta_A)_{1-c} } \right)^{|A|}  
$$

$$ \Updownarrow$$ 
$$ \left( \frac{(\delta_B)_c}{(\delta_B)_{1-c}}\right)^{\frac{H}{2} }  \sqrt{ \frac{1}{\rho_B} } \left( \frac{\rho_A}{\rho_B} \right)^k  \left( \frac{(\delta_B)_{1-c}}{(\delta_A)_{1-c} } \right)^{|A|}   <   \frac{\nu_{1-c}}{\nu_c} \leq  \left( \frac{(\delta_B)_c}{(\delta_B)_{1-c}}\right)^{\frac{H}{2} }  \sqrt{ {\rho_B} } \left( \frac{\rho_A}{\rho_B} \right)^k  \left( \frac{(\delta_B)_{1-c}}{(\delta_A)_{1-c} } \right)^{|A|}  
$$

Under the hypothesis $ \rho_A = \rho_B$:

\begin{equation}
\label{eq:two_classes_condition}
\left( \frac{(\delta_B)_c}{(\delta_B)_{1-c}}\right)^{\frac{H}{2} }  \sqrt{ \frac{1}{\rho_B} } \left( \frac{(\delta_B)_{1-c}}{(\delta_A)_{1-c} } \right)^{|A|}   <   \frac{\nu_{1-c}}{\nu_c} \leq  \left( \frac{(\delta_B)_c}{(\delta_B)_{1-c}}\right)^{\frac{H}{2} }  \sqrt{ {\rho_B} }  \left( \frac{(\delta_B)_{1-c}}{(\delta_A)_{1-c} } \right)^{|A|}  
\end{equation}
\end{proof}

\subsubsection{Experiments}
\begin{table}[ht]
    \centering
    \begin{tabular}{cccccccc}
    \toprule
    $T_{00}^A$ & $T_{11}^A$ & $T_{00}^B$ & $T_{11}^B$ & $\nu_0$ & \cref{eq:two_classes_condition} Satisfied & oMAP Acc. & MV Acc.\\
    \midrule
    0.58 & 0.8 & 0.8 & 0.58 & 0.55 & \cmark & 0.8686 & 0.8686 \\
    0.58 & 0.8 & 0.8 & 0.58 & 0.65 & \cmark & 0.8567 & 0.8567 \\
    0.58 & 0.8 & 0.8 & 0.58 & 0.95 & \xmark & 0.9644 & 0.8438 \\
    0.78 & 0.65 & 0.65 & 0.78 & 0.55 & \cmark & 0.8993 & 0.8993 \\
    0.78 & 0.65 & 0.65 & 0.78 & 0.65 & \cmark & 0.8950 & 0.8950 \\
    0.78 & 0.65 & 0.65 & 0.78 & 0.95 & \xmark & 0.9658 & 0.9058 \\
        \bottomrule
    \end{tabular}    \caption{Experiments on synthetic data to check the correctness of \cref{eq:two_classes_condition}. Two class of annotators A and B (with $|A|=3$ and $|B|=4$) have two noise transition matrices $T_A$ and $T_B$. As expected, when there is a $\cmark \,$ the accuracy of MV matches the one of oMAP, showing the empirical correctness of the proposed formulation.}
    \label{table:two_annotator_classes}
\end{table}

To verify the correctness of \cref{eq:two_classes_condition} we perform the following experiment: two classes of annotators A and B are presented, each with its own noise transition matrix $T_A$ and $T_B$ (diagonal values presented in the Table). Then \cref{eq:two_classes_condition} is used to verify if, with the current data, MV is the optimal solution or not. To check if the results from the condition are confirmed we present the accuracy of MV and oMAP in \cref{table:two_annotator_classes}. When the condition is satisfied ($\cmark$) the accuracy of the two aggregation methods is the same, confirming the correctness of the proposed solution.
Since one condition for this theorem is to have $\rho_A=\rho_B$, for sake of simplicity, we simply inverted the rows of the noise transition matrices A and B.

\end{document}